\documentclass{article} % For LaTeX2e
\usepackage{iclr2027_conference,times}
\iclrfinalcopy % arXiv preprint: show authors

\title{World Action Planner: Generalizable Robot Decision-Making with Action-Conditioned World Models}

\usepackage[utf8]{inputenc} % allow utf-8 input
\usepackage[T1]{fontenc}    % use 8-bit T1 fonts
\usepackage{enumitem}
\usepackage[pagebackref=true,breaklinks=true,colorlinks,citecolor=blue]{hyperref}
\usepackage{url}            % simple URL typesetting
\usepackage{booktabs}       % professional-quality tables
\usepackage{amsfonts}       % blackboard math symbols
\usepackage{nicefrac}       % compact symbols for 1/2, etc.
\usepackage{microtype}      % microtypography
\usepackage{xcolor}         % colors
\usepackage{subcaption}
\usepackage{amssymb}

\renewcommand{\paragraph}[1]{\noindent\textbf{#1}.}

\usepackage{algorithm}
\usepackage{algorithmic}
\usepackage{amsmath,amsfonts,bm}
\usepackage{amsthm}
\usepackage{thmtools}
\usepackage{thm-restate}
\usepackage{mathtools}
\usepackage{bm}

\usepackage{booktabs}
\usepackage{multirow}
\usepackage{subcaption}
\usepackage{graphicx}
\usepackage[most]{tcolorbox}
\usepackage[skip=3pt,font=small]{caption}
\usepackage{wrapfig}
\usepackage{lipsum}
\usepackage{diagbox}
\usepackage{makecell}
\usepackage{tablefootnote}

\newtheorem{theorem}{Theorem}
\newtheorem{assumption}{Assumption}
\newtheorem{definition}{Definition}
\newtheorem{lemma}{Lemma}
\newtheorem{proposition}{Proposition}

\newtheorem{remark}{Remark}

\DeclareMathAlphabet{\mathsfit}{\encodingdefault}{\sfdefault}{m}{sl}
\SetMathAlphabet{\mathsfit}{bold}{\encodingdefault}{\sfdefault}{bx}{n}

\newcommand{\EE}{\mathbb{E}}
\newcommand{\RR}{\mathbb{R}}

\newcommand{\cA}{\mathcal{A}}

\newcommand{\cC}{\mathcal{C}}
\newcommand{\cK}{\mathcal{K}}

\newcommand{\cF}{\mathcal{F}}

\newcommand{\cS}{\mathcal{S}}
\newcommand{\cV}{\mathcal{V}}
\newcommand{\cN}{\mathcal{N}}

\newcommand{\cD}{\mathcal{D}}

\newcommand{\innerproduct}[2]{\left\langle #1, #2 \right\rangle}
\newcommand{\cM}{\mathcal{M}}

\newcommand{\poly}{\mathrm{poly}}

\newcommand{\abs}[1]{\left| #1 \right|}

\newcommand{\bracket}[1]{\left(#1\right)}
\newcommand{\mbracket}[1]{\left[#1\right]}
\newcommand{\norm}[1]{\left\| #1 \right\|}
\newcommand{\sets}[1]{\left\{ #1 \right\}}
\newcommand{\PP}{\mathbb{P} }

\newcommand{\dist}{\mathrm{dist}}
\newcommand{\supp}{\mathrm{Supp}}
\newcommand{\unif}{\mathrm{Unif}}
\newcommand{\ve}{\mathrm{vec}}

\newif\ifsup\supfalse
\suptrue

\DeclareMathOperator*{\argmax}{argmax}
\DeclareMathOperator*{\argmin}{argmin}

\author{%
  Xiangcheng Zhang\thanks{Equal contribution.} \\
  Harvard University\\
  \texttt{xiangchengzhang@fas.harvard.edu}
  \And
  Runhan Huang\footnotemark[1] \\
  Harvard University \\
  \texttt{rhuang@fas.harvard.edu}
  \And
  Yilun Du\\
  Harvard University\\
  \texttt{ydu@seas.harvard.edu}
}

\begin{document}

\maketitle
\lhead{Preprint}

\begin{abstract}
\looseness=-1
  Building generalizable robot agents for diverse applications remains a fundamental challenge. While imitation learning-based policies can perform well in familiar training environments, they often struggle to generalize to novel scenes, layouts, and task compositions. To this end, we present \textbf{World Action Planner}, an agentic robot planning system in which the agent searches for and composes executable action plans through imagination with an action-conditioned world model. The search proceeds in a coarse-to-fine manner. First, the agent performs global action optimization by reasoning over imagined world-model rollouts to identify potential failures and refine the proposed action plan. It then performs local action search, comparing the imagined future outcomes of neighboring candidates to select the best action for execution. Across compositional long-horizon tasks, novel object layouts, and real-robot planning on novel tasks without expert demonstrations, World Action Planner consistently outperforms state-of-the-art end-to-end generalist policy models and VLM planners, demonstrating the effectiveness of world-model-based action search for generalizable robot decision making. Qualitative results and videos are available at \href{https://worldactionplanner.github.io/}{worldactionplanner.github.io}.
\end{abstract}

\section{Introduction}
\looseness=-1
Developing generalizable robotic agents capable of performing diverse tasks across heterogeneous environments remains a fundamental challenge in AI. Recent progress has been driven by end-to-end (E2E) imitation learning-based policies that leverage pretrained Vision-Language Models (VLMs) \citep{intelligence2025pi_} or video generation models \citep{ye2026world} as backbones. However, these approaches are often fundamentally limited by the scope of the demonstration data available during training. For instance, a policy trained only on individual pick-and-place trajectories may fail to navigate between objects when executing composed, long-horizon tasks. Furthermore, E2E models can overfit to spurious motion patterns in demonstrations; if the training data consistently depicts grasping at a specific coordinate, the policy may continue targeting that position even after the object has moved. Consequently, these models struggle in novel scenarios that require systematic reasoning or flexible planning beyond the trajectory distributions observed during training.

\looseness=-1
In this work, we depart from the end-to-end (E2E) imitation learning paradigm and instead propose an agentic planning framework centered on a generalizable world model that can faithfully imagine the outcomes of diverse robot actions. When facing a new task, the agent can synthesize novel action plans through search and composition within world-model imagination, enabling behaviors that remain challenging for end-to-end policies tightly coupled to previous expert trajectories. Instead of requiring demonstrations from the same task and environment, we only need to learn reusable dynamics with a world model, potentially from data collected across different tasks and environments, and then search directly for executable robot actions through imagination under the task objective in the current environment. This decouples decision making from the expert demonstrations seen during training and enables systematic generalization to novel states, tasks, and environments.

\looseness=-1
To this end, we propose \textbf{World Action Planner} (Fig.~\ref{fig:teaser}), a world-model-based robot planning system designed to solve new tasks in novel scenes and layouts. We train a generalizable video world model conditioned on robot pose images (Fig.~\ref{fig:video demo}), enabling precise conditioning on robot actions and high-fidelity prediction of physical interactions. By imagining future action rollouts, the world model reveals their physical consequences before execution, allowing an out-of-the-box VLM to evaluate predicted outcomes and provide feedback for action optimization. To efficiently search the action space, we introduce a coarse-to-fine action search hierarchy that combines agent-driven global action optimization with physically grounded local action candidate search, both performed by reasoning over imagined world-model rollouts (Fig.~\ref{fig:global-local-search}).

\begin{figure}[ht]
    \centering
\includegraphics[width=\linewidth]{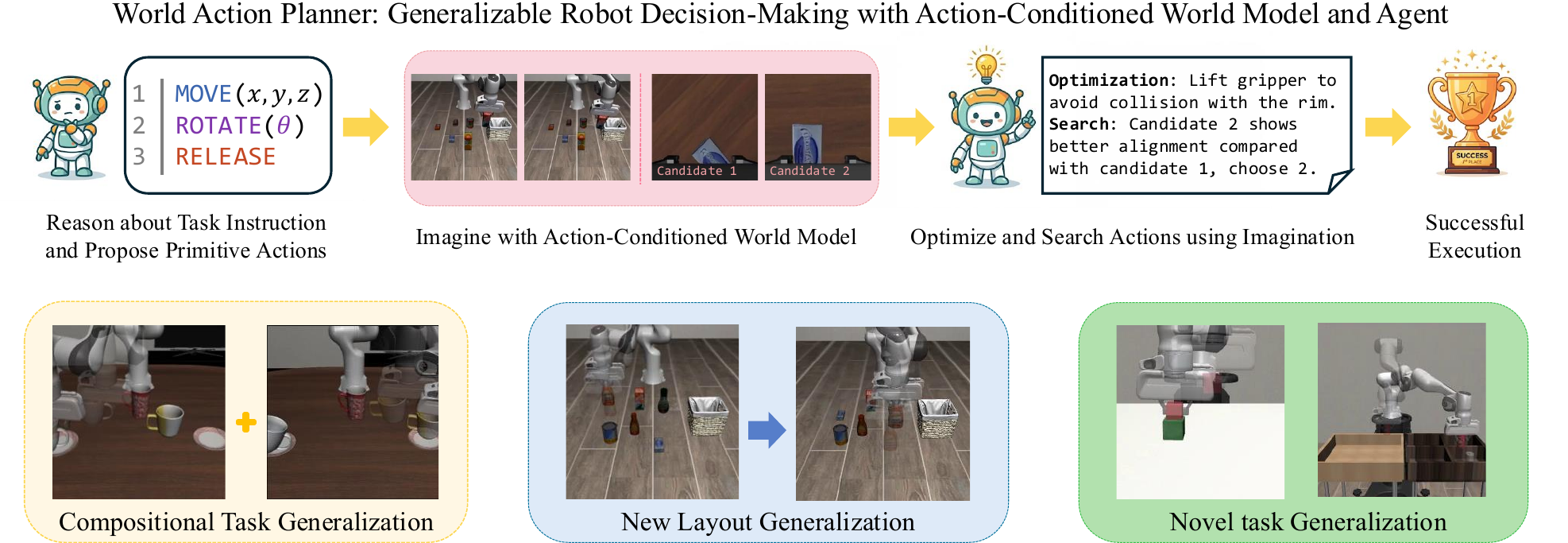}
    \caption{We introduce \textbf{World Action Planner}, a planning system in which the agent reasons about and refines its action plans via world model imagination, generalizing to novel scenarios and solving new tasks at test time.}
    \label{fig:teaser}
    % \vspace{-10pt} 
\end{figure}

\looseness=-1
We demonstrate the generalization capability of our method across compositional tasks, novel layouts, and novel task scenarios without expert demonstration in both simulation and real-world environments. In compositional long-horizon tasks, our method successfully handles transitions between subtasks by proposing and refining maneuvers through world-model imagination, whereas end-to-end models such as VLAs often stagnate after completing the first subtask. In spatial generalization settings with modified object positions, our method identifies the updated target locations and completes the task, while imitation learning policies tend to reach toward object coordinates observed during training. By reasoning over imagined world-model rollouts, global optimization identifies and avoids physical risks such as collisions, significantly improving safety and success rates in real-robot settings. Local action search further evaluates neighboring candidates to identify precise actions for fine-grained manipulation tasks such as mug grasping and cube stacking. Together, these components form a coarse-to-fine action search strategy that achieves substantially greater planning efficiency than conventional action-sampling methods such as MPPI \citep{hansen2023td} and Best-of-N \citep{cui2026dreamsteer}.

\looseness=-1
We summarize our main contributions as follows: (i) we propose World Action Planner, an agentic planning system that proposes and refines robot actions through a physically grounded coarse-to-fine search over imagination from a generalizable robot pose-image conditioned world model; (ii) we provide theoretical insights into the multi-task generalization advantages of model-based planning over imitation learning; and (iii) we demonstrate that our method significantly outperforms SOTA policies and VLM planners across compositional tasks, novel layouts, and novel task generalization scenarios.

% \vspace{-2pt}
% We summarize our main contributions as follows:
% \vspace{-5pt}
% \looseness=-1
% \begin{itemize}[leftmargin=2em, itemsep=0.1em]
%     \item[\textbf{(i)}] We introduce a generalizable, \textit{pose-image-conditioned} robot world model and demonstrate strong generalization to novel actions and scenes across diverse robot embodiments.
%     \item[\textbf{(ii)}] We provide theoretical analysis establishing the information-theoretic generalization benefits of model-based planning over imitation learning for both tabular and linear MDPs.
%     \item[\textbf{(iii)}] We propose \textbf{World Action Planner}, an agentic planning system that proposes and refines robot actions through a physically grounded coarse-to-fine search process over imagined world-model rollouts. Our results demonstrate that the method significantly outperforms state-of-the-art policies across compositional tasks, novel layouts, and zero-shot generalization scenarios.
% \end{itemize}
% \vspace{-8pt}
\section{Related Works}
% \vspace{-5pt}
Here, we review the most closely related works on planning with world models and VLMs. For an extended discussion on controllable world models and generalist policies, see Appendix~\ref{app:related works}. 

To use world-model imagination for decision making, prior MPC-based methods \citep{hansen2022temporal,assran2025v,jain2025smooth} typically rely on sampling-based optimization such as MPPI, while recent works \citep{sharma2026world,jia2026dreamplan} instead train policies directly in imagination. To leverage the semantic knowledge and reasoning capabilities of VLMs, many methods use them for high-level planning while relying on separate grounding or control modules to produce executable actions \citep{liu2024moka,duan2024manipulate,feng2025reflective}, whereas our method searches directly over executable robot controls grounded by the world model. More recent approaches \citep{zhao2024vlmpc,ning2025prompting} combine predictive models or digital twins with VLM-based evaluation, but still rely on inefficient CEM sampling. Moreover, \citet{zhao2024vlmpc} conditions its world model on low-dimensional actions rather than explicit robot pose images, potentially limiting generalization, while \citet{ning2025prompting} depends on VLM-based real-to-sim reconstruction whose accuracy may constrain planning. Concurrent work \citet{cui2026dreamsteer} improves a base VLA by imagining and ranking $K$ policy-sampled action candidates. In contrast, our method enables direct synthesis of actions beyond the policy proposal distribution while generalizing to novel tasks and environments.
% \vspace{-5pt}
\section{Methods}
% \vspace{-5pt}
Here we describe our world model with robot pose-image conditioning in Section~\ref{sec:method video} and the world action planner with our agentic action search in Section~\ref{sec:method plan}.

% \vspace{-5pt}
\subsection{Pose-Image Conditioning for Action-Controlled Robot World Modeling}\label{sec:method video}
\begin{figure}[htb]
% \vspace{-5pt}
    \centering
    \includegraphics[width=\linewidth]{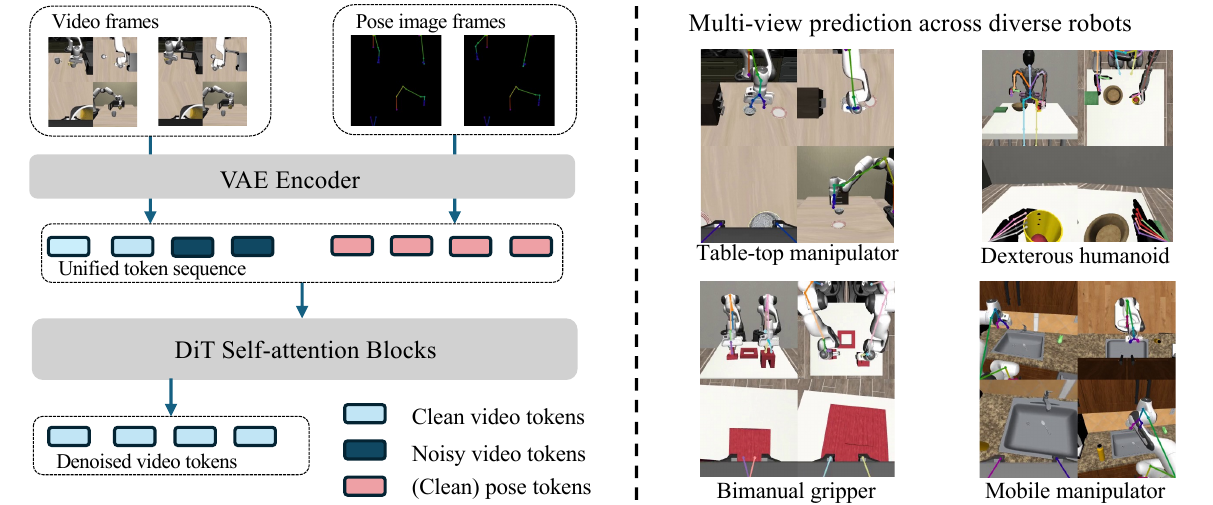}
    \caption{\textbf{Illustration of our pose-image conditioned world model}}
    \label{fig:video demo}
    % \vspace{-5pt}
\end{figure}
% \vspace{-3pt}
\looseness=-1
Prior robot action-conditioned video diffusion models typically condition on low-dimensional action vectors through AdaLN-Zero modulation \citep{quevedo2025worldgym,zhu2025irasim} or cross-attention \citep{guo2025ctrl}. However, such low-dimensional action representations can entangle robot kinematics with environment dynamics, making generalization to novel actions difficult. In our approach, we first compute the future robot joint configurations induced by the actions and render them as pose-skeleton images from the corresponding camera viewpoints. This representation makes the robot's future motion explicit in the visual input, allowing the world model to focus on predicting how the environment responds to that motion. As a result, the model can reuse learned object-interaction dynamics while generalizing to novel robot motions that were not observed during training. More details about pose-image conditioning are provided in Appendix~\ref{app:pose image}.

\looseness=-1
Inspired by \citet{tan2025ominicontrol}, we encode the pose image frames using the VAE of the video model and concatenate the resulting pose image tokens with the video tokens into a unified sequence. We adopt multi-view prediction by concatenating third-person and wrist-view camera feeds into a grid for each video and pose-image frame \citep{guo2025ctrl}, so that the model can infer the relative 3D position of each robot joint from multi-view images. We train the model using diffusion-forcing \citep{song2025history,chen2025large} with the flow-matching objective \citep{lipman2022flow}, where we add independent random noise to the history video tokens and uniform noise to the future video tokens, while keeping the pose-image tokens noise-free for conditioning.

% \vspace{-2pt}
\subsection{World Action Planner}\label{sec:method plan}
% \vspace{-2pt}
\looseness=-1
To achieve generalizable robotic decision-making, we develop World Action Planner (WAP), an agentic planning framework that generalizes to new tasks and environments. Rather than end-to-end imitation learning, WAP combines agent-driven global trajectory optimization with physically grounded local refinement to search and compose robot actions. We introduce the agentic action proposal and policy prior composition in Section~\ref{sec:action proposal}, and focus on our core action search strategy in Section~\ref{sec:action search}. The complete procedure is given in Alg.~\ref{alg:mpc}, with additional details in Appendix~\ref{app:planner detail}.

\begin{algorithm}[htb]
% % Keep the reduced bottom spacing attached to this float when it moves.
% \makeatletter
% \let\wapalgorithmfloatpost\@fs@post
% \def\@fs@post{\wapalgorithmfloatpost\vspace{-10pt}}
% \makeatother
\caption{World Action Planner (WAP)}
\label{alg:mpc}
\begin{algorithmic}
\STATE \textbf{Require:} VLM \texttt{Agent}, robot controller $\phi$, world model \texttt{WM}, optional policy $\pi$
\STATE \textbf{Input:} Task description $\ell$, initial state $s$, environment \texttt{env}

\WHILE{not \texttt{done}}

\item[]\vspace{1pt} $\triangleright$ \textbf{Agent Action Proposal} \hfill
\STATE Propose primitive action goals using the VLM agent:
$g=\texttt{Agent.Proposal}(s,\ell)$
\STATE Translate the proposed goal into robot actions:
$a=\phi(g,s_{\text{proprio}})$

\item[]\vspace{1pt} $\triangleright$ \textbf{Global Optimization} \hfill
\STATE Imagine the proposed action using the world model:
$\hat{s}_{\texttt{next}}=\texttt{WM}(s,a)$
\STATE Obtain agent feedback:
$\Delta g=\texttt{Agent.Eval}(\hat{s}_{\texttt{next}},\ell)$
\STATE Update the goal and robot action:
$g\leftarrow g+\Delta g$, $a=\phi(g,s_{\text{proprio}})$

\item[]\vspace{1pt} $\triangleright$ \textbf{Local Search and Refinement} \hfill
\STATE Generate local goal candidates:
$g_1,g_2,\ldots,g_N=\texttt{GridSearch}(g)$
\STATE Translate goal candidates into robot actions:
$a_i=\phi(g_i,s_{\text{proprio}})$, $i=1,2,\ldots,N$
\STATE Imagine the candidate actions:
$\hat{s}_i=\texttt{WM}(s,a_i)$, $i=1,2,\ldots,N$

\IF{use policy}
\STATE Imagine the policy rollout after each candidate action:
$\hat{s}_i=\texttt{WM}\left(\hat{s}_i,\pi(\hat{s}_i)\right)$, $i=1,2,\ldots,N$
\ENDIF

\STATE Select the best candidate using the VLM agent:
$i^*=\texttt{Agent.Rank}(\hat{s}_1,\hat{s}_2,\ldots,\hat{s}_N,\ell)$

\item[]\vspace{1pt} $\triangleright$ \textbf{Execute the Actions} \hfill
\STATE Execute the selected action:
$s,\texttt{done}=\texttt{env.Step}(s,a_{i^*})$

\IF{use policy}
\STATE Roll out the policy following the selected action:
$s,\texttt{done}=\texttt{env.Step}(s,\pi(s))$
\ENDIF

\ENDWHILE

\end{algorithmic}
\end{algorithm}

% \vspace{-3pt}
\subsubsection{Action Proposals and Policy Priors}\label{sec:action proposal}
% \vspace{-3pt}
\looseness=-1
When generalizing to new tasks and environments, we leverage the broad world knowledge of foundation VLMs to propose action primitives in settings where demonstration data is unavailable. The VLM generates a sequence of primitive actions, such as \texttt{MOVE}, \texttt{ROTATE}, \texttt{GRASP}, and \texttt{RELEASE}. For each \texttt{MOVE} action, the VLM identifies the target position across \textit{multi-view} images. We then triangulate the corresponding 2D pixel coordinates across views to recover the target position in 3D space, eliminating the need for explicit depth information. We employ a low-dimensional robot controller to convert the target EEF pose specified by each goal into a sequence of executable low-level actions, where the goals are being optimized and refined in the subsequent search process.

When policy priors are available for local manipulation, such as grasping, the agent flexibly composes them in states where the policy remains in distribution. For example, when an object is located at a novel position, executing the grasping policy end-to-end may fail. Instead, the agent first navigates to the vicinity of the target object and then searches for states from which the local grasping policy is likely to succeed by imagining future policy rollouts with the world model. Treating learned policies as modular action priors and composing them within the planning process therefore provides a generalization pathway for extending their capabilities beyond the expert demonstrations.

% \vspace{-3pt}
\subsubsection{Agentic Action Search with World Model}\label{sec:action search}
% \vspace{-3pt}

\looseness=-1
Here we discuss our action search strategy. Searching for successful robot actions in manipulation is challenging for two main reasons: the action space is high-dimensional and continuous, while task success is often sparse and concentrated around a small set of feasible poses, providing little continuous guidance for optimization. Generic sampling methods such as Best-of-N and CEM are therefore inefficient in this setting: they explore large action spaces through weakly structured proposal distributions, causing many world-model evaluations to be spent on physically implausible candidates, redundant variations, or actions that repeat the same underlying failure. Instead, we organize search into a coarse-to-fine hierarchy that operates in a low-dimensional goal space. Guided global optimization first uses agentic feedback to correct large, semantically meaningful errors in the proposed trajectory, while physically grounded local search subsequently concentrates computation on a compact neighborhood of plausible actions for fine-grained refinement. This procedure is illustrated in Figure~\ref{fig:global-local-search}.
\begin{figure}[htb]
    \centering
    % \vspace{-5pt}
    \includegraphics[width=\linewidth]{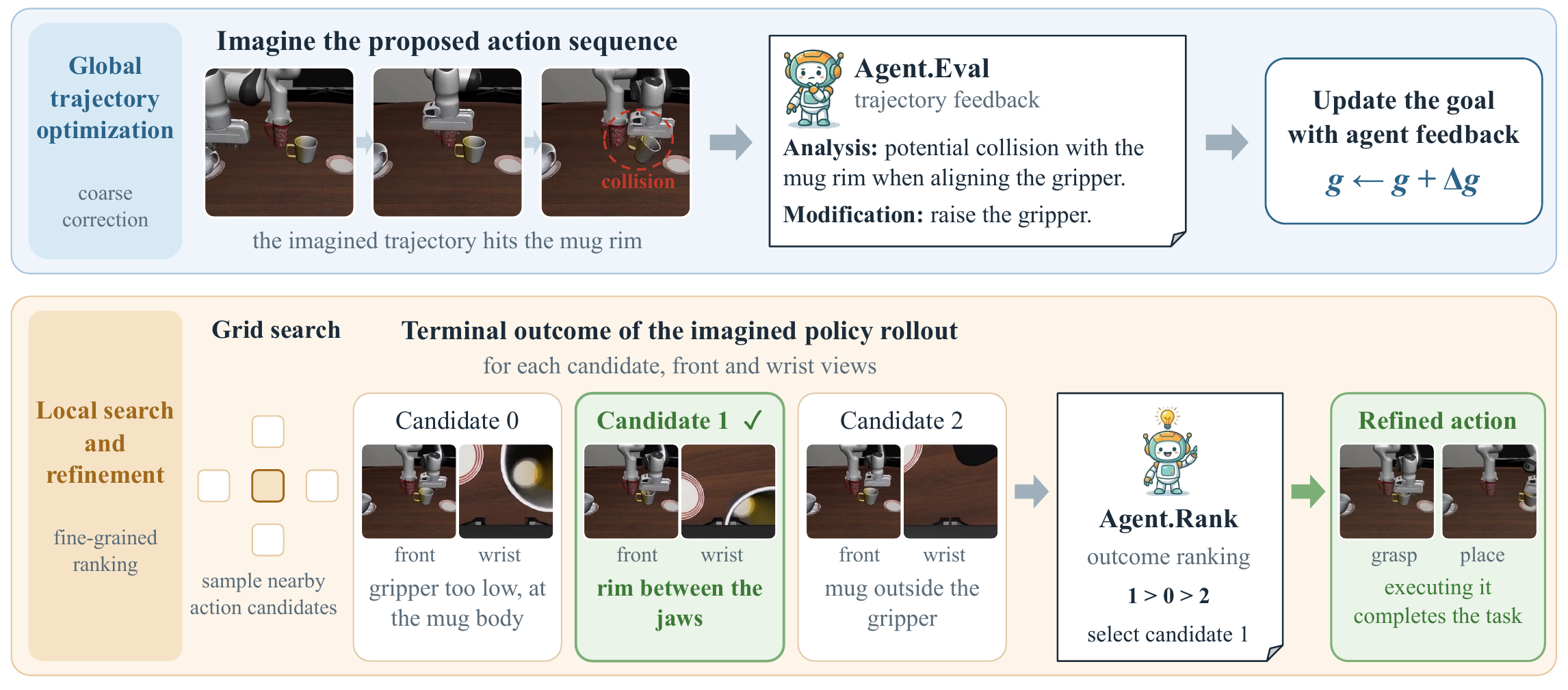}
    \caption{\textbf{Coarse-to-fine action optimization in WAP.} Global trajectory optimization imagines the proposed action sequence with the action-conditioned world model, while the VLM provides feedback to improve safety and task success. Local search then performs a grid search over neighboring actions, imagines the future outcome of each candidate, and selects the optimal action.}
    \label{fig:global-local-search}
    % \vspace{-10pt}
\end{figure}

\paragraph{Global Trajectory Optimization}
\looseness=-1
The initially proposed action sequence may contain large-scale errors that are difficult to avoid from the current observation alone, such as violating physical constraints or failing to account for obstacles in a novel environment. For instance, a naive \texttt{MOVE} command between two objects may cause the gripper to collide with an intervening obstacle. Although such failures may be difficult to anticipate without task-specific motion priors, they often become immediately apparent once the proposed trajectory is visualized through world-model imagination. Rather than exploring random perturbations around an erroneous proposal, the agent directly diagnoses the failure from the imagined visual rollout and applies a targeted modification to the corresponding goal action. For example, if the rollout reveals a potential collision between the gripper and a basket rim, the agent may raise the target pose to provide sufficient clearance. Similarly, if an object is predicted to be dropped behind a box, the agent may shift the goal position forward. This targeted refinement allows the planner to use a single imagined trajectory to correct major trajectory-level errors.

\looseness=-1
\paragraph{Local Search and Refinement}
Global optimization is effective at correcting conspicuous trajectory-level errors, but precise manipulation often depends on subtle geometric differences among nearby candidate poses, which can be difficult to evaluate directly because informative rewards may emerge only after subsequent execution. We therefore assess each candidate through its world-model-imagined downstream outcome. Specifically, we perform grid search over neighboring actions and states in the low-dimensional goal space and use the action-conditioned world model to generate a \emph{future policy rollout} from each candidate. These rollouts reveal the downstream consequences of the local candidates, allowing the VLM to rank them according to their imagined task outcomes rather than ambiguous intermediate states. For example, when refining a pre-grasp pose, we evaluate neighboring positions by imagining the subsequent grasping policy from each candidate and select the pose whose final imagined frame most clearly indicates a successful grasp. This outcome-based local search focuses world-model evaluations on a small set of physically meaningful alternatives and enables fine-grained action refinement without requiring dense rewards or reliable intermediate value estimates.

% \vspace{-8pt}
\section{Theoretical Insights}\label{sec:theory main}
% \vspace{-8pt}
\looseness=-1
We ask when learned-model planning transfers across tasks more effectively than imitation, and what determines the quality of WAP's selected action. We state the main results informally here and defer the full settings, assumptions, and proofs to Appendix~\ref{app:theory}.

% \vspace{-5pt}
\subsection{When Does Model-based Planning Outperform Imitation Learning}
% \vspace{-4pt}
\looseness=-1
\paragraph{Setting} Tasks form a contextual MDP \citep{hallak2015contextual}, indexed by $c\in\cC$. They share the state space $\cS$, action space $\cA$, horizon $H$, initial distribution $\rho$, and transition kernel $\PP$, but have different rewards $r_c$. Each learner receives $K$ training episodes and, as is standard when assembling a multi-task dataset, fixes how many go to each task before collecting any of them. The imitation learner observes expert actions, but not rewards, and fits a task-conditioned policy. The model-based learner never sees an expert action; it chooses its own exploration policy, estimates the shared dynamics and, when needed, the reward model, then plans for the test task in the estimated MDP. We measure the suboptimality $\EE_{s_1\sim\rho}[V^*(s_1;c)-V^{\pi_c}(s_1;c)]$ per task and on average across tasks.

\paragraph{Tabular Setting} Let $\cS$, $\cA$, and $\cC$ be finite, with unknown dynamics and a known reward for each task. This resembles WAP: the world model must learn the dynamics, while a VLM can judge success from a terminal state.
% \vspace{-2pt}
\begin{theorem}[Tabular, informal]\label{thm:main tabular}
With $K$ episodes, \textbf{(i)} a model-based planner attains $\Tilde{\mathcal{O}}(K^{-1/2})$ suboptimality on every task, independent of the number of tasks $|\cC|$; yet \textbf{(ii)} for any imitation learner, there exists a contextual MDP instance where the average suboptimality $\Omega\big(\min\{H,\,H|\cS||\cC|/K\}\big)$.
\end{theorem}

The $|\cC|$ factor reflects what each sample reveals: an explored transition constrains the shared kernel $\PP$ and so informs planning for every task, while an expert action labels only its own, leaving roughly $K/|\cC|$ demos per task. Two limits apply. The gain is in the task count rather than the rate, so planning wins only once $|\cC|\gtrsim H^4|\cS||\cA|/\epsilon$; and it needs rewards evaluable on unseen tasks, the VLM's role in WAP. A planner forced to estimate each $r_c$ separately pays the same penalty (Appendix~\ref{app:theory}).

\paragraph{Linear Function Approximation} In a linear MDP \citep{jin2023provably}, known features $\phi(s,a)\in\RR^d$ satisfy $\PP(\cdot\mid s,a)=\innerproduct{\phi(s,a)}{\bm{\mu}(\cdot)}$. We also assume known task features $\psi(c)\in\RR^{d'}$ and unknown matrices $\Theta_h$, with $\EE[r_h(s,a;c)]=\phi(s,a)^\top\Theta_h\psi(c)$. Features are normalized and $\|\Theta_h\|_F\le\sqrt{dd'}$. Here neither learner knows the rewards, so both must transfer through a learned representation. For imitation, we consider classifier policies $\pi_f(s;c)=\argmax_a f(\phi(s,a),\psi(c))$ \citep{rajaraman2021value}; their ability to transfer depends on the complexity of the shared score $f$ as a function of the task feature.
% \vspace{-2pt}
\begin{theorem}[Linear, informal]\label{thm:main linear}
\textbf{(i)} A planner that learns the shared dynamics and bilinear reward model attains $\Tilde{\mathcal{O}}((\kappa\sqrt{K})^{-1})$ suboptimality for every task, where $\kappa$ lower-bounds the smallest eigenvalue of the feature covariance achievable by some exploration policy. \textbf{(ii)} Yet some such instance has an optimal policy that no classifier policy can express when $f$ has degree below $|\cA|$ in the task feature, regardless of $K$.
\end{theorem}
% \vspace{-3pt}
\looseness=-1
In the construction for \textbf{(ii)}, the reward remains bilinear, but the optimal action switches $2|\cA|$ times as the task changes; a polynomial score of degree below $|\cA|$ cannot change sign that often. This is a representational gap, not a sample-size gap. The planner learns a description of the dynamics and reward that does not grow with $|\cA|$, then solves the $\argmax$ at test time; a policy must encode that $\argmax$ directly (Remark~\ref{rem:linear comparison} in Appendix~\ref{app:theory}).

% \vspace{-5pt}
\subsection{Error Analysis for WAP}
% \vspace{-3pt}
\looseness=-1
For one planning step, let $a^*$ maximize the action value $Q$ over the low-dimensional goal space, and let $\hat{a}$ denote WAP's selected action. \texttt{Agent.Proposal} outputs $a_0$ at distance $\epsilon_{\texttt{Agent.Proposal}}$ from $a^*$, \texttt{Agent.Eval} corrects it within radius $h^{\mathrm{global}}$, and \texttt{Agent.Rank} selects among $N$ imagined candidates within radius $h^{\mathrm{local}}$. Under local regularity of $Q$ and a covering condition on the candidates, Appendix~\ref{sec:wap-error-bound} gives
\begin{align*}
Q(a^*)-Q(\hat{a})\lesssim{}&\underbrace{h^{\mathrm{local}}N^{-1/d_a}}_{\text{local resolution}}
+\underbrace{\big[\epsilon_{\texttt{Agent.Proposal}}-h^{\mathrm{global}}-h^{\mathrm{local}}\big]_+}_{\text{proposal out of reach}}\\
&+\underbrace{\epsilon_{\texttt{Agent.Eval}}+\epsilon_{\mathrm{WM}}+\epsilon_{\texttt{Agent.Rank}}}_{\text{component errors}},
\end{align*}
where $d_a$ is the goal-space dimension. The $\epsilon$ terms are the errors of the learned components, defined rather than assumed small, so this is an error decomposition for a single step rather than a performance guarantee; Appendix~\ref{sec:wap-error-bound} shows that no term can be dropped from a general bound over the assumption class.

\looseness=-1
Three design consequences follow. First, resolution improves only as $N^{-1/d_a}$, motivating WAP's search in a low-dimensional goal space rather than over raw action chunks. Second, the proposal term vanishes once $\epsilon_{\texttt{Agent.Proposal}}\leq h^{\mathrm{global}}+h^{\mathrm{local}}$ so the VLM proposal need only lie within reach of the two search stages. Third, $N$ does not enter the component errors, so the bound stops improving once they dominate, which is a statement about the bound rather than a proven error floor for the deployed system. Appendix~\ref{app:world model qualitative results} accordingly tests a necessary condition, whether imagined rollouts preserve the agent's candidate ranking, rather than estimating those errors.

% \vspace{-3pt}
\section{Experiments}\label{sec:exp}
% \vspace{-3pt}
\looseness=-1
Here, we provide comprehensive experiments of both action-conditioned world modeling with our robot pose-image conditioning in Section~\ref{sec:video gen exp}, and generalizable robot planning with our World Action Planner in Section~\ref{sec:planning exp}. Experiment details and qualitative examples are in Appendix~\ref{app:exp detail}. 
% \vspace{-3pt}
\subsection{Action-Conditioned World Modeling}\label{sec:video gen exp}
% \vspace{-3pt}
\looseness=-1
We evaluate our pose-image conditioned world model by addressing three primary questions: 
\textbf{Q1:} Can we accurately generate multi-view future images conditioned on robot pose-images? 
\textbf{Q2:} Does pose-image conditioning improve generalization compared with low-dimensional action conditioning approaches? 
\textbf{Q3:} Can the model learn dynamics across diverse robot embodiments and action spaces?

\paragraph{Setup} 
We validate our approach across four simulation suites: LIBERO \citep{liu2023libero}, Robocasa \citep{nasiriany2024robocasa}, MimicGen \citep{mandlekar2023mimicgen}, and DexMimicGen \citep{jiang2025dexmimicgen}. To generalize to exploratory actions, we perturb demonstration trajectories with Gaussian noise following the MPPI algorithm \citep{williams2017model}, a standard for world-model-based control \citep{hansen2023td, jain2025smooth}. Each task uses 100 training trajectories (including perturbed versions) and 10 held-out evaluation trajectories. We employ \texttt{Wan}-T2V-1.3B \citep{wan2025wan} as the backbone, finetuning for 10K steps with a global batch size of 64. Models process 21 history frames at 7 FPS to predict 20 future frames at 20 FPS. Input consists of four camera views concatenated into a $2\times2$ grid ($224$px per view). Details in Appendix~\ref{app:wm exp setup}.

\paragraph{Baselines and Evaluation} 
\looseness=-1
We compare against state-of-the-art diffusion-based world models: WPE \citep{quevedo2025worldgym} and IRA-Sim \citep{zhu2025irasim}, which utilize AdaLN-Zero modulation \citep{peebles2023scalable} for the action embeddings, and Ctrl-World \citep{guo2025ctrl} employing cross-attention with low-dimensional action tokens. For cross-embodiment methods, we adapt VLA architectures including unified action spaces \citep{liu2024rdt}, embodiment-aware encoders \citep{bjorck2025gr00t}, and soft prompting \citep{zheng2025x}. All baselines use the same \texttt{Wan}-T2V-1.3B backbone and model configurations, but are trained for 20K steps (twice our method's duration) to ensure a competitive comparison. Evaluation metrics include LPIPS \citep{zhang2018unreasonable} and PSNR \citep{hore2010image}, as they best align with human preferences for controllable world modeling \citep{zhu2025irasim}.

\looseness=-1
We first evaluate in-distribution prediction (\textbf{Q1}) with a 7-DoF Franka-Panda (LIBERO-90) and a bi-manual dexterous robot (DexMimicGen), achieving a consistent average $11.3\%$ relative improvement in PSNR and LPIPS over the strongest baseline in each column (Table~\ref{tab:single-robot}). To assess generalization (\textbf{Q2}), we conduct zero-shot evaluation on LIBERO-Long and LIBERO-Spatial tasks, alongside few-shot finetuning on distinct hardware (Sawyer, IIWA, and UR5e) within the MimicGen-Robot, using the LIBERO-90 model. Across these novel trajectories, our pose-image conditioning yields a $16.8\%$ average relative improvement on the same measure, confirming its superior generalization capability. Qualitative results in Appendix~\ref{app:world model qualitative results}.
\begin{table}[t]
    \centering
    \caption{\textbf{Results for single-embodiment action-conditioned world modeling.} Each cell reports PSNR$\uparrow$\,/\,LPIPS$\downarrow$ averaged across the predicted frames. Methods are compared on two in-distribution suites (LIBERO-90, DexMimicGen) and three generalization settings: zero-shot transfer (LIBERO-Long, LIBERO-Spatial) and few-shot finetuning on new hardware (MimicGen-Robot). Best in \textbf{bold}, second best \underline{underlined}.}
    \label{tab:single-robot}
    \scriptsize
    \setlength{\tabcolsep}{1.5pt}
    \renewcommand{\arraystretch}{1.1}
    \begin{tabular*}{\linewidth}{@{\extracolsep{\fill}}lcccccccccc@{}}
        \toprule
        \multirow{2}{*}{Method} & \multicolumn{2}{c}{LIBERO-90} & \multicolumn{2}{c}{DexMimicGen} & \multicolumn{2}{c}{LIBERO-Long} & \multicolumn{2}{c}{LIBERO-Spatial} & \multicolumn{2}{c}{MimicGen-Robot} \\
        \cmidrule(lr){2-3}\cmidrule(lr){4-5}\cmidrule(lr){6-7}\cmidrule(lr){8-9}\cmidrule(l){10-11}
        & wrist & third & wrist & third & wrist & third & wrist & third & wrist & third \\
        \midrule
        WPE & 15.01/.339 & 19.05/.121 & 15.52/.298 & \underline{19.38}/.142 & 14.05/.364 & 17.97/.157 & 14.72/.379 & 19.50/.156 & 16.38/.193 & 17.10/.130 \\
        IRA-Sim & 14.94/.343 & 19.73/.122 & 15.43/.293 & 18.42/.142 & 14.00/.366 & 17.94/.160 & 14.76/.384 & 19.52/.153 & 16.37/.195 & 17.15/.130 \\
        Ctrl-World vel.\tablefootnote{The original implementation in Ctrl-World \citep{guo2025ctrl} uses end-effector position control, and all other baselines use end-effector velocity control.} & 15.54/.318 & 20.10/.109 & \underline{15.56}/\underline{.276} & 19.19/\underline{.135} & \underline{14.63}/\underline{.345} & 18.26/.139 & \underline{15.50}/\underline{.361} & \underline{19.76}/\underline{.137} & 16.80/.171 & 17.43/\underline{.124} \\
        Ctrl-World pos. & \underline{15.56}/\underline{.317} & \underline{20.13}/\underline{.108} & 15.55/.277 & 19.18/.136 & 14.61/.348 & \underline{18.27}/\underline{.138} & 14.23/.403 & 19.28/.154 & \underline{16.82}/\underline{.169} & \underline{17.44}/.126 \\
        \textbf{Ours} & \textbf{17.02}/\textbf{.286} & \textbf{23.13}/\textbf{.085} & \textbf{16.33}/\textbf{.266} & \textbf{21.18}/\textbf{.112} & \textbf{15.98}/\textbf{.320} & \textbf{22.14}/\textbf{.093} & \textbf{16.75}/\textbf{.322} & \textbf{22.79}/\textbf{.097} & \textbf{18.02}/\textbf{.144} & \textbf{21.40}/\textbf{.095} \\
        % \midrule
        % \textit{Rel. improve. (\%)} & 9.4/9.8 & 14.9/21.3 & 4.9/3.6 & 10.4/17.0 & 9.2/7.2 & 21.2/32.6 & 8.1/10.8 & 15.3/29.2 & 7.1/14.8 & 22.7/23.4 \\
        \bottomrule
    \end{tabular*}
    \vspace{-10pt}
\end{table}

\begin{wraptable}{r}{0.55\linewidth}
    \centering
    \vspace{-5pt}
    \caption{\textbf{Results for cross-embodiment modeling.} We use a mixture of trajectories from multiple robots with distinct action spaces. Our method consistently outperforms existing approaches. Each cell represents PSNR$\uparrow$\,/\,LPIPS$\downarrow$.}
    \label{tab:cross-robot}
    \footnotesize
    \begin{tabular*}{\linewidth}{@{\extracolsep{\fill}}lcc@{}}
        \toprule
        Method & wrist-view & third-view \\
        \midrule
        Unified action space & 14.60 / 0.323 & 17.80 / 0.164 \\
        Embodiment-aware encoder & \underline{14.68} / \underline{0.321} &  17.52 / 0.168 \\
        Soft prompt & 14.66 / 0.323 & \underline{17.83} / \underline{0.162} \\
        \textbf{Ours} & \textbf{15.11} / \textbf{0.308} & \textbf{19.44} / \textbf{0.131} \\
        \bottomrule
    \end{tabular*}
    \vspace{-5pt}
\end{wraptable}

\paragraph{Cross-Embodiment Modeling (Q3)}
\looseness=-1
To evaluate cross-embodiment world modeling, we train on a heterogeneous mixture of RoboCasa, MimicGen, and DexMimicGen, spanning action dimensions of 7 (fixed arm), 12 (mobile Panda), 14 (bi-manual parallel grippers), and 24 (bi-manual dexterous hands). 
% For all baselines, we adopt the cross-attention mechanism from Ctrl-World \citep{guo2025ctrl} as the base action-conditioner, given its strong baseline performance (Table~\ref{tab:single-robot}). 
As shown in Table~\ref{tab:cross-robot}, our method outperforms prior approaches for cross-embodiment training, suggesting that representing actions via robot pose images provides an embodiment-agnostic interface that may be of broader utility for future robot models.

% \vspace{-3pt}
\subsection{World Action Planner}\label{sec:planning exp}
% \vspace{-3pt}
\looseness=-1
In this section, we evaluate our World Action Planner in three generalization scenarios: \textit{compositional task generalization}, \textit{new layout generalization}, and \textit{novel task generalization}. Details including baseline training and in-distribution performance report, method details, additional results and visualizations, and failure analysis are provided in Appendix~\ref{app:planner exp details}.

\looseness=-1
\paragraph{Setup and Baselines} For our method, We employ Gemini 3.0 Flash \citep{pichai2025new} as our default VLM agent, and we train our optional diffusion policy with the same or substantially less expert data compared with baseline methods, detailed in each scenario below. For imitation learning baselines, we use SOTA VLA $\pi_{0.5}$ \citep{intelligence2025pi_} and WAM cosmos-policy \citep{kim2026cosmos}, where we tune the baseline setups to enable \textit{superior in-distribution performance}, as detailed in Appendix~\ref{app:planner exp details}. We also compare against prior world-model-based MPC methods that sample actions from SOTA VLAs and optimize them. Specifically, we adapt SAILOR \citep{jain2025smooth} and GPC-RANK \citep{qi2025strengthening} using the same world model and VLM-based evaluation as WAP, and for DREAMSTEER \citep{cui2026dreamsteer}, we use its original VLAC critic, with action proposals sampled from the finetuned SOTA VLAs as the generalist policy baselines. To isolate the impact of planning with the action-conditioned world model, we compare against a vision-language planner baseline that is identical to WAP with the same VLM agent and policy with additional reasoning without imagination, but executes the actions directly without optimization and search over imagined rollouts.

\begin{table}[htb]
    \centering
    \caption{\textbf{Average success rate (\%) across generalization scenarios.} 
    % Baselines span end-to-end imitation policies ($\pi_{0.5}$, cosmos-policy), world-model policy enhancement (SAILOR, GPC-RANK), and VLM planning without world-model refinement.
    Results are averaged over 5 tasks in compositional task generalization, 6 tasks in new-layout generalization, and 2 tasks in novel task generalization, with 50 trials per task, and over 2 real robot novel tasks with 20 trials per task.}
    \label{tab:libero10}
    \label{tab:libero-object}
    \label{tab:robosuite}
    \small
    \begin{tabular*}{\linewidth}{@{\extracolsep{\fill}}lcccc@{}}
        \toprule
        Method & Compositional task & New layout & Novel task & Real robot novel task \\
        \midrule
        $\pi_{0.5}$ & 9.6 & 1.7 & 10.0 & 0.0 \\
        cosmos-policy & 2.4 & 0.0 & 2.0 & 0.0\\
        SAILOR & 16.4 & 3.7 & 6.0 & 0.0\\
        GPC-RANK & 12.0 & 2.7 & 8.0 & 0.0\\
        DREAMSTEER & 20.0 & 10.7 & 18.0 & 5.0 \\
        Vision-language planner & 42.0 & 37.7 & 40.0 & 10.0\\
        \midrule
        \textbf{World action planner (ours)} & \textbf{72.0} & \textbf{82.0} & \textbf{78.0} & \textbf{65.0} \\
        \bottomrule
    \end{tabular*}
    \vspace{-10pt}
\end{table}

\looseness=-1
\paragraph{Compositional Task Generalization}
In real-world applications, robot policies are often required to generalize to compositional, long-horizon tasks. For example, a robot may be trained to pick and place individual objects, while a task such as tidying a room requires it to execute multiple such placements in a continuous sequence. We select five tasks from LIBERO-Long composed by pairs of tasks from the LIBERO-90 suite, and train our diffusion policy and world model exclusively on the corresponding atomic LIBERO-90 tasks, following the same training setting as the $\pi_{0.5}$ and Cosmos-Policy baselines. When transferring to these newly composed tasks, we observe a distinct failure mode in end-to-end VLA models: although they often complete the first sub-task successfully, they stagnate immediately afterward and produce near no-op actions. This behavior arises because the policy must transition from the terminal state of the first sub-task to the initial configuration required by the next, yet such transitions are absent from the training demonstrations. Our World Action Planner overcomes this limitation by using the world model to plan transition motions that move the robot into states from which the policy can successfully complete the subsequent sub-task. Additional details are provided in Appendix~\ref{app: libero long details}, with per-task results in Table~\ref{tab:app compositional full} and qualitative examples in Fig.~\ref{fig:libero10}.

\looseness=-1
\paragraph{New Layout Generalization} A generalizable robot must adapt to novel object configurations at test time despite training on a limited set of layouts. We evaluate this capability on six LIBERO-Object tasks with substantial test-time shifts in both target and distractor positions. Using the official $\pi_{0.5}$ and Cosmos-Policy checkpoints as baselines, we observe that generalist VLAs and WAMs often fail to ground instructions in the modified scene and instead fall back to motion priors from the training layout. This leads to failures such as grasping distractors or repeatedly attempting to grasp empty space near the original target coordinates (Fig.~\ref{fig:app libero object pi},~\ref{fig:app cosmos libero object}). For WAP, we finetune the diffusion policy and world model using demonstrations from the \textit{original layout only}. To highlight sample efficiency, we train the policy with just \textit{5 expert demonstrations} per task, compared with more than 40 used by the baseline checkpoints. Although this few-shot policy fails end-to-end, it remains effective for local grasping under modified layouts. WAP uses the world-model planner to navigate to suitable manipulation states and plan grasping and placing motions, decoupling local manipulation from global motion planning. This enables successful task completion with substantially less expert data. Additional details are provided in Appendix~\ref{app:details libero object}, with per-task results in Table~\ref{tab:app new layout full} and qualitative examples in Fig.~\ref{fig:libero object}.

\looseness=-1
\paragraph{Novel Task Generalization}
We evaluate WAP without a learned task policy or expert demonstrations on the \texttt{PickPlaceCan} and \texttt{StackCube} tasks from Robosuite \citep{zhu2020robosuite}, using only hard-coded gripper-closing and release primitives. At test time, we train the world model on 50 exploratory trajectories collected by the VLM in the target environment. We evaluate policy-based baselines both without target-environment finetuning and after finetuning on the same 50 trajectories; Table~\ref{tab:libero10} reports the later, while Table~\ref{tab:app zero shot full} reports both regimes. WAP performs best in both, with the largest margin on the precision-intensive \texttt{StackCube} task, showing the benefit of refining coarse VLM proposals through action-conditioned imagination. Additional details are provided in Appendix~\ref{app:robosuite details}, with qualitative examples in Fig.~\ref{fig:robosuite}.

\looseness=-1
\paragraph{Real Robot Novel Task Generalization}
We evaluate WAP on a real Franka Panda robot for placing either a green or purple can into a white basket. Following the novel task setup above, we train the world model at test time using 10 \textit{failed} explorations of placing the green can using the VLM and use them to finetune our world model. The resulting real-robot execution is shown in Fig.~\ref{fig:real-robot}. As shown in Fig.~\ref{fig:world-model-sucess}, our world model is able to accurately imagine O.O.D successful placement maneuvers never seen in training. Additional details are provided in Appendix~\ref{app:real robot}, with per-task results in Table~\ref{tab:app real robot full}.

\begin{figure}[h!]
    \centering
    \includegraphics[width=\linewidth]{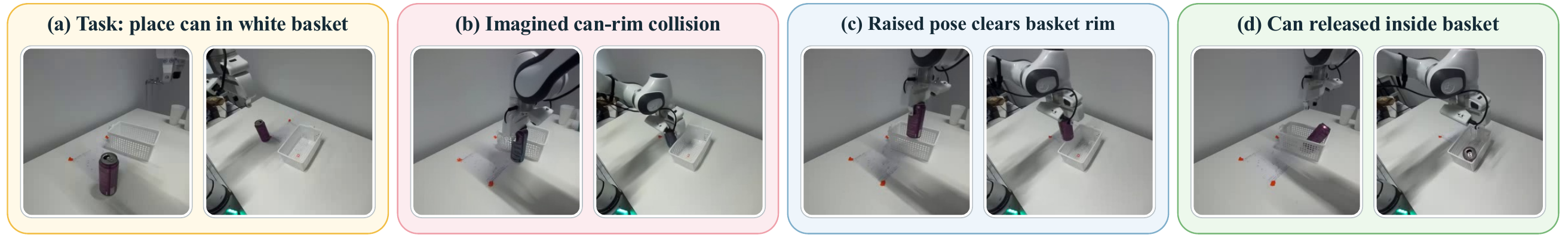}
    \caption{\textbf{Real-robot novel task generalization.} To place the can into the white basket, WAP uses world-model imagination to anticipate a potential collision between the can and basket rim. Global action optimization then lifts the gripper to create sufficient clearance, enabling successful placement.}
    \label{fig:real-robot}
\end{figure}

\looseness=-1
% \vspace{-5pt}
\subsection{Planning Efficiency}\label{sec:ablation} 
% \vspace{-5pt}
We directly compare the planning sample efficiency of our coarse-to-fine action search framework, which combines global optimization with local grid search, against prior MPC method VLMPC \citep{zhao2024vlmpc}. To isolate the efficiency of the planning protocol from world-model prediction errors, we provide the baselines with \textit{ground-truth} simulator rollouts instead of the world model imagined trajectories, and we use the same VLM agent to provide the reward. We use the same rollout length and same number of VLM calls for the baseline so runtime is determined by the total number of samples. As shown in Table~\ref{tab:mpc & bon}, our method achieves high success rates while evaluating substantially fewer candidate rollouts, demonstrating the efficiency of our structured and physically grounded planning pipeline. As illustrated in Fig.~\ref{fig:ablation}, global optimization can identify and avoid potential physical failures from a single imagined trajectory, whereas conventional sampling-based methods often repeat errors across multiple samples. Additional details are provided in Appendix~\ref{app:efficiency}.
\vspace{-15pt}
\begin{table}[htb]
    \centering
    \caption{\textbf{Planning sample efficiency against MPC.} Success rates (\%) are averaged within each setting.}
    \label{tab:mpc & bon}
    \small
    \begin{tabular*}{0.9\linewidth}{@{\extracolsep{\fill}}lccc@{}}
        \toprule
        Method & Compositional task & New layout & novel task \\
        \midrule
        VLMPC (8 samples) & 36.0 & 38.3 & 64.0 \\
        VLMPC (16 samples) & 46.0 & 50.0 & 76.0 \\
        \midrule
        \textbf{WAP w/ global optimization only (1 sample)} & 58.0 & 66.0 & 65.0 \\
        \textbf{WAP (6 samples)} & \textbf{72.0} & \textbf{82.0} & \textbf{78.0} \\
        \bottomrule
    \end{tabular*}
    \vspace{-10pt}
\end{table}

\looseness=-1
% \vspace{-8pt}
\subsection{Ablation Experiments}
% \vspace{-5pt}
We ablate the stage-wise contributions of global optimization, local search, and policy rollout imagination in Appendix~\ref{app:ablation of planning}. The results are reported in Table~\ref{tab:ablation planner}, where each stage provides gains as task complexity increases. The PnP cream cheese box task requires precise gripper rotations that are difficult to specify analytically, but can be accurately identified through local search by reasoning over imagined frames (Fig.~\ref{fig:app ablation rotation.}). For the PnP mug task, imagining future policy rollouts is critical for identifying states from which the policy can successfully complete the grasp (Fig.~\ref{fig:app ablation policy imagination}). We also ablate the choice of the VLM agent in Appendix~\ref{app:ablation vlm agent} and the choice of hyperparameters in Appendix~\ref{app:robustness}, demonstrating the robustness of our method.

% \vspace{-8pt}
\section{Conclusion and Limitations}\label{sec:conclusion}
% \vspace{-6pt}
\looseness=-1
In this work, we propose World Action Planner, an agentic robot planning system with world model to address novel tasks in generalization scenarios. Currently the planning stages are manually designed, and we left self-evolving agentic planning for future work. While world model imagination is time consuming, with wall clock time reported in Appendix~\ref{app:clocl}, inference acceleration and few step sampling techniques could be applied to further increase planning speed and real-time efficiency.

\clearpage
\subsection*{AI use statement}

In this work, we used generative AI tools for assist in the writing of proofs.
We have not used generative AI tools for Generate synthetic data sets, help develop theoretical models or conceptual frameworks, formulate mathematical claims, provide critical ingredients for proving mathematical claims,propose or refine hypotheses, design or provide feedback on research methodology or experiments, implement methods, clean and reformat dataset, support qualitative and thematic data analysis, interpret results,
and assist with translation are not applicable to this work.
Additionally, we used generative AI tools for create or modify scientific figures or images. We have reviewed all AI-assisted work. We checked the proof line by line and read the LLM-generated code. We take responsibility for the final content of this work,
including text, claims or artifacts produced with the aid of generative AI.

\subsection*{Reproducibility statement}

We provide full training details of our method, including detailed data recipe, model architecture, hyper-parameters and environment setups. We include detailed setups on how to compute the joint pose image using simulator with robot model, and detailed setup of the diffusion video model design. For our world action planner, we disclose the prompts used and specific hyper-parameter choices, and for the baseline policies, we disclose the training setup of both VLA and WAM. Thus we ensure our method and baseline results are reproducible.
%%%%%%%%%%%%%%%%%%%%%%%%%%%%%%%%%%%%%%%%%%%%%%%%%%%%%%%%%%%%
\bibliographystyle{iclr2027_conference}
\bibliography{ref}
\appendix

%%%%%%%%%%%%%%%%%%%%%%%%%%%%%%%%%%%%%%%%%%%%%%%%%%%%%%%%%%%%

\newpage
\paragraph{Appendix Overview}
First, we provide a comprehensive review of related works in Appendix~\ref{app:related works}. Proof of theoretical results is provided in Appendix~\ref{app:theory}. Then we provide the details of our method in Appendix~\ref{app:detail method} and details of the experiments in Appendix~\ref{app:exp detail}. We report the wall clock time of our world model in Appendix~\ref{app:clocl}.
\section{Additional Related Works}\label{app:related works}
\paragraph{Controllable World Models}
To generate and predict future states given control inputs, foundational video generative models typically condition on textual and visual descriptions \citep{ali2025world,agarwal2025cosmos}. To apply such controllable generation to robot manipulation, most existing works condition on low-dimensional action vectors through AdaLN-Zero modulation \citep{quevedo2025worldgym,zhu2025irasim} in diffusion models, or interleave low-dimensional action tokens and visual state tokens in auto-regressive models \citep{wu2024ivideogpt,assran2025v}. Recently, \citet{guo2025ctrl} proposed cross-attention between frames and action tokens for improved control, while adopting multi-view prediction to enhance consistency. Furthermore, \citet{team2025evaluating} demonstrated that pose-conditioned robot world models can successfully simulate OOD scenarios and unsafe executions, thereby serving as a tool to evaluate robot policies. \citet{wang2025precise} also adopted pose-image conditioning via the ControlNet \citep{zhang2023adding} mechanism; however, they extract the ground truth robot poses from future frames using image matching \citep{he2025matchanything}, while we compute the poses using forward dynamics to imagine different actions during test time planning. Concurrently, \citet{jia2026dreamplan,alzayer2026maskedvisualactionsunified} proposed rendering future arm images corresponding to actions as direct inputs to the video diffusion model.

\paragraph{World Models for Decision Making}
Recent works \citep{sharma2026world, jia2026dreamplan} have conducted reinforcement learning (RL) on generalist VLAs and VLM planners, respectively, demonstrating that RL within world models can enhance performance on physical robots. To plan using world models, MPC-based methods \citep{hansen2022temporal,hansen2023td,jain2025smooth} utilize the MPPI algorithm for exploration and update the sampling distribution based on Q-functions computed via the dynamics model. On the other hand, video planners \citep{du2023video,jang2025dreamgen,chen2025large} generate synthetic task completion videos using video generation models, and then extract the control actions using vision-based inverse-dynamics models or pose estimation techniques.

\paragraph{VLMs for Robot Decision Making}
To leverage the extensive world knowledge and reasoning capabilities of VLMs, the majority of existing works utilize these models for high-level task planning \citep{duan2024manipulate,liu2024moka,huang2023voxposer,hu2023look}. However, to ground executable actions in the physical world, these methods typically rely on manually engineered or sampling-based manipulation techniques to generate low-level actions---such as grasp synthesis \citep{liu2024moka,duan2024manipulate}, volumetric value maps \citep{huang2023voxposer}, motion arrows \citep{nasiriany2024pivot}, or teleoperation and kinesthetic teaching \citep{hu2023look}. In contrast, our method leverages highly-expressive generative policies \citep{chi2025diffusion} as modular tools to execute atomic skills. By grounding the low-level actions proposed by the VLM through our fine-grained, pose-image conditioned world model, we enable the VLM to synthesize and reason directly over low-level executable actions.

\paragraph{VLMs and World Models for Imitation Learning Policies}
Recent attempts have adopted pretrained VLMs or video world models as backbones for imitation learning policies, such as vision-language-action models (VLA) \citep{intelligence2025pi_,kim2024openvla} and world-action-models (WAM) \citep{ye2026world,kim2026cosmos,li2026causal}. These methods finetune action heads on top of the VLM or video generation model backbone to generate future actions, where OpenVLA \citep{kim2024openvla} adopts auto-regressive action generation and $\pi_{0.5}$ \citep{intelligence2025pi_} further incorporates diffusion-based action expert that denoises the action chunk. To adapt video generation models for policies, \citet{ye2026world,kim2026cosmos,li2026causal} first generates future frames and then generates the actions auto-regressively with the action head. However, these approaches often necessitate substantial quantities of high-quality action-labeled demonstration data, and their generalization performance can remain constrained, as observed in our experimental evaluation.

\paragraph{MPC with VLMs or World Models}
More closely related to our work, we discuss methods that employ MPC for general robot decision-making with foundation models. \citet{feng2025reflective} iteratively improves VLM action proposals using imagination from a diffusion generative model, but requires in-domain finetuning of both the VLM and the world model and operates on abstract semantic textual actions rather than precise robot control, which makes it impossible to directly operate robots so we do not include a direct comparison. \citet{ning2025prompting} uses digital twins built with physical simulators, rather than learned world models, for imagination and adopts CEM-based action sampling, which has lower sample efficiency than our physically grounded search framework. \citet{liu2025simpact} constructs a simulation of the task environment using the VLM and iteratively optimizes the action inside simulation, however, there precision is limited by the simulation construction capabilities of the VLM. Concurrent work by \citet{cui2026dreamsteer} improves a base VLA policy through imagination with an action-conditioned world model by selecting among $K$ action candidates sampled from the VLA. In contrast, we employ a more efficient search strategy and enable generalization to new environments beyond the training distribution of the base VLA.
\section{Theoretical Insights}\label{app:theory}
Here we provide the theoretical results which demonstrates the benefits of model-based learning with exploration over imitation learning \citep{rajaraman2020toward,rajaraman2021value} in multi-task generalization settings, where different tasks share the same dynamics but have different rewards \citep{kwon2023reward}. We focus on tabular MDPs and representation learning with linear function approximation. High level speaking, we show that model-based exploration can leverage data collected across all tasks, while imitation learning can only learn from the corresponding task expert, thus suffering from sub-optimal sample complexity.

\subsection{Literature Overview}
\subsubsection{Separation Between Model-Based and Model-Free Learning}
Prior research has demonstrated that model-free methods can exhibit exponentially worse sample complexity compared to model-based methods \citep{sun2019model} for specific classes of MDPs. However, these findings typically rely on strong structural assumptions regarding the model class; in the standard tabular setting, there is no information-theoretic gap between model-based and model-free approaches \citep{sun2019model,young2022benefits}. In our analysis, we show that within the tabular setting and under linear function approximation, multi-task imitation learning can be significantly more challenging than model-based planning. This stands in contrast to single-task settings, where imitation learning is generally more sample-efficient.

\subsubsection{Generalization in Reinforcement Learning}
In task-agnostic settings, where the agent lacks knowledge of the test environment's context and is evaluated on its \textit{average performance} across a family of MDPs, \citet{malik2021generalizable} established that MDPs must share certain similarities for a policy to generalize effectively. Furthermore, \citet{ye2023power} derived average optimality gaps comparable to those in supervised learning without online interaction, and demonstrating that pretraining can significantly improve sample complexity when subsequent online interactions are permitted in the test environment. Also, \citet{ghosh2021generalization,malik2021generalizable,ye2023power} have shown that such task agnostic policies can not achieve instance optimality in new environment without further interaction or history dependence. In our framework, we assume the context of each environment is provided to the learner, and we aim for \textit{instance optimality} without the need for online interactions during testing. Regarding multi-task representation learning, \citet{lu2021power,hu2021near} showed that learning a shared representation across tasks can improve sample complexity. We further assume that task-dependent context features are known to the learner and use them to construct context-dependent policies. Our setting is more aligned with modern robotic practices where task instructions serve as inputs to a generalist policy, and we derive representations for the tasks using pretrained models \citep{chen2025large,intelligence2025pi_}.

\subsection{Setup}
We consider multi-task learning as contextual MDPs \citep{hallak2015contextual,weltevrede2024exploration} where MDPs with different contexts have shared dynamics and different rewards. Each MDP is described by the tuple $\cM=\bracket{\cS, \cA, \PP, r, \rho, H}$, where $\cS$ is the state space, $\cA$ is the action space, and $H$ is the horizon. The initial state $s_1$ is drawn from the initial state distribution $\rho\in\Delta(\cS)$, and the state transition is governed by the transition probability $\PP$. At each time step $h\in[H]$, the agent chooses action $a_h\in\cA$ in state $s_h$ according to policy $\pi(s_h)\in\Delta(\cA)$, receives reward $r_h(s_h, a_h)$ and transitions to the next state according to the transition kernel $\PP_h(s_{h+1}|s_h,a_h)$. The episode terminates after $H$ steps. Define the state-action value function (Q-function) and the state value function as the expected cumulative reward of $\pi$:
\begin{align*}
    Q_h^{\pi}(s,a)&=\EE\mbracket{\sum_{h'=h}^Hr_{h'}(s_{h'}, a_{h'})~\Big|~s_h=s, a_h=a, \pi}\,,\\
    V_h^\pi(s)&=\EE\mbracket{\sum_{h'=h}^Hr_{h'}(s_{h'}, a_{h'})~\Big|~s_h=s, \pi}\,,
\end{align*}
with the convention $V_{H+1}^\pi\equiv 0$. Both are maps into $[0,H]$ and satisfy the Bellman equations
\begin{align*}
    Q_h^\pi(s,a)=r_h(s,a)+\sum_{s'\in\cS}\PP_h(s'|s,a)\,V_{h+1}^\pi(s')\,,~~~V_h^\pi(s)=\EE_{a\sim\pi_h(s)}\mbracket{Q_h^\pi(s, a)}\,.
\end{align*}
Both sides of each identity are real numbers: $Q_h^\pi$ and $V_h^\pi$ are value functions, not actions.
The goal of the agent is to achieve maximum expected cumulative reward $\EE_{s_1\sim\rho}V_1^{\pi}(s_1)$. Since the action space and episode length is finite, there always exist an optimal Q-function $Q_h^*(s,a)=\sup_{\pi}Q_h^\pi(s,a)$ and value function $V_h^*(s)=\sup_{\pi}V_h^\pi(s)$ for all $s,a$ that satisfies the Bellman optimality equation:
\begin{align*}
    Q_h^*(s,a)=r_h(s,a)+\sum_{s'\in\cS}\PP_h(s'|s,a)V_{h+1}^*(s')\,,~~~ V_h^*(s)=\max_{a\in\cA}Q_h^*(s,a)\,,~~~V_{H+1}^*\equiv 0\,.
\end{align*}
Note that the second identity takes the \emph{maximum} rather than the $\argmax$: $V_h^*(s)\in[0,H]$ is a value, whereas $\argmax_{a}Q_h^*(s,a)\in\cA$ is an action. The two are related through the greedy policy: there exists a deterministic policy $\pi_h^*(s)=\delta\bracket{\argmax_{a}Q_h^*(s,a)}$ that is optimal with $Q_h^{\pi^*}=Q_h^*$ and $V_h^{\pi^*}=V_h^*$. When the policy is deterministic, we abuse the notation and denote $a=\pi(s)$. 

In the multi-task setting, we have a set of contexts (tasks) $\cC$, with different reward functions $\sets{r_c:\cS\times\cA\rightarrow[0,1], c\in\cC}$, and the MDP associated with context $c$ is $\cM_c=\bracket{\cS, \cA, \PP, r_c, \rho, H}$. We assume that the context $c$ is determined before the episode and revealed to the agent. We then denote the policy and value functions given context $c$ as $\pi_c(s)$, $Q_h^\pi(s,a;c)$ and $V_h^\pi(s;c)$ respectively. Since the context is known to the agent, we aim to obtain a context dependent policy $\pi_c(s)$ that maximizes $\EE_{s_1\sim\rho}V_1^\pi(s_1;c)$ for all context $c\in\cC$. When rewards for different contexts share the same function approximation, we abuse notation and write $r(s,a;c)=r_c(s,a)$; similarly when policies for different contexts are from the same  class, we write $\pi(s;c)=\pi_c(s)$.

\subsubsection{Interaction Protocol}\label{para:protocol} In this work we investigate the generalization task, where the agent is able to explore the environment and collect data during training time, across contexts of its own choosing. At test time, the agent is evaluated by the performance of its policy under a randomly selected context. For the imitation learning setup, we further assume the existence of expert policies to collect demonstrations for any given context. Denote the total rounds of episode during the training phase as $K$.

\paragraph{Task Schedule} Before data collection begins, an agent commits to a \emph{task schedule}: non-negative integers $\sets{n_c}_{c\in\cC}$ with $\sum_{c\in\cC}n_c=K$, giving the number of episodes it spends on each context. The schedule may depend on $\cC$, on the feature maps where those are available, and on the agent's own randomness, but not on the data, which has not been observed yet. This is how multi-task datasets are assembled in practice, either as a fixed number of demonstrations per task or by i.i.d.\ task sampling.

We impose the schedule on \emph{both} agents, so neither is advantaged by it, and it costs the model-based agent nothing in either setting we study: the tabular algorithm explores reward-free and therefore never consults the context, and the linear algorithm fixes its contexts through a G-optimal design computed from the known task features (Theorem~\ref{thm:design}). Within its schedule the model-based agent still chooses each exploration policy $\pi_k$ adaptively from the data collected so far. We specify the two protocols below.

\paragraph{Imitation Learning Agent} At each episode $k\in[K]$, for the context $c_k\in\cC$ assigned by its schedule, the imitation learning agent queries the deterministic expert policy $\pi_{c_k}^*$ to rollout in the environment $\cM_{c_k}$ and collect data $\tau_{k,c_k}=\sets{(s_h^k, a_h^k)}_{h=1}^H$, where $a_h^k=\pi^*_{c_k}(s_h^k)$. We denote the total data collected as $\cD=\sets{\tau_{k,c_k}}_{k=1}^K$.
After the data collection phase, the learner outputs a context dependent policy $\widehat{\pi}=\sets{\widehat{\pi}_c, c\in\cC}$ minimizing the empirical risk of $\pi$ with respect to $\cD$ in a policy class $\Pi$:
\begin{align*}
\widehat{\pi} = \argmin_{\pi\in \Pi}\text{Pr}_{(s,a,c)\sim\cD}\bracket{\pi_c(s)\neq a}\,,
\end{align*}
where we break ties arbitrarily. 

\paragraph{Model-based Agent} At each episode $k\in[K]$, for the context $c_k\in\cC$ assigned by its schedule, the agent may explore with any policy $\pi_k$, itself chosen from the data collected so far, to collect the trajectory $\tau_{k,c_k}=\sets{(s_h^k, a_h^k, r_h^k)}_{h=1}^H$ in $\cM_{c_k}$, where $a_h^k\sim\pi_k(s_h^k)$ and $r_h^k\sim r_c(s_h^k, a_h^k)$. After collecting the dataset $\cD=\sets{\tau_{k,c_k}}_{k=1}^K$, the learner outputs an estimate of the model $\widehat{\PP}$ , and if the reward is unknown, an estimate of the reward $\widehat{r}=\sets{\widehat{r}_c, c\in\cC}$. At test time, given the random context $c$, the agent adopts a planner that outputs the optimal policy $\widehat{\pi}_c$ with estimated MDP $\widehat{\cM}_c=\bracket{\cS, \cA, \widehat{\PP}, \widehat{r}_c, \rho, H}$. 

\paragraph{Information Available to Each Learner}\label{para:information}
The two protocols above are not two algorithms fed the same observations. They are two different \emph{interfaces} onto the environment, and what we compare is the interfaces. Spelling out what each learner gets to see therefore keeps the source of the separation in plain sight.
\begin{itemize}[leftmargin=1.5em,itemsep=1pt,topsep=2pt]
\item The imitation learner observes state--action pairs $(s_h^k,a_h^k)$ with $a_h^k=\pi^*_{c_k}(s_h^k)$. It never observes a reward, and it is not given the reward functions $\sets{r_c}$. This is the defining feature of imitation learning: the task is specified only through expert behavior.
\item The model-based learner observes transitions $(s_h^k,a_h^k,s_{h+1}^k)$ and reward samples $r_h^k$, and it can evaluate the reward of \emph{any} task at planning time --- in the tabular setting of Section~\ref{app:tabular} because $\sets{r_c}$ is given, and in the linear setting of Section~\ref{app:linear} because the shared reward representation $\sets{\Theta_h}$ is estimated from data. It never observes an expert action.
\end{itemize}
Neither interface contains the other. The imitation learner is handed optimal action labels the model-based learner never sees, and the model-based learner is handed reward information the imitation learner never sees. Our claim is that the second is the more valuable of the two once the tasks are many, for the simple reason that the ability to score a reward carries over to a new task while an expert label does not. Remark~\ref{rem:information} pins down which half of the asymmetry is responsible for which result, and shows the tabular separation surviving even after the dynamics information is leveled between the two.

\paragraph{Suboptimality Gap} We define the suboptimality gap of the context dependent policy class $\pi=\sets{\pi_c, c\in\cC}$ returned by the agent as:
\begin{align}\label{eq:subopt}
    \EE_{c\sim \unif(\cC), ~s_1\sim\rho}\mbracket{ V^{*}(s_1;c)-V^{\pi_c}(s_1;c)}\,.
\end{align}

\subsection{Tabular MDPs}\label{app:tabular}
Here we consider tabular MDPs with finite state and action spaces and finite context set. We assume the reward function for each context is known to the agent and the transition is unknown \citep{hu2022nearly}. We denote the number of states, actions and contexts as $|\cS|$, $|\cA|$ and $|\cC|$.

\subsubsection{Model-based Planning Agent}
Here we prove Part \textbf{(i)} of Theorem~\ref{thm:main tabular}. The result for the model-based exploration agent when generalizing to new rewards comes immediately with the reward-free exploration result of \citep{jin2020reward}.
\begin{theorem}[Theorem 3.1 of \citet{jin2020reward}]\label{thm:jin rf}
Fix $\epsilon\in(0,H]$ and $\delta\in(0,1)$. There exists a reward-free model-based exploration algorithm which, with probability at least $1-\delta$, outputs a model estimate $\widehat{\PP}$ such that \emph{simultaneously} for an arbitrary collection of known reward functions $\sets{r_c:\cS\times\cA\to[0,1]}$, the optimal policy $\widehat{\pi}_c$ of the estimated MDP $\widehat{\cM}_c=\bracket{\cS, \cA, \widehat{\PP}, {r}_c, \rho, H}$ satisfies $V^*_{\cM_c}-V^{\widehat{\pi}_c}_{\cM_c}\leq\epsilon$. The number of exploration episodes is $\Tilde{\mathcal{O}}\bracket{\frac{H^5|\cS|^2|\cA|}{\epsilon^2}}$, where $\Tilde{\mathcal{O}}$ hides logarithmic factors in $|\cS|,|\cA|,H$ and $1/\delta$.
\end{theorem}
Inverting the sample complexity of Theorem~\ref{thm:jin rf} gives Part \textbf{(i)}: after $K$ exploration episodes the model-based planner attains, for \emph{every} task simultaneously,
\begin{align}\label{eq:tabular mb rate}
    V^*_{\cM_c}-V^{\widehat{\pi}_c}_{\cM_c}\leq \Tilde{\mathcal{O}}\bracket{\min\sets{H,~\sqrt{\frac{H^5|\cS|^2|\cA|}{K}}}}\,,\quad\forall c\in\cC\,,
\end{align}
where the $\min$ with $H$ records that no bound can be informative below the trivial range of the value function, so the rate $K^{-1/2}$ is only meaningful once $K\gtrsim H^3|\cS|^2|\cA|$. The bound is independent of the number of tasks $|\cC|$, because the exploration phase is reward-free and its guarantee holds uniformly over reward functions.

\subsubsection{Imitation Learning Agent}
Here we provide the hardness result for imitation learning in Part \textbf{(ii)} of Theorem~\ref{thm:main tabular}. High level speaking, since the expert policies differ across tasks, the average number of demonstrations available for each task is $\frac{K}{|\cC|}$, and the suboptimality gap inherits the $\Omega\bracket{1/N}$ lower bound of \citet{rajaraman2020toward} with $N=\frac{K}{|\cC|}$. Stated that baldly the result is not yet correct, for two reasons that also happen to determine when the separation matters.

\paragraph{Regime and Clipping} The first is a ceiling. No policy can be worse than $H$, so no lower bound may exceed $H$ either, and the bound has to saturate somewhere. The second concerns the construction. The hard instance of \citet{rajaraman2020toward} spreads mass $\xi$ over $|\cS|-1$ ``light'' initial states, which needs $(|\cS|-1)\xi\leq 1$. Setting $\xi=\frac{1}{\Bar{N}+1}$ outright breaks this as soon as $\Bar{N}<|\cS|-2$, leaving a signed measure in place of a distribution, and it breaks it precisely in the data-scarce regime the bound is meant to describe. We therefore define
\begin{align}\label{eq:xi def}
    \xi=\min\sets{\frac{1}{\Bar{N}+1},~\frac{1}{|\cS|-1}}\,,\qquad \Bar{N}=\frac{K}{|\cC|}\,,
\end{align}
which is a valid construction in both regimes and yields a lower bound that is clipped at $H$, as it must be.

We design the hard MDP following \cite[Sec.~A.4.1]{rajaraman2020toward}. For each state $s\in\cS$, it is an absorbing state where $\PP(\cdot|s,a)=\delta(s)$, $\forall a\in\cA$. With $\xi$ defined in \eqref{eq:xi def}, we define the initial state distribution $\rho=\sets{\xi, \xi, \cdots,\xi, 1-(|\cS|-1)\xi}$, which is a probability distribution because $(|\cS|-1)\xi\leq 1$.

We define the class of all context dependent deterministic policies as $\Pi_{\text{det}}$. We first sample the expert policy for each context uniformly and independently from $\Pi_{\text{det}}$, that is for each $s\in\cS$, $c\in\cC$, the optimal action $\pi^*_c(s)$ is uniformly sampled from $\cA$. Given $\pi_c^*$, we define the reward function $r_c(\pi_c^*)$ where at each state, choosing the expert's action $\pi_c^*(s)$ receives reward $1$ and otherwise $0$. We denote the MDP with the transition and initial state distribution constructed above and the reward function $r_c(\pi_c^*)$ as $\cM_c(\pi_c^*)$. It is obvious from the construction that the value of $\pi_c^*$ in $\cM_c(\pi_c^*)$ is $H$ for any initial state $s_1$, since it receices a reward of $1$ at every step.

Define the set of deterministic policies $\pi=\sets{\pi_c, c\in\cC}$ that agrees with the expert policy in dataset $D$ as:
\begin{align*}
    \Pi(\cD)=\sets{\pi\in\Pi_{\text{det}}^{|\cC|}:~~\pi_{c_k}(s_h^k)=a_h^k, ~\forall k\in[K], h\in[H]}\,.
\end{align*}
where we recall $a_h^k=\pi_{c_k}^*(s_h^k)$ in the expert dataset. We also denote
\begin{align*}
    \cK_c=\sets{k=1,2,\cdots,K:~c_k=c}\,,\quad \cD_c=\sets{\tau_{c_k,k},~k\in\cK_c}\,.
\end{align*}
as the subset of episode indices and episodes the context of which is $c$. We then denote the class of deterministic policies conditioning on context $c$ that agrees with the expert $\pi_c^*$ in dataset $D_c$ as:
\begin{align*}
    \Pi_c(\cD_c)=\sets{\pi\in\Pi_{\text{det}}:~~\pi_{c}(s_h^k)=a_h^k\,, k\in\cK_c}\,. 
\end{align*}
We have the conditional distribution of $\pi^*=\sets{\pi_c^*, c\in\cC}$ conditioned on $\cD$ is uniform in $\Pi(\cD)$, with the marginal conditional distribution of $\pi_c^*$ being uniform in $\Pi_c(\cD_c)$. Thus, the Bayes suboptimality gap of the output policy $\widehat{\pi}=\sets{\widehat{\pi}_c, c\in\cC}$ under the constructed MDP class $\cM(\pi^*)$ can be written as:
\begin{align}\label{eq:bayes tabular}
    \nonumber&\EE_{\pi^*\sim\unif(\Pi(\cD))}\mbracket{\EE_{c\in\unif(\cC)}\mbracket{\EE_{s_1\sim\rho}\mbracket{V^*_{\cM_c(\pi_c^*)}(s_1;c)-V^{\widehat{\pi}_c}_{\cM_c(\pi_c^*)}(s_1;c)}}}\\
    =&\EE_{c\in\unif(\cC)}\mbracket{\EE_{\pi_c^*\sim\unif(\Pi_c(\cD_c))}\mbracket{\EE_{s_1\sim\rho}\mbracket{H-V_{\cM_c(\pi_c^*)}^{\widehat{\pi}_c}(s_1;c)}}}
\end{align}
which is because the value of $\pi_c^*$ in our constructed MDP $\cM_c(\pi_c^*)$ is $H$; and each entry $\pi_c^*$ of $\pi^*$ is independently uniform distributed in $\Pi_c(\cD_c)$ so we apply Fubini's Theorem to exchange the expectations. 

We denote the set of initial states visited in episodes with context $c$ as 
\begin{align*}
    \cS_1(\cD_c)=\sets{s\in\cS: \exists k\in\cK_c, s_1^k=s}\,.
\end{align*}
Recall that in our MDP design each state is an absorbing state, so $\cS_1(\cD_c)$ is also the set of total states visited under context $c$ in $\cD$. Intuitively speaking, in context $c$, if $s_1\sim\rho$ is from a unvisited state outside of $\cS_1(\cD_c)$, then the the estimated policy $\widehat{\pi}$ will suffer from constant errors. This is because the potential distribution of optimal action $\pi^*_c(s_1)$ given $\cD$ is distributed uniformly in $\cA$, independent of the data from other contexts, thus the agent can not guess the right action with high probability. Given context $c\in\cC$, we have the following Lemma from \citet{rajaraman2020toward}.
\begin{lemma}[Lemma A.19 of \citet{rajaraman2020toward}]\label{lem: suboptimal}
Given context $c\in\cC$ and dataset $\cD$, we have the Bayes suboptimality of the imitation policy $\widehat{\pi}_c$ lower bounded as: 
    \begin{align*}
       \EE_{\pi_c^*\sim\unif(\Pi_c(\cD_c))}\mbracket{\EE_{s_1\sim\rho}\mbracket{H-V_{\cM_c(\pi_c^*)}^{\widehat{\pi}_c}(s_1;c)}}\geq H\bracket{1-\frac{1}{|\cA|}}\bracket{1-\rho\bracket{\cS_1(\cD_c)}}
    \end{align*}
\end{lemma}
where $\rho(\cS_1(\cD_c))=\sum_{s\in \cS_1(\cD_c)}\rho(s)$ is the total measure of all initial states visited in $\cD_c$.

Next, we lower bound the average uncovered initial-state mass $\EE_{c}\mbracket{1-\rho(\cS_1(\cD_c))}$. Even with the task schedule fixed, $\cS_1(\cD_c)$ is a random set, because the initial states are drawn from $\rho$; the statement below is therefore an expectation over the draw of $\cD$ rather than a deterministic statement about one dataset.
\begin{lemma}\label{lem:tabular geometric ineq}
Let $\sets{n_c}_{c\in\cC}$ be the agent's task schedule (Section~\ref{para:protocol}), so that $\sum_{c\in\cC}n_c=K$ and $\cD$ contains $n_c$ expert episodes for each context $c$. Then, with $\xi$ as in \eqref{eq:xi def},
    \begin{equation*}
        \EE_{\cD}\EE_{c\sim\unif(\cC)}\mbracket{1-\rho(\cS_1(\cD_c))}\geq (|\cS|-1)\,\xi\,\bracket{1-\xi}^{K/|\cC|}\geq \frac{1}{e}\min\sets{1,~\frac{|\cS|-1}{\frac{K}{|\cC|}+1}}\,.
    \end{equation*}
\end{lemma}
\begin{proof}
    Recall $\rho=\sets{\xi, \xi, \cdots,\xi, 1-(|\cS|-1)\xi}$ with $\xi$ from \eqref{eq:xi def} and $\Bar{N}=\frac{K}{|\cC|}$.
   Since for every episode index $k\in\cK_c$ the initial state $s_1^k$ is drawn independently from $\rho$, and the schedule fixes $n_c$ before those draws,
    \begin{align*}
        \EE_{\cD}\mbracket{1-\rho(\cS_1(\cD_c))}=&\sum_{s\in\cS}\rho(s)(1-\rho(s))^{n_c}\\
        \geq &(|\cS|-1)\,\xi\bracket{1-\xi}^{n_c}\,,
    \end{align*}
    where we kept only the $|\cS|-1$ light states and dropped the non-negative contribution of the heavy state. The map $t\mapsto(1-\xi)^t$ is convex, so Jensen's inequality and $\frac{1}{|\cC|}\sum_{c}n_c=\Bar{N}$ give
    \begin{align*}
        \EE_{\cD}\EE_{c\sim\unif(\cC)}\mbracket{1-\rho(\cS_1(\cD_c))}\geq& (|\cS|-1)\,\xi\,\frac{1}{|\cC|}\sum_{c\in\cC}\bracket{1-\xi}^{n_c}\\
        \geq&(|\cS|-1)\,\xi\bracket{1-\xi}^{\Bar{N}}\,.
    \end{align*}
    It remains to evaluate the two regimes of \eqref{eq:xi def}.

    \emph{Data-rich regime, $\Bar{N}\geq|\cS|-2$.} Here $\xi=\frac{1}{\Bar{N}+1}$, and using $\bracket{1-\frac{1}{m}}^{m-1}\geq e^{-1}$ with $m=\Bar{N}+1$,
    \begin{align*}
        (|\cS|-1)\,\xi\bracket{1-\xi}^{\Bar{N}}\geq \frac{|\cS|-1}{e\bracket{\Bar{N}+1}}=\frac{1}{e}\min\sets{1,\frac{|\cS|-1}{\Bar{N}+1}}\,,
    \end{align*}
    where the last equality uses $|\cS|-1\leq\Bar{N}+1$ in this regime.

    \emph{Data-scarce regime, $\Bar{N}<|\cS|-2$.} Here $\xi=\frac{1}{|\cS|-1}$, so $(|\cS|-1)\xi=1$ and
    \begin{align*}
        (|\cS|-1)\,\xi\bracket{1-\xi}^{\Bar{N}}=\bracket{1-\frac{1}{|\cS|-1}}^{\Bar{N}}\geq\bracket{1-\frac{1}{|\cS|-1}}^{|\cS|-2}\geq \frac{1}{e}=\frac{1}{e}\min\sets{1,\frac{|\cS|-1}{\Bar{N}+1}}\,,
    \end{align*}
    where the last equality uses $\Bar{N}+1<|\cS|-1$ in this regime. Combining the two regimes proves the claim.
\end{proof}
The second regime is exactly the term that a bound of the form $\Omega\bracket{\frac{|\cC|}{K}}$ omits: when each task receives fewer demonstrations than there are states, the uncovered mass is a constant and no rate in $K$ is available at all.

Finally, we prove the main theorem.
\begin{proof}[Proof of Part \textbf{(ii)}, Theorem~\ref{thm:main tabular}]
    To bound the worst-case suboptimality of $\widehat{\pi}$ it suffices to bound the Bayes suboptimality under our joint distribution of expert policies $\pi^*\sim\unif\bracket{\Pi(\cD)}$ and constructed MDP instances $\cM(\pi^*)$. Combining \eqref{eq:bayes tabular} with Lemma~\ref{lem: suboptimal} and then taking the expectation over $\cD$ and applying Lemma~\ref{lem:tabular geometric ineq}, we obtain
    \begin{align*}
        &\EE_{\cD}\EE_{\pi^*\sim\unif(\Pi(\cD))}\mbracket{\EE_{c\in\unif(\cC)}\mbracket{\EE_{s_1\sim\rho}\mbracket{V^*_{\cM_c(\pi_c^*)}(s_1;c)-V^{\widehat{\pi}_c}_{\cM_c(\pi_c^*)}(s_1;c)}}}\\
        \geq{}& H\bracket{1-\frac{1}{|\cA|}}\EE_{\cD}\EE_{c\sim\unif(\cC)}\mbracket{1-\rho(\cS_1(\cD_c))}\\
        \geq{}& \frac{H}{e}\bracket{1-\frac{1}{|\cA|}}\min\sets{1,~\frac{|\cS|-1}{\frac{K}{|\cC|}+1}}
        ~=~\Omega\bracket{\min\sets{H,~\frac{H|\cS||\cC|}{K}}}\,.
    \end{align*}
    Thus there exists a worst-case instance $\pi^*=\sets{\pi_c^*, c\in\cC}$ and MDP family $\cM(\pi^*)=\sets{\cM_c(\pi_c^*), c\in \cC}$ under which the average suboptimality of the imitation learner is at least $\Omega\bracket{\min\sets{H,~\frac{H|\cS||\cC|}{K}}}$.
\end{proof}
Equivalently, in terms of sample complexity, the imitation learner needs $K=\Omega\bracket{\frac{H|\cS||\cC|}{\epsilon}}$ episodes to reach average suboptimality $\epsilon\in(0,H)$, whereas by Theorem~\ref{thm:jin rf} the model-based planner needs $K=\Tilde{\mathcal{O}}\bracket{\frac{H^5|\cS|^2|\cA|}{\epsilon^2}}$ episodes and is independent of $|\cC|$.

\begin{remark}[When the separation is real]\label{rem:tabular regime}
Quoting the comparison as ``$K^{-1/2}$ versus $|\cC|/K$'' invites the wrong reading, because in $K$ alone the imitation rate $K^{-1}$ is the \emph{faster} of the two and the bounds are not ordered for every choice of parameters. Everything hinges on the $|\cC|$ factor. The model-based sample complexity is the smaller one exactly when
\begin{align}\label{eq:tabular separation regime}
    |\cC|~\gtrsim~\frac{H^4|\cS||\cA|}{\epsilon}\,,
\end{align}
that is, once the tasks are many relative to the size of any one of them and to the accuracy sought. Our experiments sit in this regime: one manipulation workspace, compactly described, with an open-ended list of language-specified jobs to do in it. Outside \eqref{eq:tabular separation regime}, with few tasks and a large state space, imitation learning is the better protocol, which is the familiar single-task picture.
\end{remark}

\begin{remark}[What information each bound actually uses]\label{rem:information}
The two halves of Theorem~\ref{thm:main tabular} are proved under the two different interfaces of Section~\ref{para:information}. Since the interfaces differ, it matters which part of the difference is carrying the result.

\textbf{(a) The gap does not come from the imitation learner's ignorance of the dynamics.} Every state in the hard instance is absorbing, so one trivial transition kernel serves the entire instance family. We can therefore give the imitation learner $\PP$ outright and lose nothing: $\PP$ is a constant of the family, so a learner reading $(\cD,\PP)$ is still a learner reading $\cD$, and the proof runs unchanged. Theorem~\ref{thm:main tabular}\textbf{(ii)} accordingly holds against imitation learners that know the dynamics exactly, which disposes of the most obvious objection to the comparison.

\textbf{(b) It comes from reward evaluability, and that assumption is load-bearing.} The one thing the model-based learner has and the imitation learner lacks is the ability to score any task's reward at planning time. Take it away and the advantage goes with it. Suppose the model-based learner also had to estimate $r_c$ from reward samples, with no structure shared across tasks, as the tabular setting allows. In this instance $r_c$ encodes the expert action at every state, so pinning down $r_c$ at a state means visiting that state under context $c$; Lemma~\ref{lem: suboptimal} and Lemma~\ref{lem:tabular geometric ineq} then apply to it word for word, and it too pays $\Omega\bracket{\min\sets{H,H|\cS||\cC|/K}}$. Read the tabular theorem, then, as measuring what a transferable success detector is worth alongside shared-dynamics data. In our system that detector is a VLM judging an unseen task from a terminal frame. The theorem does not say that planning beats behavior cloning when the two see the same observations.

\textbf{(c) The linear-MDP results of Section~\ref{app:linear} are the matched-information counterpart.} There the reward is withheld from both: the model-based learner estimates $\sets{\Theta_h}$ from noisy reward samples much as the imitation learner fits a classifier to expert labels. Each sees $K$ episodes and each must reach unseen tasks through a learned representation, so whatever separates them there is about representation rather than about access (Section~\ref{app:linear imitation}).
\end{remark}

\subsection{Linear Function Approximation}\label{app:linear}
In this section, we provide the results under linear function approximation. Unlike the tabular setting, here the reward is unknown to \emph{both} learners: the model-based agent must estimate the shared reward representation from noisy reward samples, while the imitation agent must fit a policy to expert labels. The separation established below is therefore representational rather than informational (cf.\ Remark~\ref{rem:information}\textbf{(c)}). First, we provide the definition of linear MDP following \citet{jin2023provably}:
\begin{definition}[Linear MDP \citep{jin2023provably}]\label{def:linear mdp}
MDP $\cM=\bracket{\cS,\cA,\PP,r,\rho, H}$ is called a linear MDP with feature map $\phi:\cS\times\cA\rightarrow\RR^d$, if there exist $d$ (unknown) signed measures $\bm{\mu}_h=\bracket{\mu_h^{(1)}, \mu_h^{(2)},\cdots,\mu_h^{(d)}}$ over $\cS$ such that the transition probability can be represented as linear functions
\begin{align*}
    \PP_h(\cdot|s,a)=\innerproduct{\phi(s,a)}{\bm{\mu}_h(\cdot)}\,,
\end{align*}
and (unknown) vectors $\theta_h\in\RR^d$ such that the expectation of the random reward can be represented as a linear function
\begin{align*}
    \EE\mbracket{{r_h(s,a)}}=\innerproduct{\phi(s,a)}{\theta_h}\,,~~0\leq {r_h(s,a)}\leq 1,
\end{align*}
for all $(s,a)\in\cS\times\cA$ and $h\in[H]$.
    Without loss of generality, we assume $\norm{\phi(s,a)}\leq 1$ for all $(s,a)$ and $\max\sets{\norm{\bm{\mu}_h(\cS)}, \norm{\theta_h}}\leq \sqrt{d}$ for all $h$.
\end{definition}

Given the above definition, we further explore the setting of learning the reward function for generalizing to new tasks with task features. Specifically, we assume that the per-task reward vectors $\theta_{h,c}\in\RR^d$ share a common task feature map $\psi(c)\in\RR^{d'}$ known to the agent, with $\theta_{h,c}=\Theta_h\psi(c)$. Thus, by learning the transformations $\sets{\Theta_h}_{h\in[H]}$ using reward samples from source tasks, we can generalize to any new task using the representation $\psi$. We point out that learning such a reward function with linear structure is non-trivial due to the random noise in the reward received \citep{abbasi2011improved}.
\begin{assumption}[Reward representation]\label{ass:reward rep}
There exists a feature map $\psi:\cC\rightarrow\RR^{d'}$ and (unknown) matrices $\Theta_h\in\RR^{d\times d'}$, $h\in[H]$, such that
    the random reward for context $c$ satisfies $\EE\mbracket{r_{h,c}(s,a)}=\phi(s,a)^{\top}\Theta_h\psi(c)$ with $0\leq{r_{h,c}(s,a)}\leq 1$ and $\norm{\psi(c)}\leq 1$, for all $(s,a,c,h)\in\cS\times\cA\times\cC\times[H]$.
    We assume the set $\sets{\psi(c),c\in\cC}\subset \RR^{d'}$ is compact and $\text{span}(\sets{\psi(c),c\in\cC})=\RR^{d'}$.
    We further assume the normalization $\norm{\Theta_h}_F\leq\sqrt{dd'}$ for all $h\in[H]$, equivalently $\norm{\ve(\Theta_h)}\leq\sqrt{dd'}$, which matches the normalization $\norm{\theta_h}\leq\sqrt{d}$ placed on the single-task reward vector in Definition~\ref{def:linear mdp}.
    Allowing $\Theta_h$ to depend on $h$ is what makes the class closed under step-dependent rewards; a single shared $\Theta$ cannot, for instance, express a reward that is identically zero at $h=1$ and non-trivial at $h=2$, a construction we use in Section~\ref{app:linear imitation}.
\end{assumption}
The normalization of $\Theta_h$ has to be assumed; it is not implied by the rest of the assumption. Bounded rewards and bounded task features leave the parameter free to blow up whenever the task features are small: in the scalar case $d=d'=1$, taking $\phi\equiv1$, $\psi(c)\equiv\eta$ and $\Theta=1/\eta$ gives reward $1$ for every $\eta>0$, while $\abs{\Theta}\rightarrow\infty$ as $\eta\rightarrow0$. The span condition rules out $\psi\equiv0$ but not this sequence. Since the ridge estimator of Lemma~\ref{lem:app linear r bound} pays a regularization bias proportional to $\norm{\vartheta_h}$, the bound enters the analysis and we state it up front. A general bound $\norm{\Theta_h}_F\leq B_\Theta$ works equally well and enters the reward error of Lemma~\ref{lem:app linear r bound}, and hence the rate of Theorem~\ref{thm:main linear}, linearly in $B_\Theta$; we fix $B_\Theta=\sqrt{dd'}$ only to keep the $\poly(d,d')$ bookkeeping in one place.
To explore the feature set, we define the G-optimal design of the feature set as:
\begin{theorem}[Kiefer--Wolfowitz]\label{thm:design}
    There exists a distribution $p\in\Delta(\cC)$ such that $|\supp(p)|\leq \frac{d'(d'+1)}{2}$ and $\max_{c\in\cC}\norm{\psi(c)}_{W^{-1}}^2= d'$, i.e.\ $\max_{c\in\cC}\norm{\psi(c)}_{W^{-1}}=\sqrt{d'}$, where $W$ is the covariance matrix under $p$:
    \begin{align*}
W=\sum_{c\in\cC}p(c)\psi(c)\psi(c)^{\top}\,.
    \end{align*}
\end{theorem}
We write $p_c=p(c)$ below. Note that the design guarantee is on the \emph{squared} norm; carrying it as $\norm{\psi(c)}_{W^{-1}}\le d'$ would overstate the dimension dependence by a square.

We also assume that the MDP has well-conditioned covariates and there exists a \textit{unknown} policy that can explore the $h$ layer of the MDP, following standard practice \citep{kong2023improved,li2020sample}
\begin{assumption}[Exploratory assumption]\label{ass:exp}
There exists $\kappa>0$, such that 
\begin{align*}
    \min_{h\in[H]}\sup_{\pi}\lambda_{\min}\bracket{\EE_{\pi}\mbracket{\phi(s_h,a_h)\phi(s_h,a_h)^{\top}}}\geq \kappa\,.
\end{align*}
\end{assumption}

\subsubsection{Model-based Planning Agent}\label{app:linear model}
In this section, we prove the suboptimality gap for the model-based planning agent under our linear function approximation. We first restate the result from \citet{wagenmaker2022reward} that enables sufficient exploration of the feature space, without knowledge of any exploratory policies:
\begin{theorem}[Theorem 4, \citep{wagenmaker2022reward}]\label{thm:wagen}
Fix $\gamma\in[0,1]$, under Assumption \ref{ass:exp}. There exists an algorithm that collects observations $\cD=\sets{(s_1^{k}, a_1^{k}), \cdots, (s_H^k, a_H^k)}_{k=1}^K$, such that with probability $1-\delta$, we have for any $h$:
\begin{align*}
    \lambda_{\min}\bracket{\sum_{k=1}^K\phi(s_h^k, a_h^k)\phi(s_h^k, a_h^k)^{\top}}\geq \frac{\kappa}{\gamma^2}\,.
\end{align*}
    after running at most 
    \begin{align*}
        K=\mathcal{O}\bracket{\frac{dH}{\kappa\gamma^2}}
    \end{align*}
    where $\mathcal{O}$ hides polynomial dependencies on $H$ and $\log \frac{1}{\delta}$.
\end{theorem}

\paragraph{Data Collection Protocol} We jointly explore the state-action feature space $\phi(s,a)$ and the context feature space $\psi(c)$ by combining our exploration algorithm with a G-optimal design. Specifically, for each context $c \in \supp(p_c)$, we execute the exploration procedure described in Theorem~\ref{thm:wagen} for $K_c$ episodes. We denote the dataset collected under context $c$ as $\cD_c$, and the corresponding episode indices as $\cK_c = \{k \in [K] : c_k = c\}$. The total dataset across all contexts is defined as $\cD = \cup_{c\in\cC} \cD_c$. To estimate the model parameters, we define the context-specific covariance matrices at step $h$ as $\Lambda_h^c = \sum_{k\in\cK_c} \phi(s_h^k, a_h^k)\phi(s_h^k, a_h^k)^{\top}$, and the aggregate covariance across all episodes as $\Lambda_h = \sum_{k=1}^K \phi(s_h^k, a_h^k)\phi(s_h^k, a_h^k)^{\top}+\bm{I}$.

Building on the high probability event in Theorem \ref{thm:wagen}, we have the total number of episodes be bounded as $\mathcal{O}\bracket{\frac{\poly(d, d',H)}{\kappa\gamma^2}}$:
\begin{lemma}\label{lem:total K}
    For any $\delta>0$, with probability $1-\delta$, we have for each $c\in\supp(p_c)$ and $h\in[H]$, $\lambda_{\min}\bracket{\Lambda_h^c}\geq \frac{\kappa}{\gamma^2}$, and the total number of episodes is bounded as $K=\tilde{\mathcal{O}}\bracket{\frac{d(d')^2H}{\kappa\gamma^2}}$.
\end{lemma}
\begin{proof}
    From Theorem \ref{thm:design}, we have $|\supp(p_c)|=\mathcal{O}\bracket{(d')^2}$. Thus, the procedure in Theorem~\ref{thm:wagen} is run for at most $|\supp(p_c)|$ rounds, each using $\mathcal{O}\bracket{\frac{dH}{\kappa\gamma^2}}$ episodes, and the guarantee follows by a union bound of the event in Theorem~\ref{thm:wagen} over all rounds (with failure probability $\delta/|\supp(p_c)|$ per round).
\end{proof}

Next, we bound the error of the estimated model:
\begin{lemma}\label{lem:app linear transition}
    We estimate the transition model as:
    \begin{align*}
    \widehat{\bm{\mu}}_h=\Lambda_h^{-1}\sum_{k=1}^K \phi(s_h^k, a_h^k)\delta\bracket{s_{h+1}^k}^{\top}\,.
\end{align*}
We have the following error bound for any function $V:\cS\rightarrow\RR$ that is the optimal value function for a linear MDP: with probability at least $1-\delta$, we have
\begin{align*}
    \abs{{\bracket{\widehat{\PP}_h-\PP_h}V}(s,a)}\leq \tilde{\mathcal{O}}\bracket{\poly\bracket{d,H}\sqrt{\frac{\gamma^2}{\kappa}}}\,,
\end{align*}
for any $(s,a)\in\cS\times\cA$, $h\in[H]$ and $0<\delta<1$.
\end{lemma}
\begin{proof}
    We decompose the LHS as:
    \begin{align*}
        \abs{{\bracket{\widehat{\PP}_h-\PP_h}V}(s,a)}=&\abs{\innerproduct{\phi(s,a)}{\bracket{\widehat{\bm{\mu}}_h-\bm{\mu}_h}V}}\\
        \leq& \norm{\phi(s,a)}_{\Lambda_h^{-1}}\norm{\bracket{\widehat{\bm{\mu}}_h-\bm{\mu}_h}V}_{\Lambda_h}\,,
    \end{align*}
    where $\bracket{\widehat{\bm{\mu}}_h-\bm{\mu}_h}V=\sum_{s'\in\cS}\bracket{\widehat{\bm{\mu}}_h-\bm{\mu}_h}(s')V(s')\in\RR^d$ is a vector that does not depend on $(s,a)$.
    According to our least squares estimator, we have:
    \begin{align*}
        \widehat{\bm{\mu}}_h-\bm{\mu}_h=&\Lambda_h^{-1}\sum_{k=1}^K \phi(s_h^k, a_h^k)\delta\bracket{s_{h+1}^k}^{\top}-\Lambda_h^{-1}\bracket{\sum_{k=1}^K\phi(s_h^k,a_h^k)\PP(\cdot|s_h^k,a_h^k)+\bm{\mu}_h}\\
        =& \Lambda_h^{-1}\sum_{k=1}^K \phi(s_h^k, a_h^k)\bracket{\delta\bracket{s_{h+1}^k}^{\top}-\PP(\cdot|s_h^k,a_h^k)} - \Lambda_h^{-1}\bm{\mu}_h\,,
    \end{align*}
which is because of our linear transition function:
\begin{align*}
    \sum_{k=1}^K\phi(s_h^k,a_h^k)\PP(\cdot|s_h^k,a_h^k)+\bm{\mu}_h=\sum_{k=1}^K\phi(s_h^k,a_h^k)\phi(s_h^k,a_h^k)^{\top}\bm{\mu}_h+\bm{\mu}_h=\Lambda_h\bm{\mu}_h\,.
\end{align*}
Thus, we have:
\begin{align*}
    \norm{\bracket{\widehat{\bm{\mu}}_h-\bm{\mu}_h}V}_{\Lambda_h}\leq& \norm{\sum_{k=1}^K\phi(s_h^k,a_h^k)\bracket{V(s_{h+1}^k)-\EE\mbracket{V(s_{h+1}^k)|s_h^k,a_h^k}}}_{\Lambda_h^{-1}}+\sqrt{d}\,.
\end{align*}
Applying Lemma~\ref{lem:v} and Lemma~\ref{lem:covering} with the covering parameter $\epsilon=1/K$, we have:
\begin{align*}
    \norm{\bracket{\widehat{\bm{\mu}}_h-\bm{\mu}_h}V}_{\Lambda_h}\leq \poly\bracket{d,H,\log K, \log\frac{1}{\delta}}\,.
\end{align*}
On the other hand, based the event in on Lemma~\ref{lem:total K}, we have $\Lambda_h\succeq\sup_{c\in\supp(p_c)}\Lambda_h^c\succeq \frac{\kappa}{\gamma^2}\bm{I}$, thus we have $\norm{\phi(s,a)}_{\Lambda_h^{-1}}\leq \sqrt{\frac{\gamma^2}{\kappa}}$. Also, we have $K=\tilde{\mathcal{O}}\bracket{\frac{d(d')^2H}{\kappa\gamma^2}}$ by Lemma~\ref{lem:total K}.
Putting these together proves the lemma.
\end{proof}
Now, we move on to estimate the reward parameters $\Theta_h$. Note that the reward mean can be written as:
\begin{align*}
    \EE[r_h(s,a;c)]=\phi(s,a)^\top\Theta_h\psi(c)=\bracket{\psi(c)^{\top}\otimes\phi(s,a)^{\top}}\ve\bracket{\Theta_h}
\end{align*}
Denote $\varphi(s,a,c)$ as $\psi(c)\otimes\phi(s,a)$ and $\ve\bracket{\Theta_h}={\vartheta}_h$, we have $\phi(s,a)^\top\Theta_h\psi(c)=\varphi(s,a,c)^\top\vartheta_h$. 
We have the following results for the reward estimate $\widehat{r}$:
\begin{lemma}\label{lem:app linear r bound}
    Denote, for each $h\in[H]$,
\begin{align*}
    \Sigma_h=\bracket{\bm{I}+\sum_{k=1}^K \varphi(s_h^k,a_h^k,c_k)\varphi(s_h^k,a_h^k,c_k)^{\top}}\in\RR^{dd'\times dd'}\,,
\end{align*}
then we estimate $\widehat{\vartheta}_h$, the vectorization of $\widehat\Theta_h$, using linear regression:
\begin{align*}
   \widehat{\vartheta}_h=\Sigma_h^{-1}\sum_{k=1}^K \varphi(s_h^k,a_h^k,c_k)r_h(s_h^k,a_h^k,c_k)\,.
\end{align*}
We have with probability at least $1-\delta$ for $0<\delta<1$, for all $(s,a,c)\in\cS\times\cA\times\cC$,
\begin{align*}
    \abs{\widehat{r}_h(s,a;c)-r_h(s,a;c)}\leq \tilde{\mathcal{O}}\bracket{\poly\bracket{d,d'}\sqrt{\frac{\gamma^2}{\kappa}}}\,.
\end{align*}
\end{lemma}
\begin{proof}
    Using the linear regression estimator, we have:
    \begin{align*}
        \widehat{\vartheta}_h-\vartheta_h=\Sigma_h^{-1}\sum_{k=1}^K \varphi(s_h^k,a_h^k,c_k)\bracket{r_h(s_h^k,a_h^k,c_k)-\EE\mbracket{r_h(s_h^k,a_h^k,c_k)}}-\Sigma_h^{-1}\vartheta_h\,.
    \end{align*}
Thus, we can bound the estimation error of as:
\begin{align*}
    &\abs{\widehat{r}_h(s_h,a_h,c)-r_h(s_h,a_h,c)}\\
    =&\abs{\innerproduct{\varphi(s_h,a_h,c)}{\widehat{\vartheta}_h-\vartheta_h}}\\
    \leq &\norm{\varphi(s_h,a_h,c)}_{\Sigma_h^{-1}}\norm{-\vartheta_h + \sum_{k=1}^K\varphi(s_h^k,a_h^k,c_k)\eta_h^k}_{\Sigma_h^{-1}}\\
    \leq& \tilde{\mathcal{O}}\bracket{\sqrt{d{d'}} + \sqrt{dd'+\log\frac{1}{\delta}}}\norm{\varphi(s_h,a_h,c)}_{\Sigma_h^{-1}}
\end{align*}
where $\eta_h^k$ is the noise of the random reward:
\begin{align*}
    \eta_h^k=\bracket{r_h(s_h^k,a_h^k,c_k)-\EE\mbracket{r_h(s_h^k,a_h^k,c_k)}}\,,
\end{align*}
which is bounded in $[-1,1]$ and hence $1$-subGaussian, and we use Lemma~\ref{lem:hoeffding} to bound $\norm{\sum_{k=1}^K\varphi(s_h^k,a_h^k,c_k)\eta_h^k}_{\Sigma_h^{-1}}$; the term $\norm{\vartheta_h}_{\Sigma_h^{-1}}\leq\norm{\vartheta_h}\leq\sqrt{dd'}$ uses $\Sigma_h\succeq\bm{I}$ together with the normalization $\norm{\Theta_h}_F\leq\sqrt{dd'}$ of Assumption~\ref{ass:reward rep}.

We now process the term $\norm{\varphi(s_h,a_h,c)}_{\Sigma_h^{-1}}$. First, we lower bound $\Sigma_h$:
    \begin{align*}
    &\sum_{k=1}^K \varphi(s_h^k,a_h^k,c_k)\varphi(s_h^k,a_h^k,c_k)^{\top}\\
    =& \sum_{k=1}^K \bracket{\psi(c_k)\otimes\phi(s_h^k,a_h^k)}\bracket{\psi(c_k)^\top\otimes\phi(s_h^k,a_h^k)^\top}\\
    =& \sum_{k=1}^K\bracket{\psi(c_k)\psi(c_k)^{\top}}\otimes\bracket{\phi(s_h^k,a_h^k)\phi(s_h^k,a_h^k)^\top}\\
    =&\sum_{c\in\supp(p_c)}\bracket{\psi(c)\psi(c)^{\top}}\otimes\bracket{\sum_{k\in\cK_c}\bracket{\phi(s_h^k,a_h^k)\phi(s_h^k,a_h^k)^\top}}\\
    \succeq& \bracket{\sum_{c\in\supp(p_c)}\bracket{\psi(c)\psi(c)^{\top}}}\otimes \frac{\kappa}{\gamma^2}\bm{I}
\end{align*}
where the last step uses Lemma~\ref{lem:total K} and the fact that $A\otimes B\succeq A\otimes B'$ whenever $A\succeq 0$ and $B\succeq B'$.

Denote $G=\sum_{c\in\supp(p_c)}\bracket{\psi(c)\psi(c)^{\top}}$. Since $p_c\in[0,1]$ we have $G\succeq \sum_{c\in\supp(p_c)}p_c\bracket{\psi(c)\psi(c)^{\top}}=W$, so by the property of the G-optimal design in Theorem~\ref{thm:design}, for any $c$,
\begin{align*}
    \norm{\psi(c)}_{G^{-1}}\leq \norm{\psi(c)}_{W^{-1}}\leq \sqrt{d'}\,.
\end{align*}

As a result, we have:
\begin{align*}
    &\norm{\varphi(s_h,a_h,c)}_{\Sigma_h^{-1}}\\
    =&\norm{\bracket{\psi(c)\otimes\phi(s_h,a_h)}\Sigma_h^{-1/2}}_2 \\
    \leq & \norm{\bracket{\psi(c)\otimes\phi(s_h,a_h)}\bracket{G\otimes \frac{\kappa}{\gamma^2}\bm{I}}^{-1/2}}\\
    \leq&\sqrt{\frac{\gamma^2}{\kappa}}\norm{\bracket{\psi(c)G^{-1/2}}\otimes \phi(s_h,a_h)}\\
    =& \sqrt{\frac{\gamma^2}{\kappa}} \norm{\psi(c)}_{G^{-1}}\norm{\phi(s_h,a_h)}\\
    \leq& \sqrt{\frac{\gamma^2}{\kappa}} \sqrt{d'}
\end{align*}
using $\norm{\phi(s_h,a_h)}\leq1$. Thus we have with probability at least $1-\delta$, for any $(s_h,a_h,c)$:
\begin{align*}
     \abs{\widehat{r}_h(s_h,a_h,c)-r_h(s_h,a_h,c)}\leq \tilde{\mathcal{O}}\bracket{\sqrt{\frac{\gamma^2}{\kappa}}\sqrt{d}\,d'}\,.
\end{align*}
\end{proof}
\begin{proof}[Proof of Part \textbf{(i)}, Theorem~\ref{thm:main linear}]
Finally, we bound the error of $\widehat{\pi}$ which is the optimal policy in the estimated MDP $\widehat{\cM}$, against $\pi^*$ which is the optimal policy of $\cM$. For the value function of a policy $\pi$ in $\cM$ and $\widehat{\cM}$, we denote them as $V^{\pi}$ and $\widehat{V}^{\pi}$, respectively. We can decompose the optimality gap as:
\begin{align*}
    V^*-V^{\widehat{\pi}}=\bracket{V^{\pi^*}-\widehat{V}^{\pi^*}} + \bracket{\widehat{V}^{\pi^*}-\widehat{V}^{\widehat{\pi}}} +\bracket{ \widehat{V}^{\widehat{\pi}}-V^{\widehat{\pi}}}\leq \bracket{V^{\pi^*}-\widehat{V}^{\pi^*}}+\bracket{ \widehat{V}^{\widehat{\pi}}-V^{\widehat{\pi}}}\,.
\end{align*}
Using the simulation lemma, we have:
\begin{align*}
    V^{\pi^*}-\widehat{V}^{\pi^*}=&\EE_{\widehat{\cM}, \pi^*}\mbracket{\sum_{h=1}^H \bracket{r_c(s_h,a_h)-\widehat{r}_c(s_h,a_h)}+\bracket{\PP_h-\widehat{\PP}_h}V^{\pi^*}(s_h,a_h)}\,,\\
    \widehat{V}^{\widehat{\pi}}-{V}^{\widehat{\pi}}=&\EE_{\cM, \widehat{\pi}}\mbracket{\sum_{h=1}^H \bracket{\widehat{r}_c(s_h,a_h)-r_c(s_h,a_h)}+\bracket{\widehat{\PP}_h-\PP_h}\widehat{V}^{\widehat{\pi}}(s_h,a_h)}\,.
\end{align*}
Here $\widehat{V}^{\widehat{\pi}}$ is computed with the clipped backups of Remark~\ref{rem:clipping}, so both $V^{\pi^*}$ and $\widehat{V}^{\widehat{\pi}}$ take values in $[0,H]$ and both lie in a class of the form $\cV_R$ covered by Lemma~\ref{lem:covering}. Invoking Lemma~\ref{lem:app linear transition} and Lemma~\ref{lem:app linear r bound}, we have for any $(s,a,c)\in\cS\times\cA\times\cC$:
\begin{align*}
    &\abs{\widehat{r}_c(s_h,a_h)-r_c(s_h,a_h)}\leq\tilde{\mathcal{O}}\bracket{\sqrt{\frac{\gamma^2}{\kappa}}\poly\bracket{d,d'}}\,,\\
    &\abs{{\bracket{\widehat{\PP}_h-\PP_h}V}(s,a)}\leq \tilde{\mathcal{O}}\bracket{\poly\bracket{d,H}\sqrt{\frac{\gamma^2}{\kappa}}}\,,~~V\in\sets{\widehat{V}^{\widehat{\pi}},V^*}\,.
\end{align*}
Based on Lemma~\ref{lem:total K}, we have $K=\Tilde{\mathcal{O}}\bracket{\frac{d(d')^2H}{\kappa\gamma^2}}$, i.e.\ $\sqrt{\frac{\gamma^2}{\kappa}}=\tilde{\mathcal{O}}\bracket{\frac{\sqrt{d(d')^2H}}{\kappa\sqrt{K}}}$, so both $V^{\pi^*}-\widehat{V}^{\pi^*}$ and $\widehat{V}^{\widehat{\pi}}-{V}^{\widehat{\pi}}$ are bounded by $\tilde{\mathcal{O}}\bracket{\frac{\poly\bracket{d,d',H}}{\kappa\sqrt{K}}}$. We then obtain the desired result
\begin{align*}
    V^*-V^{\widehat{\pi}}\leq \tilde{\mathcal{O}}\bracket{\min\sets{H,~\frac{\poly\bracket{d,d',H}}{\kappa\sqrt{K}}}}\,,
\end{align*}
where, as in \eqref{eq:tabular mb rate}, the $\min$ with $H$ records that the bound is informative only once $K\gtrsim\poly(d,d',H)/\kappa^2$; the guarantee holds simultaneously for every context $c\in\cC$ and in particular does not depend on $|\cC|$.
\end{proof}

\subsubsection{Imitation Learning Agent}\label{app:linear imitation}
Here we provide proof for Part \textbf{(ii)} of Theorem~\ref{thm:main linear}.

\paragraph{Positive Result in Single Task Learning} In the single task setting, according to the properties of linear MDP, for each $h\in[H]$ there exists $w_h^*\in\RR^d$ such that $Q_h^*(s,a)=\innerproduct{\phi(s,a)}{w_h^*}$, so the optimal policy can be represented as a linear classifier using $Q_h^*$: $\pi_h^*(s)=\argmax_a Q_h^*(s,a)=\argmax_a\innerproduct{\phi(s,a)}{w_h^*}$. Following the results in \citet{rajaraman2021value}, we can learn the optimal linear policy with suboptimality gap $\tilde{\mathcal{O}}\bracket{\frac{\poly(d,H)}{K}}$ through imitation learning with empirical risk minimization (ERM).

Next, we provide the proof of the hardness result in the multi-task setting. Let $\cF$ be a class of score functions $f(\phi(s,a);\psi(c))$ inducing the context-conditioned classifier policies $\pi_f(s;c)=\argmax_{a}f(\phi(s,a);\psi(c))$. To prove that no $f\in\cF$ satisfies $\pi^*(s;c)=\argmax_a f(\phi(s,a);\psi(c))$ when $\cF$ has low polynomial degree in $\psi(c)$, we construct a set of contexts $c$ that are necessarily roots of $f(\phi(s,a_1);\psi(c))-f(\phi(s,a_2);\psi(c))$, by exhibiting contexts at which the optimal action changes from $a_1$ to $a_2$. By creating a large set of such roots we show that $f$ must have high-order dependency on $\psi(c)$. Notice that we do not require $Q^*\in\cF$ for the optimal policy to be realizable, although the optimal Q-function $Q_h^*(s,a;c)$ also depends on $\psi(c)$ in a complicated way. 

Here, we provide the construction of the hard instance. We construct our MDP with two levels $H=2$ and feature dimension $d=3$. We denote actions as $\cA=\sets{1,2,\cdots,|\cA|}$. Now, we define the feature maps as follows. At the fixed initial state $s_1$, the actions have features $\phi(s_1,1)=[1,0,0]$, $\phi(s_1,2)=[0,1,0]$ and for all actions $a=3, \cdots,|\cA|$, we have $\phi(s_1,a)=[0,0,1]$. We have three states at $h=2$, with $\bm{\mu}_1(1)=[1,0,0]$, $\bm{\mu}_1(2)=[0,1,0]$ and $\bm{\mu}_1(3)=[0,0,1]$. So taking action $1$ will move to state $1$, taking action $2$ will move to state $2$ and taking action $a\geq 3$ will move to state $3$. We design state $3$ as an empty state, such that for all actions in state $3$, the feature $\phi(3,a)=\mathbf{0}$ so the reward is always $0$ in state $3$. As a result, the value of state $3$ is always $0$.

At the second level, for the actions $a=1,2,\cdots, |\cA|$, we design the feature set in states $s_2=1,2$ to be
\begin{align*}
    \phi(1,a)=\frac{1}{\sqrt2}\mbracket{\cos \pi\tfrac{2a}{|\cA|}, \sin \pi\tfrac{2a}{|\cA|},1}\,,\qquad
    \phi(2,a)=\frac{1}{\sqrt2}\mbracket{\cos \pi\tfrac{2a-1}{|\cA|}, \sin \pi\tfrac{2a-1}{|\cA|},1}\,.
\end{align*}
The factor $\frac{1}{\sqrt2}$ is needed: $\mbracket{\cos\alpha,\sin\alpha,1}$ has Euclidean norm $\sqrt2$ and would violate the normalization $\norm{\phi(s,a)}\leq1$ required by Definition~\ref{def:linear mdp}, which the rest of the analysis (notably Lemma~\ref{lem:hoeffding} and Lemma~\ref{lem:covering}) relies on.

For the context dependent reward, we define the continuous context set as $\cC=[0,1)$, and the context feature as $\psi(c)=\frac{1}{\sqrt2}\mbracket{\cos 2\pi c, \sin 2\pi c, 1}$, so that $\norm{\psi(c)}=1$. We set
\begin{align*}
    \Theta_1=\bm{0}\,,\qquad \Theta_2=\bm{I}\,,
\end{align*}
which is the step-dependent reward representation of Assumption~\ref{ass:reward rep}. The instance meets that assumption's normalization, since $\norm{\Theta_1}_F=0$ and $\norm{\Theta_2}_F=\sqrt{3}\leq3=\sqrt{dd'}$ here. The choice $\Theta_1=\bm{0}$ is what makes the reward at the first level identically $0$; with a single step-independent $\Theta=\bm{I}$ one would instead get $r_1(s_1,1;c)=\frac{1}{2}\cos2\pi c\neq0$, contradicting the construction. With $\Theta_1=\bm 0$ the first-level reward vanishes and the Q-function $Q_1^*(s_1,a;c)$ of all actions $a\geq3$ equals $0$, while $Q_1^*(s_1,1;c)=V_2^*(1;c)>0$ and $Q_1^*(s_1,2;c)=V_2^*(2;c)>0$, so the optimal first-level action is always $1$ or $2$.

Thus, we have the reward at the second level defined as $r_2(1,a;c)=\frac{1}{2}\bracket{\cos\bracket{2\pi c-2\pi\frac{a}{|\cA|}}+1}\in[0,1]$, and $r_2(2,a;c)=\frac{1}{2}\bracket{\cos\bracket{2\pi c - \pi\frac{2a-1}{|\cA|}}+1}\in[0,1]$. Suppose $2|\cA|c\in[n_c, n_c+1]$ where $n_c$ is an integer, then at the second level the optimal value function for the two states are: 
\begin{align*}
    V^*(1,c)=&\max_{i}\frac{1}{2}\cos\bracket{2\pi c-2\pi\frac{i}{|\cA|}}+\frac{1}{2}\\
    =&\max_{i}\frac{1}{2}\cos\bracket{\frac{\pi}{|\cA|}\bracket{2|\cA|c-2i}} + \frac{1}{2} \\
    =&\frac{1}{2}\cos\bracket{2\pi c-\pi\frac{2\lfloor\frac{n_c+1}{2}\rfloor}{|\cA|}}+\frac{1}{2}\\
    V^*(2,c)=&\max_{i}\frac{1}{2}\cos\bracket{2\pi c-\pi\frac{2i-1}{|\cA|}}+\frac{1}{2}\\
    =& \max_{i}\frac{1}{2}\cos\bracket{\frac{\pi}{|\cA|}\bracket{2|\cA|c-(2i-1)}}+\frac{1}{2}\\
    =&\frac{1}{2}\cos\bracket{2\pi c-\pi\frac{2\lfloor\frac{n_c}{2}\rfloor+1}{|\cA|}}+\frac{1}{2}
\end{align*} 
Thus, the optimal action $\pi(s_1;c)$ is $1$ if $V^*(1,c)>V^*(2,c)$ and $2$ otherwise, since taking action $1$ will move to state $1$ and taking action $2$ will move to state $2$. 

Now, assume that the optimal policy can be represented by a function 
\begin{align*}
    \pi^*(s_1;c)=\argmax_{a\in\cA} f_a(\psi(c))\,,
\end{align*} 
where, as established above, the maximum is attained at $a\in\sets{1,2}$ for every $c$, and where we assume $f_a$ can be written as a polynomial in the features $\psi(c)$ with degree $n$:
\begin{align*}
    f_a(\psi(c))=\sum_{u+v+w\leq n}g_{a}^{u+v+w}\psi(c)_0^{u}\psi(c)_1^{v}\psi(c)_2^w=\sum_{u+v\leq n}g_{a}^{u,v}\sin^u 2\pi c\cos^v 2\pi c\,,a=1,2\,.
\end{align*}
Consider $f(\psi(c))=f_1(\psi(c))-f_2(\psi(c))$, which by the display above is a trigonometric polynomial of degree at most $n$ in $2\pi c$. Notice that at the points $c_i=\frac{2i+1}{4|\cA|}$, $i=0,1,2,\cdots, 2|\cA|-1$, we have $V^*(2,c_i)=V^*(1,c_i)$; and for any small enough $\epsilon>0$,
\begin{align*}
    \bracket{V^*(2,c_i-\epsilon)-V^*(1,c_i-\epsilon)}\bracket{V^*(2,c_i+\epsilon)-V^*(1,c_i+\epsilon)}<0\,.
\end{align*}
That is, the optimal action $\pi^*(s_1;c)$ switches between $1$ and $2$ as $c$ moves past $c_i$. If $\pi^*(s_1;c)=\argmax_{a}f_a(\psi(c))$ realizes the optimal policy, then $f$ must be $>0$ on one side of $c_i$ and $<0$ on the other, and since $f$ is continuous it has a root in every neighborhood of $c_i$; taking $\epsilon\downarrow0$ and using continuity again gives $f(\psi(c_i))=0$. In particular $f\not\equiv0$, since a vanishing $f$ would leave the two actions tied everywhere and could not induce a policy that switches. By Lemma~\ref{lem:polynomial}, a non-zero trigonometric polynomial of degree $n$ has at most $2n$ roots in $[0,2\pi)$, and we have exhibited $2|\cA|$ distinct roots $c_i$. Hence $2n\geq2|\cA|$, i.e.\ $n\geq|\cA|$.

\begin{remark}[What the two linear results do and do not compare]\label{rem:linear comparison}
The two parts of Theorem~\ref{thm:main linear} are not two points on one axis. Part \textbf{(i)} is an estimation rate in $K$; Part \textbf{(ii)} is a statement about what a class can express, and it holds for every $K$. What can be compared is the fixed description the planner carries against the degree the classifier is forced to. The model-based planner is described by $\bm{\mu}_1,\ldots,\bm{\mu}_H$ and $\Theta_1,\ldots,\Theta_H$, a description that does not grow with $|\cA|$, and it drives its error to $0$ at rate $\tilde{\mathcal{O}}((\kappa\sqrt K)^{-1})$ with that fixed budget. Part \textbf{(ii)} says the classifier's score must instead reach degree $|\cA|$ in the task feature, and any class of lower degree carries the resulting bias however much data it is given.

We deliberately stop at the degree and do not convert it into a count of parameters. It is tempting to say that a degree-$n$ polynomial in $\psi(c)\in\RR^{d'}$ has $\binom{n+d'}{d'}$ coefficients and therefore that realizability costs $\Omega(|\cA|^{d'})$ of them, but that counts a dense expansion and our argument does not rule out a sparse or otherwise structured score of the same degree being realized with far fewer free parameters. Our root-counting argument (Lemma~\ref{lem:polynomial}) constrains only the degree, so that is all we claim. Neither learner is spared an estimation problem here: the model-based learner never sees $\Theta_h$ and must recover it from noisy rewards (Lemma~\ref{lem:app linear r bound}), while the imitation learner sees optimal action labels the model-based learner never does. The asymmetry behind the tabular result (Remark~\ref{rem:information}) is gone, and what is left is a question of which representation the structure of the task happens to be simple in.
\end{remark}

\subsubsection{Technical Lemmas}
Here we provide the concentration inequality for vector-valued martingales, which will be used in our analysis of the model-based method in the linear function approximation setting.
\begin{lemma}[Hoeffding concentration inequality for vector-valued martingales, Theorem 1 in \citet{abbasi2011improved}]\label{lem:hoeffding}
Let $\sets{\eta_{\tau}}_{\tau=1}^{\infty}$ be a stochastic process with corresponding filtration $\sets{\cF_{\tau}}_{\tau=1}^{\infty}$. Let $\eta_\tau|\cF_{\tau-1}$ be zero mean and $\sigma$-subGaussian. Let $\sets{\phi_{\tau}}_{\tau=1}^{\infty}$ be a $\RR^d$-valued stochastic process with $\phi_{\tau}\in\cF_{\tau-1}$, and $\norm{\phi_{\tau}}\leq 1$. Let $\Lambda_k=\lambda \bm{I} +\sum_{\tau=1}^k \phi_{\tau}\phi_{\tau}^{\top}$. Then for any $0<\delta<1$, with probability at least $1-\delta$, we have for all $k>1$,
\begin{align*}
    \norm{\sum_{\tau=1}^k\phi_{\tau}\eta_{\tau}}_{\Lambda_k^{-1}}^2\leq \sigma^2d\log(1+k/(d\lambda))+\log \frac{1}{\delta}
\end{align*}
\end{lemma}
When bounding the learning error for the transition model, we need to apply the above bound with respect to the value functions $v\in\cV$, where we have the following lemma:
\begin{lemma}[Lemma D.4 of \citet{jin2023provably}]\label{lem:v}
Let $\sets{x_{\tau}}_{\tau=1}^{\infty}$ be a stochastic process on state space $\cS$ with corresponding filtration $\sets{\cF_{\tau}}_{\tau=1}^{\infty}$. Let $\sets{\phi_{\tau}}_{\tau=1}^{\infty}$ be a $\RR^d$-valued stochastic process with $\phi_{\tau}\in\cF_{\tau-1}$, and $\norm{\phi_{\tau}}\leq 1$. Let $\Lambda_k=\lambda \bm{I} +\sum_{\tau=1}^k \phi_{\tau}\phi_{\tau}^{\top}$. Then for any $\delta>0$, and any $\tau>0$ and $v\in\cV$ that $\sup_x\abs{V(x)}\leq H$, we have:
\begin{align*}
    \norm{\sum_{\tau=1}^k\phi_{\tau}\sets{V(x_{\tau})-\EE\mbracket{V(x_{\tau}|\cF_{\tau-1})}}}_{\Lambda_k^{-1}}^2\leq4H^2 \mbracket{\frac{d}{2}\log\bracket{\frac{k+\lambda}{\lambda}}+\log\cN_{\epsilon}+\log \frac{1}{\delta}} + \frac{8k^2\epsilon^2}{\lambda}\,.  
\end{align*}
    where $\cN_{\epsilon}$ is the $\epsilon$-covering number of $\cV$ with respect to the distance defined as the maximum distance over all states in $x\in\cS$: $\dist\bracket{V,V'}=\sup_x\abs{V(x)-V'(x)}$. 
\end{lemma}

Given the linear MDP assumption, we have the following bound on the covering number of optimal value functions:
\begin{lemma}\label{lem:covering}
    Fix $R>0$ and let
    \begin{align*}
        \cV_R=\sets{V:\cS\rightarrow[0,H]~\Big|~V(s)=\min\sets{\max_{a\in\cA}\innerproduct{\phi(s,a)}{w},~H}~\text{for some}~w\in\RR^d,~\norm{w}\leq R}\,.
    \end{align*}
    Then the covering number of $\cV_R$ satisfies $\log \cN_{\epsilon}\leq d\log \bracket{1+2R/\epsilon}$. In particular, the optimal value functions of any linear MDP with feature map $\phi$ satisfying Definition~\ref{def:linear mdp} lie in $\cV_R$ with $R=2H\sqrt{d}$, giving $\log\cN_\epsilon\leq d\log\bracket{1+4H\sqrt d/\epsilon}$.
\end{lemma}
\begin{proof}
    Using Lemma B.1 of \citet{jin2023provably}, for any policy $\pi$ the Q-function of a linear MDP satisfying Definition~\ref{def:linear mdp} can be written as $Q_h^{\pi}(s,a)=\innerproduct{\phi(s,a)}{w_h^\pi}$ with $\norm{w_h^\pi}\leq 2H\sqrt{d}$, and $V_h^*(s)=\max_{a}\innerproduct{\phi(s,a)}{w_h^*}$, which is already in $[0,H]$ so the truncation is inactive. For two members of $\cV_R$ with parameters $w,w'$, since truncation and $\max_a$ are $1$-Lipschitz,
    \begin{align*}
        \dist(V,V')=&\sup_s\abs{V(s)-V'(s)}\\
        \leq&\sup_{s,a}\abs{\innerproduct{\phi(s,a)}{w-w'}}\\
        \leq& \norm{w-w'}\,,
    \end{align*}
    using $\norm{\phi(s,a)}\leq1$. Thus the $\epsilon$-covering number of $\cV_R$ is bounded by the $\epsilon$-covering number of the Euclidean ball of radius $R$, which is $\log \cN_{\epsilon}\leq d\log \bracket{1+2R/\epsilon}$ by Lemma D.5 of \citet{jin2023provably}.
\end{proof}
\begin{remark}[Why the estimated MDP needs the truncated class]\label{rem:clipping}
The value functions that appear in the proof of Part \textbf{(i)}, Theorem~\ref{thm:main linear}, are $V^*$, the optimal value function of the \emph{true} MDP $\cM$, and $\widehat{V}^{\widehat{\pi}}=\widehat{V}^*$, the optimal value function of the \emph{estimated} MDP $\widehat{\cM}$. Only the first is covered by the plain statement for linear MDPs: the estimate $\widehat{\bm{\mu}}_h$ is a least-squares solution and is not guaranteed to satisfy the normalization $\norm{\widehat{\bm{\mu}}_h(\cS)}\leq\sqrt{d}$ of Definition~\ref{def:linear mdp}, so $\widehat{\cM}$ need not be a linear MDP in the normalized sense and its optimal value function need not lie in $[0,H]$. We therefore (i) define the planner to compute $\widehat{V}_h=\min\sets{\max_a\innerproduct{\phi(\cdot,a)}{\widehat{w}_h},H}$, clipping each backup to the range of the true value function, and (ii) note that $\norm{\widehat{\bm{\mu}}_h(\cS)}=\norm{\Lambda_h^{-1}\sum_{k}\phi(s_h^k,a_h^k)}\leq K$ because $\Lambda_h\succeq\bm{I}$, so $\widehat{V}_h\in\cV_R$ with $R=\mathcal{O}(HK)$. Applying Lemma~\ref{lem:covering} with this $R$ changes $\log\cN_\epsilon$ from $d\log\bracket{1+4H\sqrt{d}/\epsilon}$ to $\mathcal{O}\bracket{d\log\bracket{HK/\epsilon}}$, which is absorbed by the $\tilde{\mathcal{O}}$ in Lemma~\ref{lem:app linear transition}. Clipping is what makes the Bellman backups in $\widehat{\cM}$ type-correct as elements of $[0,H]^{\cS}$; without it neither the covering argument nor the simulation lemma below is valid.
\end{remark}
This lemma is used to bound the degree of the optimal function in the hard instance for imitation learning.
\begin{lemma}\label{lem:polynomial}
    Consider a $n$-degree polynomial defined on $\mbracket{\sin x,  \cos x}$:
    \begin{align*}
        f(x)=\sum_{u+v\leq n}g^{u,v}\sin^u x\cos^v x\,.
    \end{align*}
    Then it has at most $2n$ roots in $x\in[0,2\pi)$.
\end{lemma}
\begin{proof}
    We can write $\sin x=\frac{z-z^{-1}}{2i}$ and $\cos x=\frac{z+z^{-1}}{2}$, where $z=\cos x+i\sin x$. Thus, $z^nf$
 becomes a $2n$-degree univariate polynomial in $z$. Since a  univariate polynomial with degree $2n$ can have at most $2n$ complex roots, we have at most $2n$ roots for $x\in[0,2\pi)$.
 \end{proof}

\subsection{Error Bound of WAP}
\label{sec:wap-error-bound}

In this section, we analyze the error of World Action Planner (WAP) and characterize how its performance depends on the initial proposal, global evaluation and correction, the number of local candidate evaluations, world-model prediction, and local candidate ranking.

\paragraph{Status and Scope of this Analysis} What follows is an error \emph{decomposition}, and it is worth being clear about the difference between that and a performance guarantee. The quantities $\epsilon_{\texttt{Agent.Proposal}}$, $\epsilon_{\texttt{Agent.Eval}}$, $\epsilon_{\texttt{Agent.Rank}}$ and $\bar\epsilon_{\mathrm{WM}}$ are \emph{defined} to be the errors of the corresponding learned components (Definitions~\ref{def:wap-agent-eval}, \ref{def:wap-wm} and \ref{def:wap-agent-rank}). We neither assume they are small nor bound them by a separate argument, and no line below establishes that a VLM or a video world model is accurate; Section~\ref{sec:wap-instantiation} is where those quantities meet measurements. What the analysis does supply is the rule by which the errors combine: which of them a search radius can swallow, which of them the candidate budget never enters, how the two search stages divide the work between them, and what the candidate budget buys at a given radius. Being an upper bound, it can say that a larger $N$ stops improving the \emph{bound}; it cannot say that a larger $N$ stops improving the system. Only one substantive modeling assumption enters, the local regularity of $Q$ in Assumption~\ref{ass:wap-lip-q}; the rest of the work is geometry, in Assumption~\ref{ass:wap-local-cover} and Lemma~\ref{lem:wap-candidate-cover}. A decomposition whose terms were all removable would say nothing, so Proposition~\ref{prop:wap-tight} exhibits, for each term, an instance of the assumption class in which it alone is responsible for the error. That is a worst-case statement about the class, and it does not claim our tasks are such instances.

One further restriction: all of this concerns a \emph{single} planning decision, with $\hat a$ the action WAP commits to at one replanning step and $Q$ the action value there. Errors compounding around a closed loop are outside its scope. Replanning at every step is what keeps $\epsilon_{\texttt{Agent.Proposal}}$ from accumulating in practice, but we do not analyze that here.

Given state $s$ and task description $\ell$, we denote the action-value function of taking action $a\in\cA$ as $Q(s,a;\ell)$. In the following analysis, when $s$ and $\ell$ are fixed, we abuse notation and simply write $Q(a)$. We denote an optimal action at the current state as
\begin{align*}
a^* \in \argmax_{a\in\cA} Q(a).
\end{align*}

The action space $\cA$ is continuous. In WAP, $a$ can be interpreted as the low-dimensional goal parameter optimized by the planner, e.g., an end-effector goal pose, which is subsequently converted into an action chunk by the low-level controller.

For simplicity, we assume that the local action space considered by WAP is convex under the distance metric $d$, as is the case for the Euclidean goal-position space used in our implementation.

\paragraph{Local Regularity of the Action Value}
We first assume that actions sufficiently close to an optimal action have similar task values.

\begin{assumption}[Local regularity of $Q$]
\label{ass:wap-lip-q}
For actions in the region considered by WAP,
\begin{align*}
Q(a^*) - Q(a)
\leq
L_Q d(a,a^*),
\end{align*}
where $L_Q>0$ and $d$ is a distance metric on the action space, e.g., the Euclidean distance between end-effector goal positions.
\end{assumption}

Assumption~\ref{ass:wap-lip-q} is local and does not require the global action-value landscape to be convex or smooth.

\paragraph{Initial Proposal and Global Evaluation}
Let $a_0=\texttt{Agent.Proposal}(s,\ell)$ denote the initial action proposed by WAP. We define the proposal error as
\begin{align}
\epsilon_{\texttt{Agent.Proposal}}
=
d(a_0,a^*).
\label{eq:wap-proposal-error}
\end{align}

WAP then performs a coarse global correction by evaluating the imagined outcome of $a_0$ using \texttt{Agent.Eval}. Let $h^{\mathrm{global}}$ denote the maximum correction radius of this global stage.

Define
\begin{align}
a_{\mathrm{global}}^*
\in
\operatorname*{arg\,min}_{a:\,d(a,a_0)\leq h^{\mathrm{global}}}
d(a,a^*),
\label{eq:wap-best-global-reachable}
\end{align}
i.e., $a_{\mathrm{global}}^*$ is the point within the global correction radius that is closest to the optimal action.

If the geodesic from $a_0$ to $a^*$ stays inside the local action space --- which holds under the convexity assumption above when $a^*$ itself lies in that space --- then the point of that geodesic at distance $\min\sets{h^{\mathrm{global}},d(a_0,a^*)}$ from $a_0$ is feasible in \eqref{eq:wap-best-global-reachable}, and therefore
\begin{align}
d(a_{\mathrm{global}}^*,a^*)
\leq
\left[
\epsilon_{\texttt{Agent.Proposal}}
-
h^{\mathrm{global}}
\right]_+,
\label{eq:wap-ideal-global-residual}
\end{align}
where $[x]_+=\max\{x,0\}$. We write this as an inequality rather than an equality: equality also requires that no point of the ball be closer to $a^*$ than the geodesic point, which holds in Euclidean goal space but not for a general metric, and only the upper bound is used below.

The global correction produced by \texttt{Agent.Eval} depends on the imagined rollout of the initial proposal. To separate the intrinsic error of \texttt{Agent.Eval} from world-model prediction error, we consider the correction that would be produced if \texttt{Agent.Eval} were given the true physical rollout.

Let $a_g^{\mathrm{true}}$ denote the corrected action obtained when \texttt{Agent.Eval} observes the true rollout, and let $a_g$ denote the corrected action actually obtained from the world-model-predicted rollout.

\begin{definition}[\texttt{Agent.Eval} error]
\label{def:wap-agent-eval}
Define the value-level error of \texttt{Agent.Eval} during global correction as
\begin{align}
\epsilon_{\texttt{Agent.Eval}}
\coloneqq
L_Q
d\left(
a_g^{\mathrm{true}},
a_{\mathrm{global}}^*
\right).
\label{eq:wap-agent-eval-error}
\end{align}
\end{definition}

Note that this is a definition rather than an assumption. Whatever gap separates the correction \texttt{Agent.Eval} produces from a faithful rollout and the best correction the global radius allows, that gap is $\epsilon_{\texttt{Agent.Eval}}$. The factor $L_Q$ puts it in value units, so it can be added to the other terms of the final bound. Had we instead posited an inequality bounding this gap, the theorem would largely be repeating its own premise back. Nothing below asks $\epsilon_{\texttt{Agent.Eval}}$ to be small; Section~\ref{sec:wap-instantiation} discusses how one gets at it empirically.

\paragraph{World-Model Prediction Error}
Both \texttt{Agent.Eval} and the subsequent local candidate evaluation rely on trajectories generated by the learned world model. We therefore characterize both effects using a common underlying world-model prediction error.

For an action $a$, let $\tau(a)$ denote its true physical rollout and let $\hat{\tau}(a)$ denote the corresponding rollout predicted by the world model.

\begin{definition}[World-model prediction error]
\label{def:wap-wm}
Let
\begin{align}
\bar{\epsilon}_{\mathrm{WM}}
\coloneqq
\sup_{a}
d_{\tau}\left(
\tau(a),
\hat{\tau}(a)
\right),
\label{eq:wap-wm-rollout}
\end{align}
the supremum being over the rollouts actually considered during one WAP planning step, where $d_{\tau}$ is a task-relevant trajectory distance.
\end{definition}

The remaining content about the world model is a genuine assumption, namely that the downstream use of a rollout is stable in $d_\tau$ rather than merely measurable in it.

\begin{assumption}[Downstream stability of world-model error]
\label{ass:wap-wm}
For the global correction produced by \texttt{Agent.Eval}, we assume
\begin{align}
d\left(
a_g,
a_g^{\mathrm{true}}
\right)
\leq
L_{\mathrm{WM}}^{\mathrm{global}}
d_{\tau}\left(
\tau(a_0),
\hat{\tau}(a_0)
\right),
\label{eq:wap-global-wm-stability}
\end{align}
where $L_{\mathrm{WM}}^{\mathrm{global}}$ characterizes how world-model prediction error perturbs the global correction.

For local candidate evaluation, let $Q_{\mathrm{WM}}(a)$ denote the ideal task value associated with the world-model-predicted rollout $\hat{\tau}(a)$. We assume
\begin{align}
\left|
Q_{\mathrm{WM}}(a)-Q(a)
\right|
\leq
L_{\mathrm{WM}}^{\mathrm{local}}
d_{\tau}\left(
\tau(a),
\hat{\tau}(a)
\right)
\label{eq:wap-local-wm-stability}
\end{align}
for every candidate evaluated by WAP.
\end{assumption}

Assumption~\ref{ass:wap-wm} carries real content: it asks that the two consumers of imagined rollouts degrade gracefully as $d_\tau$ grows. It is also the assumption we would expect to fail first. A video world model can be faithful almost everywhere in $d_\tau$ and still hallucinate one contact event, and one contact event is enough to reverse a comparison. Hence the choice in Section~\ref{sec:wap-instantiation} to compare the \emph{rankings} that imagined and real rollouts induce rather than a pixel metric: reordering is the failure that matters to a planner, and $d_\tau$ is not observable on its own. That comparison probes a necessary condition rather than measuring $L_{\mathrm{WM}}^{\mathrm{local}}\bar\epsilon_{\mathrm{WM}}$, as Section~\ref{sec:wap-instantiation} spells out.

The trajectory distance $d_{\tau}$ does not need to correspond to pixel-level reconstruction error. Instead, it can capture task-relevant differences in object motion, contact outcomes, robot pose, or other physical quantities that affect action selection.

By Definition~\ref{def:wap-wm} and Assumption~\ref{ass:wap-wm},
\begin{align}
d\left(
a_g,
a_g^{\mathrm{true}}
\right)
&\leq
L_{\mathrm{WM}}^{\mathrm{global}}
\bar{\epsilon}_{\mathrm{WM}},
\label{eq:wap-global-wm-error}\\
\left|
Q_{\mathrm{WM}}(a)-Q(a)
\right|
&\leq
L_{\mathrm{WM}}^{\mathrm{local}}
\bar{\epsilon}_{\mathrm{WM}}.
\label{eq:wap-local-wm-error}
\end{align}

We can now bound the residual distance after global evaluation and correction. By the triangle inequality,
\begin{align}
d(a_g,a^*)
&\leq
d\left(
a_g,a_g^{\mathrm{true}}
\right)
+
d\left(
a_g^{\mathrm{true}},a_{\mathrm{global}}^*
\right)
+
d(a_{\mathrm{global}}^*,a^*) \nonumber\\
&\leq
\left[
\epsilon_{\texttt{Agent.Proposal}}
-
h^{\mathrm{global}}
\right]_+
+
\frac{\epsilon_{\texttt{Agent.Eval}}}{L_Q}
+
L_{\mathrm{WM}}^{\mathrm{global}}
\bar{\epsilon}_{\mathrm{WM}}
~\eqqcolon~
\Delta_{\mathrm{global}}\,.
\label{eq:wap-global-residual}
\end{align}

Thus, the residual after the global stage depends on the initial \texttt{Agent.Proposal} error, the global correction radius, the error of \texttt{Agent.Eval}, and the world-model prediction error.

\paragraph{Local Candidate Search}
After global correction, WAP constructs a finite candidate set
\begin{align*}
\cA_N=\{a_1,\ldots,a_N\}
\end{align*}
around the corrected proposal $a_g$, where $N$ is the number of local candidates simulated by the world model. Let $h^{\mathrm{local}}$ denote the radius covered by the local candidate search.

Rather than specifying the exact geometry of the candidate set, we parameterize it by the resolution it achieves.

\begin{assumption}[Local candidate coverage]
\label{ass:wap-local-cover}
The candidate set $\cA_N$ is a $\varrho_N$-net of the local search region: for every action $a\in\cA$ satisfying
\begin{align*}
d(a,a_g)\leq h^{\mathrm{local}},
\end{align*}
there exists a candidate $\tilde{a}\in\cA_N$ with $d(a,\tilde{a})\leq\varrho_N$. When the candidates form a grid or a uniformly random sample of the $d_a$-dimensional local region, the resolution obeys the standard covering rate
\begin{align}
\varrho_N
\leq
c_{d_a}
h^{\mathrm{local}}
N^{-1/d_a}
\label{eq:wap-local-cover-rate}
\end{align}
for a constant $c_{d_a}>0$ depending only on the geometry and dimension $d_a$ of the local action space.
\end{assumption}

Assumption~\ref{ass:wap-local-cover} captures the tradeoff between the number of candidate evaluations and the resolution of the local search. Increasing $N$ gives finer coverage, while increasing $h^{\mathrm{local}}$ under a fixed candidate budget gives coarser coverage. We keep $\varrho_N$ and the rate \eqref{eq:wap-local-cover-rate} apart because they are different claims, and our own implementation makes the first but not the second. WAP evaluates the five-point cross $\sets{(x,y),(x\pm\delta,y),(x,y\pm\delta)}$ (Appendix~\ref{app:search details}). That is a $\varrho_N$-net of the radius-$\delta$ region with $\varrho_N=\Theta(\delta)$, not $\Theta(\delta N^{-1/2})$: the cross probes four directions rather than tiling the disc. The rate \eqref{eq:wap-local-cover-rate} describes a dense grid, and the sampling baselines we compare against; for the cross all one can say is $\varrho_N\leq h^{\mathrm{local}}$, with refinement arriving through replanning rather than through a larger $N$ inside a single step. Everything below is written in terms of $\varrho_N$, and we substitute \eqref{eq:wap-local-cover-rate} only where the dependence on $N$ is the point at issue.

We next bound the distance between the optimal continuous action and the closest candidate generated by WAP.

\begin{lemma}[Candidate coverage error]
\label{lem:wap-candidate-cover}
Under Assumption~\ref{ass:wap-local-cover},
\begin{align}
\min_{a\in\cA_N} d(a,a^*)
\leq&
\varrho_N
+
[d(a_g,a^*)-h^{\mathrm{local}}]_+.
\label{eq:wap-candidate-cover}
\end{align}
Consequently, with $\Delta_{\mathrm{global}}$ as in \eqref{eq:wap-global-residual},
\begin{align}
\min_{a\in\cA_N} d(a,a^*)
\leq&
\varrho_N
+
\left[
\Delta_{\mathrm{global}}
-
h^{\mathrm{local}}
\right]_+
\label{eq:wap-candidate-cover-tight}\\
\leq&
\varrho_N
+
\left[
\epsilon_{\texttt{Agent.Proposal}}
-
h^{\mathrm{global}}
-
h^{\mathrm{local}}
\right]_+
\nonumber\\
&+
\frac{\epsilon_{\texttt{Agent.Eval}}}{L_Q}
+
L_{\mathrm{WM}}^{\mathrm{global}}
\bar{\epsilon}_{\mathrm{WM}}.
\label{eq:wap-candidate-cover-final}
\end{align}
The intermediate form \eqref{eq:wap-candidate-cover-tight} is the sharper of the two and the one that says something structural: the local radius absorbs the \emph{whole} global-stage residual, the \texttt{Agent.Eval} and world-model contributions along with the uncorrected proposal error. That information is what gets thrown away on the way to \eqref{eq:wap-candidate-cover-final}, which is why the additive form keeps charging $\epsilon_{\texttt{Agent.Eval}}$ in full even in cases where a generous $h^{\mathrm{local}}$ has already covered it.
\end{lemma}

\begin{proof}
We consider two cases.

If
\begin{align*}
d(a_g,a^*)\leq h^{\mathrm{local}},
\end{align*}
then $a^*$ lies inside the local search region. By Assumption~\ref{ass:wap-local-cover}, there exists a candidate $a\in\cA_N$ satisfying
\begin{align*}
d(a,a^*)
\leq
\varrho_N.
\end{align*}

Otherwise,
\begin{align*}
d(a_g,a^*)>h^{\mathrm{local}}.
\end{align*}
Since the local action space is convex, consider a point $a_{\mathrm{local}}$ on the line segment from $a_g$ toward $a^*$ such that
\begin{align*}
d(a_g,a_{\mathrm{local}})
=
h^{\mathrm{local}}.
\end{align*}
Then
\begin{align*}
d(a_{\mathrm{local}},a^*)
=
d(a_g,a^*)-h^{\mathrm{local}}.
\end{align*}

By Assumption~\ref{ass:wap-local-cover}, there exists a candidate $a\in\cA_N$ satisfying
\begin{align*}
d(a,a_{\mathrm{local}})
\leq
\varrho_N.
\end{align*}
Therefore, by the triangle inequality,
\begin{align*}
d(a,a^*)
&\leq
d(a,a_{\mathrm{local}})
+
d(a_{\mathrm{local}},a^*)\\
&\leq
\varrho_N
+
d(a_g,a^*)
-
h^{\mathrm{local}}.
\end{align*}

Combining the two cases gives
\begin{align*}
\min_{a\in\cA_N} d(a,a^*)
\leq
\varrho_N
+
[d(a_g,a^*)-h^{\mathrm{local}}]_+.
\end{align*}

Substituting Eq.~\ref{eq:wap-global-residual} and using that $x\mapsto[x-h^{\mathrm{local}}]_+$ is non-decreasing gives \eqref{eq:wap-candidate-cover-tight}:
\begin{align*}
\min_{a\in\cA_N} d(a,a^*)
\leq&
\varrho_N
+
\left[
\left[
\epsilon_{\texttt{Agent.Proposal}}
-
h^{\mathrm{global}}
\right]_+
-
h^{\mathrm{local}}
+
\frac{\epsilon_{\texttt{Agent.Eval}}}{L_Q}
+
L_{\mathrm{WM}}^{\mathrm{global}}
\bar{\epsilon}_{\mathrm{WM}}
\right]_+.
\end{align*}

For the weaker additive form, use
\begin{align*}
[x+y]_+\leq[x]_++y
\end{align*}
for $y\geq0$, together with
\begin{align*}
\left[
\left[
\epsilon_{\texttt{Agent.Proposal}}
-
h^{\mathrm{global}}
\right]_+
-
h^{\mathrm{local}}
\right]_+
=
\left[
\epsilon_{\texttt{Agent.Proposal}}
-
h^{\mathrm{global}}
-
h^{\mathrm{local}}
\right]_+,
\end{align*}
which proves Eq.~\ref{eq:wap-candidate-cover-final}.
\end{proof}

Lemma~\ref{lem:wap-candidate-cover} reveals the complementary roles of the global and local search stages. Ignoring learned-component errors, the two stages can cover proposal error as large as
\begin{align*}
h^{\mathrm{global}}+h^{\mathrm{local}},
\end{align*}
while the approximation error within the covered region is $\varrho_N$, which decreases as $N^{-1/d_a}$ for a dense grid.

Let
\begin{align*}
a_{\cA_N}^*
\in
\argmax_{a\in\cA_N} Q(a)
\end{align*}
denote the best action among the $N$ generated candidates.

Let $a_{\mathrm{near}}\in\cA_N$ denote a candidate closest to $a^*$. Since $a_{\cA_N}^*$ is the best candidate according to the true action value,
\begin{align*}
Q(a_{\cA_N}^*)
\geq
Q(a_{\mathrm{near}}).
\end{align*}

By Assumption~\ref{ass:wap-lip-q} and Lemma~\ref{lem:wap-candidate-cover}, provided $a^*$ and the candidates all lie in the region where Assumption~\ref{ass:wap-lip-q} applies,
\begin{align}
Q(a^*)-Q(a_{\cA_N}^*)
\leq&
L_Q
\varrho_N
\nonumber\\
&+
L_Q
\left[
\epsilon_{\texttt{Agent.Proposal}}
-
h^{\mathrm{global}}
-
h^{\mathrm{local}}
\right]_+
\nonumber\\
&+
\epsilon_{\texttt{Agent.Eval}}
+
L_Q
L_{\mathrm{WM}}^{\mathrm{global}}
\bar{\epsilon}_{\mathrm{WM}}.
\label{eq:wap-candidate-value-error}
\end{align}

The first term characterizes the finite-search error caused by evaluating only $N$ local candidates, while the second term captures proposal error that remains outside the combined global and local search range.

\paragraph{World-Model Evaluation and \texttt{Agent.Rank}}
For each of the $N$ local candidates, WAP predicts its physical outcome using the world model and uses \texttt{Agent.Rank} to select among the predicted outcomes.

Let
\begin{align*}
a_{\mathrm{WM}}^*
\in
\argmax_{a\in\cA_N}
Q_{\mathrm{WM}}(a)
\end{align*}
denote the best candidate according to an ideal evaluation of the world-model-predicted rollouts.

Since \texttt{Agent.Rank} performs comparative ranking rather than necessarily predicting a calibrated scalar value, we characterize its error directly through ranking regret.

\begin{definition}[\texttt{Agent.Rank} error]
\label{def:wap-agent-rank}
Let $\hat{a}$ denote the candidate selected by \texttt{Agent.Rank}. Define its ranking regret among the locally imagined candidates as
\begin{align}
\epsilon_{\texttt{Agent.Rank}}
\coloneqq
Q_{\mathrm{WM}}(a_{\mathrm{WM}}^*)
-
Q_{\mathrm{WM}}(\hat{a})
~\geq0.
\label{eq:wap-agent-rank-error}
\end{align}
\end{definition}
Again a definition rather than a hypothesis. We phrase it as a regret rather than a calibration error because \texttt{Agent.Rank} only ever compares candidates against one another; it never emits a number one could be calibrated about.

We next bound the error introduced by world-model prediction and \texttt{Agent.Rank} during candidate selection.

\begin{lemma}[Candidate selection error]
\label{lem:wap-selection}
Under Definitions~\ref{def:wap-wm} and~\ref{def:wap-agent-rank} and Assumption~\ref{ass:wap-wm},
\begin{align}
Q(a_{\cA_N}^*)-Q(\hat{a})
\leq
2L_{\mathrm{WM}}^{\mathrm{local}}
\bar{\epsilon}_{\mathrm{WM}}
+
\epsilon_{\texttt{Agent.Rank}}.
\label{eq:wap-selection-error}
\end{align}
\end{lemma}

\begin{proof}
From Eq.~\ref{eq:wap-local-wm-error},
\begin{align*}
Q(a_{\cA_N}^*)
\leq
Q_{\mathrm{WM}}(a_{\cA_N}^*)
+
L_{\mathrm{WM}}^{\mathrm{local}}
\bar{\epsilon}_{\mathrm{WM}}.
\end{align*}

Since $a_{\mathrm{WM}}^*$ maximizes $Q_{\mathrm{WM}}$ over $\cA_N$,
\begin{align*}
Q_{\mathrm{WM}}(a_{\cA_N}^*)
\leq
Q_{\mathrm{WM}}(a_{\mathrm{WM}}^*).
\end{align*}

By Definition~\ref{def:wap-agent-rank},
\begin{align*}
Q_{\mathrm{WM}}(a_{\mathrm{WM}}^*)
\leq
Q_{\mathrm{WM}}(\hat{a})
+
\epsilon_{\texttt{Agent.Rank}}.
\end{align*}

Applying Eq.~\ref{eq:wap-local-wm-error} again gives
\begin{align*}
Q_{\mathrm{WM}}(\hat{a})
\leq
Q(\hat{a})
+
L_{\mathrm{WM}}^{\mathrm{local}}
\bar{\epsilon}_{\mathrm{WM}}.
\end{align*}

Combining the above inequalities gives
\begin{align*}
Q(a_{\cA_N}^*)-Q(\hat{a})
\leq
2L_{\mathrm{WM}}^{\mathrm{local}}
\bar{\epsilon}_{\mathrm{WM}}
+
\epsilon_{\texttt{Agent.Rank}},
\end{align*}
which proves the result.
\end{proof}

To simplify the final result, we define the total value-level error caused by world-model prediction as
\begin{align}
\epsilon_{\mathrm{WM}}
=
\left(
L_Q L_{\mathrm{WM}}^{\mathrm{global}}
+
2L_{\mathrm{WM}}^{\mathrm{local}}
\right)
\bar{\epsilon}_{\mathrm{WM}}.
\label{eq:wap-total-wm-error}
\end{align}

This quantity combines the two ways in which world-model prediction error affects WAP: perturbing the global correction produced by \texttt{Agent.Eval} and distorting the relative values of the locally imagined candidates used by \texttt{Agent.Rank}.

We are now ready to state the overall error bound.

\begin{theorem}[Error bound of WAP]
\label{thm:wap-error}
Under Assumptions~\ref{ass:wap-lip-q},
\ref{ass:wap-wm} and
\ref{ass:wap-local-cover}, with the error quantities of Definitions~\ref{def:wap-agent-eval},
\ref{def:wap-wm} and
\ref{def:wap-agent-rank},
the action $\hat{a}$ selected by WAP from $N$ local candidates satisfies the sharp bound
\begin{align}
Q(a^*)-Q(\hat{a})
\leq
L_Q\varrho_N
+
L_Q\left[\Delta_{\mathrm{global}}-h^{\mathrm{local}}\right]_+
+
2L_{\mathrm{WM}}^{\mathrm{local}}\bar\epsilon_{\mathrm{WM}}
+
\epsilon_{\texttt{Agent.Rank}}\,,
\label{eq:wap-sharp-bound}
\end{align}
and consequently, relaxing $\left[\Delta_{\mathrm{global}}-h^{\mathrm{local}}\right]_+$ and substituting the grid rate \eqref{eq:wap-local-cover-rate},
\begin{align}
Q(a^*)-Q(\hat{a})
\leq&
\underbrace{
L_Q
c_{d_a}
h^{\mathrm{local}}
N^{-1/d_a}
}_{\text{finite candidate search error}}
\nonumber\\
&+
\underbrace{
L_Q
\left[
\epsilon_{\texttt{Agent.Proposal}}
-
h^{\mathrm{global}}
-
h^{\mathrm{local}}
\right]_+
}_{\text{\texttt{Agent.Proposal} coverage error}}
\nonumber\\
&+
\underbrace{
\epsilon_{\texttt{Agent.Eval}}
}_{\text{\texttt{Agent.Eval} error}}
+
\underbrace{
\epsilon_{\mathrm{WM}}
}_{\text{world-model error}}
+
\underbrace{
\epsilon_{\texttt{Agent.Rank}}
}_{\text{\texttt{Agent.Rank} error}}.
\label{eq:wap-final-bound}
\end{align}
\end{theorem}

\begin{proof}
We decompose the action-selection error into the error caused by restricting the continuous action space to the $N$ generated candidates and the error caused by selecting among these candidates:
\begin{align*}
Q(a^*)-Q(\hat{a})
=
\left[
Q(a^*)-Q(a_{\cA_N}^*)
\right]
+
\left[
Q(a_{\cA_N}^*)-Q(\hat{a})
\right].
\end{align*}

For the sharp form \eqref{eq:wap-sharp-bound}, Assumption~\ref{ass:wap-lip-q} applied to the candidate nearest $a^*$ together with \eqref{eq:wap-candidate-cover-tight} gives $Q(a^*)-Q(a_{\cA_N}^*)\leq L_Q\varrho_N+L_Q\left[\Delta_{\mathrm{global}}-h^{\mathrm{local}}\right]_+$, and adding Lemma~\ref{lem:wap-selection} yields \eqref{eq:wap-sharp-bound}. The additive form follows by relaxing the same terms, which we now do explicitly.

From Eq.~\ref{eq:wap-candidate-value-error} and \eqref{eq:wap-local-cover-rate},
\begin{align*}
Q(a^*)-Q(a_{\cA_N}^*)
\leq&
L_Q
c_{d_a}
h^{\mathrm{local}}
N^{-1/d_a}
\\
&+
L_Q
\left[
\epsilon_{\texttt{Agent.Proposal}}
-
h^{\mathrm{global}}
-
h^{\mathrm{local}}
\right]_+
\\
&+
\epsilon_{\texttt{Agent.Eval}}
+
L_Q
L_{\mathrm{WM}}^{\mathrm{global}}
\bar{\epsilon}_{\mathrm{WM}}.
\end{align*}

From Lemma~\ref{lem:wap-selection},
\begin{align*}
Q(a_{\cA_N}^*)-Q(\hat{a})
\leq
2L_{\mathrm{WM}}^{\mathrm{local}}
\bar{\epsilon}_{\mathrm{WM}}
+
\epsilon_{\texttt{Agent.Rank}}.
\end{align*}

Adding the two inequalities gives
\begin{align*}
Q(a^*)-Q(\hat{a})
\leq&
L_Q
c_{d_a}
h^{\mathrm{local}}
N^{-1/d_a}
\\
&+
L_Q
\left[
\epsilon_{\texttt{Agent.Proposal}}
-
h^{\mathrm{global}}
-
h^{\mathrm{local}}
\right]_+
+
\epsilon_{\texttt{Agent.Eval}}
\\
&+
\left(
L_Q L_{\mathrm{WM}}^{\mathrm{global}}
+
2L_{\mathrm{WM}}^{\mathrm{local}}
\right)
\bar{\epsilon}_{\mathrm{WM}}
+
\epsilon_{\texttt{Agent.Rank}}.
\end{align*}

Using the definition of $\epsilon_{\mathrm{WM}}$ in Eq.~\ref{eq:wap-total-wm-error} proves Eq.~\ref{eq:wap-final-bound}.
\end{proof}

Eq.~\ref{eq:wap-final-bound} directly associates the error terms with the components of Algorithm~\ref{alg:mpc}. The initial error $\epsilon_{\texttt{Agent.Proposal}}$ does not need to be small enough for the initial proposal itself to be optimal. Instead, it only needs to lie within the combined global and local search reach. In particular, when
\begin{align}
\epsilon_{\texttt{Agent.Proposal}}
\leq
h^{\mathrm{global}}
+
h^{\mathrm{local}},
\label{eq:wap-sufficient-reach}
\end{align}
the proposal coverage error vanishes and the bound becomes
\begin{align}
Q(a^*)-Q(\hat{a})
\leq&
L_Q
c_{d_a}
h^{\mathrm{local}}
N^{-1/d_a}
+
\epsilon_{\texttt{Agent.Eval}}
+
\epsilon_{\mathrm{WM}}
+
\epsilon_{\texttt{Agent.Rank}}.
\label{eq:wap-covered-bound}
\end{align}

Therefore, once the global and local search ranges are sufficient to cover the proposal error, increasing the number of candidates reduces the finite-search approximation error at rate
\begin{align*}
N^{-1/d_a}.
\end{align*}

Equivalently, to make the finite candidate search error no larger than $\epsilon_{\mathrm{search}}$, it is sufficient to choose
\begin{align}
N
\geq
\left(
\frac{
L_Q c_{d_a} h^{\mathrm{local}}
}{
\epsilon_{\mathrm{search}}
}
\right)^{d_a}.
\label{eq:wap-candidate-complexity}
\end{align}

Eq.~\ref{eq:wap-candidate-complexity} directly relates planning accuracy to the number of world-model evaluations. Each local candidate requires one world-model rollout, so the local planning cost grows linearly with $N$. Meanwhile, the number of candidates required for a fixed approximation error grows exponentially with the local action-space dimension $d_a$. This motivates WAP to search over a low-dimensional goal space rather than directly sampling high-dimensional action sequences.

\paragraph{Coarse-to-Fine Search Tradeoff}
The two search radii $h^{\mathrm{global}}$ and $h^{\mathrm{local}}$ determine how WAP distributes search between coarse correction and fine local evaluation. Increasing $h^{\mathrm{global}}$ allows \texttt{Agent.Eval} to correct a larger \texttt{Agent.Proposal} error and therefore reduces the local search radius required for sufficient coverage. For fixed $N$, reducing $h^{\mathrm{local}}$ directly decreases the finite candidate search error
\begin{align*}
L_Q
c_{d_a}
h^{\mathrm{local}}
N^{-1/d_a}.
\end{align*}

By the sharp form \eqref{eq:wap-sharp-bound}, the smallest local radius that maintains coverage is on the scale of the entire global-stage residual,
\begin{align}
h^{\mathrm{local}}
\approx
\Delta_{\mathrm{global}}
=
\left[
\epsilon_{\texttt{Agent.Proposal}}
-
h^{\mathrm{global}}
\right]_+
+
\frac{\epsilon_{\texttt{Agent.Eval}}}{L_Q}
+
L_{\mathrm{WM}}^{\mathrm{global}}
\bar{\epsilon}_{\mathrm{WM}}\,,
\label{eq:wap-local-radius}
\end{align}
not merely of the uncorrected proposal error: local search also has to absorb the mistakes made by \texttt{Agent.Eval} and by the world model during the global stage.

Therefore, global evaluation is useful not only because \texttt{Agent.Eval} directly corrects the initial proposal, but also because it reduces the continuous region that must subsequently be covered by the $N$ candidates. Under a fixed number of world-model evaluations, concentrating the local candidate set within a smaller residual region produces finer effective action resolution. This provides a theoretical interpretation of WAP's coarse-to-fine structure: \texttt{Agent.Proposal} provides the initial action hypothesis, \texttt{Agent.Eval} performs coarse correction using the imagined outcome, and \texttt{Agent.Rank} performs fine-grained discrimination among the locally imagined candidates.

\paragraph{Tightness}
A decomposition is cheap if its terms are merely sufficient; one can always add more of them. The proposition below rules that out for \eqref{eq:wap-sharp-bound} by treating each term in turn: for each there is an instance, meeting every assumption, on which the other three vanish and the gap is of the order of the term in question.

\begin{proposition}[No term of \eqref{eq:wap-sharp-bound} is removable]
\label{prop:wap-tight}
Let the local action space be $\RR^{d_a}$ with the Euclidean metric and let $L_Q>0$. Then:
\begin{enumerate}[label=(\roman*),leftmargin=2em,itemsep=1pt,topsep=2pt]
\item \emph{(finite search)} For every $h^{\mathrm{local}}>0$ and $N\geq1$ there is an instance with $\Delta_{\mathrm{global}}=h^{\mathrm{local}}$, $\bar\epsilon_{\mathrm{WM}}=0$ and $\epsilon_{\texttt{Agent.Rank}}=0$ on which $Q(a^*)-Q(\hat a)=\Omega\bracket{L_Q\varrho_N}$.
\item \emph{(proposal coverage)} For every $\epsilon_{\texttt{Agent.Proposal}}>h^{\mathrm{global}}+h^{\mathrm{local}}$ there is an instance with $\epsilon_{\texttt{Agent.Eval}}=\bar\epsilon_{\mathrm{WM}}=\epsilon_{\texttt{Agent.Rank}}=0$ and $\varrho_N$ arbitrarily small on which $Q(a^*)-Q(\hat a)=L_Q\bracket{\epsilon_{\texttt{Agent.Proposal}}-h^{\mathrm{global}}-h^{\mathrm{local}}}$.
\item \emph{(global evaluation)} For every $\epsilon_{\texttt{Agent.Eval}}>L_Qh^{\mathrm{local}}$ there is an instance with $\epsilon_{\texttt{Agent.Proposal}}\leq h^{\mathrm{global}}$, $\bar\epsilon_{\mathrm{WM}}=\epsilon_{\texttt{Agent.Rank}}=0$ and $\varrho_N$ arbitrarily small on which $Q(a^*)-Q(\hat a)=\epsilon_{\texttt{Agent.Eval}}-L_Qh^{\mathrm{local}}$.
\item \emph{(local world-model error and ranking)} For every $\bar\epsilon_{\mathrm{WM}}\geq0$ and $\epsilon_{\texttt{Agent.Rank}}\geq0$ there is an instance with $a^*\in\cA_N$, hence $\varrho_N=0$ and $\Delta_{\mathrm{global}}=0$, on which $Q(a^*)-Q(\hat a)=2L_{\mathrm{WM}}^{\mathrm{local}}\bar\epsilon_{\mathrm{WM}}+\epsilon_{\texttt{Agent.Rank}}$.
\end{enumerate}
\end{proposition}

\begin{proof}
In all four cases take $Q(a)=-L_Q\norm{a-a^*}$, which satisfies Assumption~\ref{ass:wap-lip-q} with equality.

\emph{(i)} Set $h^{\mathrm{global}}=0$, $\epsilon_{\texttt{Agent.Eval}}=0$ and $\bar\epsilon_{\mathrm{WM}}=0$, so $\Delta_{\mathrm{global}}=\epsilon_{\texttt{Agent.Proposal}}$, and choose $\epsilon_{\texttt{Agent.Proposal}}=h^{\mathrm{local}}$; then $a_g=a_0$ lies at distance exactly $h^{\mathrm{local}}$ from $a^*$ and $[\Delta_{\mathrm{global}}-h^{\mathrm{local}}]_+=0$. Let $\cA_N$ be a grid of resolution $\varrho_N$ in $B(a_g,h^{\mathrm{local}})$ translated so that $a^*$ falls at the centre of a grid cell, so $\min_{a\in\cA_N}\norm{a-a^*}\geq\varrho_N/2$. With $\bar\epsilon_{\mathrm{WM}}=0$ and $\epsilon_{\texttt{Agent.Rank}}=0$ we have $\hat a=a_{\cA_N}^*$, and $Q(a^*)-Q(\hat a)\geq L_Q\varrho_N/2$.

\emph{(ii)} Set $\epsilon_{\texttt{Agent.Eval}}=\bar\epsilon_{\mathrm{WM}}=0$, so $a_g=a_{\mathrm{global}}^*$ lies on the segment $[a_0,a^*]$ at distance $\epsilon_{\texttt{Agent.Proposal}}-h^{\mathrm{global}}>h^{\mathrm{local}}$ from $a^*$. Every point of $B(a_g,h^{\mathrm{local}})$ is then at distance at least $\epsilon_{\texttt{Agent.Proposal}}-h^{\mathrm{global}}-h^{\mathrm{local}}$ from $a^*$, and taking $\cA_N$ to be a net of that ball with $\varrho_N\downarrow0$ and a perfect ranker gives the claim.

\emph{(iii)} Set $\epsilon_{\texttt{Agent.Proposal}}\leq h^{\mathrm{global}}$, so $a_{\mathrm{global}}^*=a^*$, and let \texttt{Agent.Eval} return $a_g^{\mathrm{true}}=a_g$ at distance $\epsilon_{\texttt{Agent.Eval}}/L_Q>h^{\mathrm{local}}$ from $a^*$, which is consistent with Definition~\ref{def:wap-agent-eval}. The remainder is as in (ii).

\emph{(iv)} Let $\cA_N=\sets{a^*,a'}$ with $\norm{a'-a^*}=\bracket{2L_{\mathrm{WM}}^{\mathrm{local}}\bar\epsilon_{\mathrm{WM}}+\epsilon_{\texttt{Agent.Rank}}}/L_Q$, so $Q(a^*)-Q(a')=2L_{\mathrm{WM}}^{\mathrm{local}}\bar\epsilon_{\mathrm{WM}}+\epsilon_{\texttt{Agent.Rank}}$. Let the world model under-value $a^*$ and over-value $a'$ by $L_{\mathrm{WM}}^{\mathrm{local}}\bar\epsilon_{\mathrm{WM}}$ each, which saturates \eqref{eq:wap-local-wm-stability}; then $Q_{\mathrm{WM}}(a')=Q_{\mathrm{WM}}(a^*)+\epsilon_{\texttt{Agent.Rank}}$, so $a_{\mathrm{WM}}^*=a'$. Letting \texttt{Agent.Rank} return $\hat a=a'$ realizes ranking regret $0$ when $\epsilon_{\texttt{Agent.Rank}}=0$; for $\epsilon_{\texttt{Agent.Rank}}>0$, extend $\cA_N$ by a third candidate at the corresponding value so that the ranker's regret is exactly $\epsilon_{\texttt{Agent.Rank}}$. In either case $Q(a^*)-Q(\hat a)=2L_{\mathrm{WM}}^{\mathrm{local}}\bar\epsilon_{\mathrm{WM}}+\epsilon_{\texttt{Agent.Rank}}$.
\end{proof}

Case \emph{(iii)} also marks the limit of what the bound claims. The \texttt{Agent.Eval} term survives in the bound only once $\epsilon_{\texttt{Agent.Eval}}>L_Qh^{\mathrm{local}}$. Below that threshold the sharp form \eqref{eq:wap-sharp-bound} folds it into the local radius while the additive form \eqref{eq:wap-final-bound} keeps charging for it. Quote the additive form for intuition, then, and the sharp one when tightness is at stake.

\subsection{Empirical Instantiation of the WAP Bound}
\label{sec:wap-instantiation}

A bound is only informative to the extent that its symbols name things one can observe, so we walk through \eqref{eq:wap-sharp-bound} term by term and attach each to the configuration or the measurement it corresponds to in the rest of the paper. We make no claim to identify the constants $L_Q$, $L_{\mathrm{WM}}^{\mathrm{global}}$ and $L_{\mathrm{WM}}^{\mathrm{local}}$ individually. What can be observed, and what we report, are the composite value-level quantities they appear inside.

\paragraph{Search Dimension $d_a$ and Radii $h^{\mathrm{global}},h^{\mathrm{local}}$} WAP searches over the end-effector goal pose rather than the low-level action sequence, which puts $d_a=2$ for the planar position search and $d_a=3$ once yaw is included (Appendix~\ref{app:search details}). Since $d_a$ sits in the exponent of \eqref{eq:wap-candidate-complexity}, this is the concrete reason a candidate budget that would be hopeless for sampling an action chunk is ample here. Neither radius is tuned; both come from the geometry of the object, with $h^{\mathrm{local}}$ at roughly half the object size and $h^{\mathrm{global}}$ at roughly its height. Appendix~\ref{app:robustness} finds performance insensitive to the exact values, which is what \eqref{eq:wap-sharp-bound} would lead one to expect in the covered regime $\Delta_{\mathrm{global}}\leq h^{\mathrm{local}}$: there the bound drops $h^{\mathrm{global}}$ altogether and touches $h^{\mathrm{local}}$ only through $\varrho_N$.

\paragraph{Candidate Budget $N$ and Resolution $\varrho_N$} WAP evaluates $N=6$ candidates per planning step against $8$ and $16$ for the sampling baseline (Table~\ref{tab:mpc & bon}). Comparing the raw budgets misses the point; what matters is $N$ at a given $\varrho_N$. A baseline sampling the same neighborhood isotropically needs $\Theta((h/\varrho)^{d_a})$ samples to reach resolution $\varrho$, while WAP has already shrunk $h$ in the global stage before the local stage spends anything. Table~\ref{tab:mpc & bon} bears this out: at $N=6$ with world-model rollouts WAP reaches $72.0/82.0/78.0\%$, ahead of VLMPC handed $8$ or $16$ \emph{ground-truth simulator} rollouts. The $N=1$ row, global optimization only, at $58.0/66.0/65.0\%$, is the configuration with the local stage switched off, and so isolates what the $L_Q\varrho_N$ term is worth.

\paragraph{Proposal Coverage} We cannot measure $L_Q[\epsilon_{\texttt{Agent.Proposal}}-h^{\mathrm{global}}-h^{\mathrm{local}}]_+$ directly, since $\epsilon_{\texttt{Agent.Proposal}}$ is a distance to an optimum nobody knows. Its qualitative signature is testable, however, and the zero-shot comparison against the vision-language planner baseline in Section~\ref{sec:exp} is where it shows: a planner that emits $a_0$ and executes it pays the full $L_Q\epsilon_{\texttt{Agent.Proposal}}$, while WAP pays only for whatever lies outside the combined reach of its two stages.

\paragraph{World-Model and Ranking Error} The terms $2L_{\mathrm{WM}}^{\mathrm{local}}\bar\epsilon_{\mathrm{WM}}$ and $\epsilon_{\texttt{Agent.Rank}}$ are the ones $N$ never enters, so the bound stops improving in $N$ once they dominate. The ranking-agreement study of Appendix~\ref{app:world model qualitative results} is the nearest evidence we have about them, and it is worth saying precisely what it does and does not show. The protocol asks one agent to pick the best of a candidate group twice, once from imagined and once from ground-truth rollouts, and records how often the two picks coincide: $82/90/78\%$ for our world model against $52/64/44\%$ for the cross-attention baseline (Table~\ref{tab:app wm vlm align}).

That quantity is the stability of one agent's choice when imagination is substituted for reality. It is not the regret of that choice, and it does not measure either error term. The reference pick is the agent's own selection on true rollouts rather than $a^*$, so it carries $\epsilon_{\texttt{Agent.Rank}}$ itself; the two passes can agree on a candidate that is wrong in both, which is compatible with $\epsilon_{\texttt{Agent.Rank}}$ being large. Agreement is also a $0/1$ event while the terms of \eqref{eq:wap-sharp-bound} are in value units, so a flip between two near-equal candidates and a flip between a good candidate and a useless one enter the statistic identically, and no bound on $Q(a^*)-Q(\hat a)$ follows from the rate. Nor does the protocol touch $Q_{\mathrm{WM}}$, the quantity $\epsilon_{\texttt{Agent.Rank}}$ is defined against. What the study does support is a necessary condition: a world model that frequently reorders the agent's candidates cannot support planning on imagination in place of reality, and ours reorders them less often than the baseline does.

The direction of \eqref{eq:wap-sharp-bound} deserves the same care. It is an upper bound, so a term in it that does not shrink with $N$ shows that the bound stops improving, not that the error of the deployed system stops improving; a loose bound would behave the same way. Proposition~\ref{prop:wap-tight} rules out dropping the term from a general bound over the assumption class, which is a worst-case statement and not evidence about our tasks. That the baseline in Table~\ref{tab:mpc & bon} does not close its gap as $N$ goes from $8$ to $16$ is consistent with the picture, and we claim no more than consistency.

\paragraph{What Remains Uninstantiated} Two quantities escape every experiment we run. The first is $\epsilon_{\texttt{Agent.Eval}}$, which stays confounded with $\epsilon_{\texttt{Agent.Proposal}}$ throughout: ablating the global stage removes the correction and its error together, and prying them apart would take an oracle global corrector we do not have. The second is the set of constants $L_Q$, $L_{\mathrm{WM}}^{\mathrm{global}}$, $L_{\mathrm{WM}}^{\mathrm{local}}$ and $c_{d_a}$, which we never measure. Read \eqref{eq:wap-sharp-bound}, then, as an account of which quantities matter and how they combine, and not as a numerical prediction of success rates.

\newpage
\section{Implementation Details}\label{app:detail method}
\subsection{World Model}\label{app:detail world model}
Here we provide the architecture details of our pose-image conditioned world model. Following \citet{song2025history}, we add independent noise levels to the history frames with probability 0.5 and retain clear history frames with probability 0.5. For the future frames, we add uniform random noise and denoise them uniformly, without any auto-regressive schedule to speed up inference. We adopt multi-view prediction by concatenating 4 camera views into a grid, so that the model can infer the 3D positions of the robot. Specifically, for table top manipulators, we use wrist-view, front-view, side-view and top-view. For bi-manual robots, we use two wrist-views and the front-view and the side-view. For Robocasa, we utilize the 2 side-views, together with a wrist-view and top-view. Since the robot is mobile in Robocasa, we fix the top view camera, so we can infer the 2D movements from the top-view pose. We share the position embeddings of the frame tokens and pose image tokens, and only compute the flow matching loss of the frame tokens.
\subsubsection{Pose-Image Conditioning: Decoupling Physical Simulation and Neural Imagination}\label{app:pose image}

While modern video foundation models generate high-fidelity and physically plausible videos \citep{wiedemer2025video}, they often lack the dynamics necessary to predict the specific physical consequences of low-dimensional robot actions. These outcomes are inherently dependent on the robot's unique kinematics and control parameters. To address this, we map abstract action vectors into visual joint pose skeletons, providing a semantically rich action representation. We compute anticipated joint positions by executing forward dynamics of the robot model. In practice, physical simulators such as MuJoCo, PyBullet, or Pinocchio can be employed to achieve these computations efficiently, given standard URDF or MJCF configurations. The simulator only have access to the robot model without any access to scene geometry and object states. Crucially, these pose images are not intended to be exact physical predictions. Instead, they serve as a mapping of actions, where the world model learns accurate real world physical dynamics during training on trajectories. Notably, we do not access any future ground truth robot joint states in both training and inference, relying on computation from the initial states only.
% \begin{figure}[h]
%     \centering
%     \includegraphics[width=\linewidth]{appendix/figures/fig1.png}
%     \caption{\textbf{Illustration of inaccurate pose image.} In the image, the first row is the actual environment rollout and the second row is the computed pose image using forward dynamics. Because the gripper action is \texttt{close} for every step, in the computed robot joint pose, the grippers jaws are closed immediately. In actual execution, due to the presence of the cube between the jaws, the opening between the jaws remain unchanged.}
%     \label{fig:app pose demo}
% \end{figure}

% \newpage

\subsection{World Action Planner}\label{app:planner detail}
Here we provide additional details about our planning framework, including the robot controller, optional policy, and search process.

\subsubsection{Robot Controller}
We use the robot controller as an inverse-dynamics module that navigates the robot to a desired goal pose. In our simulation experiments, we implement the controller as a diffusion policy with a UNet backbone that takes the current and target end-effector poses as input and generates executable action chunks. We train the controller on completely random trajectories without any task-specific information. For the real-robot experiment with a Franka Panda robot, we interpolate the goal position into waypoints spaced 10\,mm apart and directly use the robot's low-level controller to navigate through these waypoints without any neural networks.

\subsubsection{Optional Policy}
For compositional task generalization and new-layout generalization, we additionally use an optional policy to perform local manipulation, such as grasping. We implement this policy as a diffusion transformer, with details provided in Appendix~\ref{app: libero long details}. We emphasize that this lightweight diffusion policy is trained using the same amount of expert data as the baseline generalist policies, such as $\pi_{0.5}$ and Cosmos-Policy, for compositional task generalization, and substantially less expert data for new-layout generalization. When using the policy, we further imagine its rollout after each candidate action to verify that the resulting state allows the policy to successfully complete the subsequent manipulation.

\subsubsection{Search Details}\label{app:search details}
To efficiently search for plausible actions using our world model, we adopt a global-local search strategy inspired by the framework of \citet{zhang2025inference}. This approach follows a coarse-to-fine hierarchy: we use a larger search scale for global action optimization guided by agent feedback, followed by a finer grid scale for local search. The prompt templates for agent action proposal, feedback-driven optimization, and candidate ranking are provided in Figs.~\ref{fig:prompt proposal},~\ref{fig:optimization prompt}, and~\ref{fig:ranking prompt}, respectively.

During global action optimization, we prompt the agent to propose directional modifications rather than precise coordinates. For \texttt{GridSearch}, we search the planar coordinates of the target gripper position over the set $\sets{(x,y),(x\pm\delta,y),(x,y\pm\delta)}$. We set the search scale $\delta$ for local grid search and global optimization to approximately half the object size and object height in magnitude, respectively, thereby encouraging physically plausible candidate generation. As shown in Appendix~\ref{app:robustness}, our method is insensitive to the exact values of these hyperparameters. In practice, the physical dimensions of objects can either be obtained directly or estimated from camera observations using computer vision models. We emphasize that these hyperparameter choices are physically grounded and substantially more interpretable than the sampling variances used in MPPI-based methods. When modifying the orientation, we search the target yaw angle over the set $\sets{\psi,\psi+45^{\circ}, \psi-45^{\circ}}$. Importantly, our search operates in the low-dimensional goal-pose space rather than over individual low-level actions, resulting in a substantially more compact search space.

% \newpage
\section{Experiment Details}\label{app:exp detail}
In this section, we provide detailed setup and more qualitative results for our experiments and ablations.

\subsection{World Model}\label{app:wm exp details}
Here we provide the details of the action-conditioned world model experiments in Section~\ref{sec:video gen exp}.

\subsubsection{Setup}\label{app:wm exp setup} 
To simulate the action sampling typical of the widely used Model Predictive Path Integral (MPPI) algorithm and to improve model generalization across novel exploratory actions, we augment expert demonstrations with Gaussian noise prior to trajectory replay. Specifically, we generate 100 trajectories per task: 40 from clean demonstrations, 40 with Gaussian noise added to the actions ($\sigma=0.08$), and 20 with higher-intensity noise ($\sigma=0.16$). We observe that model-based exploration primarily necessitates spatial reasoning regarding the end-effector's location; thus, for the high-variance noise ($\sigma=0.16$), we limit perturbations to the first three dimensions of the action vector (corresponding to target gripper positions) while maintaining fixed orientations. Incorporating such varied trajectories, including potential failures, is a standard practice in the literature \citep{guo2025ctrl,jain2025smooth} for improving world model robustness. The model is trained with a global batch size of 64 on 16 H100 GPUs. For the cross-embodiment baselines, we adopt cross-attention as our base action-conditioning mechanism due to its superior performance (Table~\ref{tab:single-robot}), and then apply the cross-embodiment-unification methods. Since our conditioning mechanism does not require any text input, we use a dummy text embedding for the \texttt{Wan}-T2V model for all inputs.

\begin{table}[h]
    \centering
    \caption{\textbf{World model design ablations.} Each cell reports PSNR$\uparrow$\,/\,LPIPS$\downarrow$ on LIBERO-Long, which evaluates generalization to novel action sequences. Both ablations share the same model and training data, and differ only in the component under study; the second row of each subtable is our final configuration.}
    \label{tab:app wm ablations}
    \footnotesize
    \setlength{\tabcolsep}{3pt}
    \begin{subtable}[t]{0.48\linewidth}
        \centering
        \caption{Sparse history frames.}
        \label{tab:app sparse history}
        \begin{tabular*}{\linewidth}{@{\extracolsep{\fill}}lcc@{}}
            \toprule
            History FPS & wrist-view & third-view \\
            \midrule
            20 FPS (dense) & 15.68 / 0.326 & 22.04 / 0.094 \\
            7 FPS (sparse) & \textbf{15.98} / \textbf{0.320} & \textbf{22.14} / \textbf{0.093} \\
            \bottomrule
        \end{tabular*}
    \end{subtable}
    \hfill
    \begin{subtable}[t]{0.48\linewidth}
        \centering
        \caption{Action conditioning representation.}
        \label{tab:app joint pos}
        \begin{tabular*}{\linewidth}{@{\extracolsep{\fill}}lcc@{}}
            \toprule
            Conditioning & wrist-view & third-view \\
            \midrule
            Joint position & 14.62 / 0.346 & 18.25 / 0.138 \\
            Pose image & \textbf{15.98} / \textbf{0.320} & \textbf{22.14} / \textbf{0.093} \\
            \bottomrule
        \end{tabular*}
    \end{subtable}
\end{table}
\paragraph{World Model Design} We ablate two design choices of our world model in Table~\ref{tab:app wm ablations}, both evaluated on LIBERO-Long. First, we use a sampling rate of 7 FPS for history frames, in contrast to the 20 FPS used for prediction frames and environment control. This sparse history allows for a longer temporal window without increasing the total number of history frames. Additionally, utilizing sparse frames helps mitigate spurious correlations rooted in local dynamics between continuous frames, forcing the model to learn the global environment configuration rather than relying on short-term motion patterns. As shown in Table~\ref{tab:app sparse history}, this approach improves performance when generalizing to novel action sequences. Second, we compare our pose-image conditioning against cross-attention conditioning on joint positions computed through the same forward dynamics. As shown in Table~\ref{tab:app joint pos}, given the same joint positions, rendering them as pose images generalizes substantially better, indicating that the benefit comes from the spatial rendering rather than from the forward-dynamics information itself.

\subsubsection{Additional Results}\label{app:world model qualitative results}
\paragraph{Ranking Agreement} Beyond pixel-level fidelity, what a planner needs from imagined rollouts is that they preserve the \emph{ordering} the VLM agent would assign to candidate actions had it observed the real outcomes. For each generalization setting, we collect 50 groups of candidate actions from our planner experiments and ask the same VLM agent to select the best candidate twice: once from world-model imagination, and once from ground-truth simulator rollouts. As shown in Table~\ref{tab:app wm vlm align}, our model preserves the VLM's choice far more often than the cross-attention baseline, indicating that its advantage in pixel-space metrics translates into the ranking fidelity that planning depends on. We read this as a necessary condition rather than as a measure of action quality. The reference is the agent's own choice on true rollouts, not the optimal action, so the two passes may agree on a candidate that is wrong in both; and agreement is recorded as a binary event, so a flip between two near-equal candidates counts the same as a flip between a good candidate and a useless one. The comparison shows that our world model reorders candidates less often, not by how much the reorderings cost.
\begin{table}[htb]
    \centering
    \caption{\textbf{Agreement between imagined and ground-truth candidate rankings.} For each setting we report the percentage of 50 candidate groups on which the VLM agent selects the same best action from world-model imagination as from ground-truth simulator rollouts. Higher is better.}
    \label{tab:app wm vlm align}
    \small
    \begin{tabular*}{0.8\linewidth}{@{\extracolsep{\fill}}lccc@{}}
        \toprule
        World model & Compositional task & New layout & Novel task \\
        \midrule
        Ctrl-World \citep{guo2025ctrl} & 52 & 64 & 44 \\
        \textbf{Pose-image conditioned (ours)} & \textbf{82} & \textbf{90} & \textbf{78} \\
        \bottomrule
    \end{tabular*}
\end{table}

\paragraph{Qualitative Results} We first evaluate the generalization capabilities of our pose-image conditioned world model by demonstrating its ability to predict the physical consequences of novel actions that lead to failures not present in the training set. As illustrated in Fig.~\ref{fig:app world model demo}, the world model successfully imagines the physical outcomes of actions that knock a mug over in various directions, such as backward or to the right. Notably, while Gaussian noise was incorporated during training, the noise levels were minimal and insufficient to produce such collisions; therefore, these ``knocking over'' trajectories are strictly out-of-distribution, highlighting our model's capacity for zero-shot physical simulation and generalization.
\begin{figure}[htb]
\begin{subfigure}{\linewidth}
    \centering
    \includegraphics[width=\linewidth]{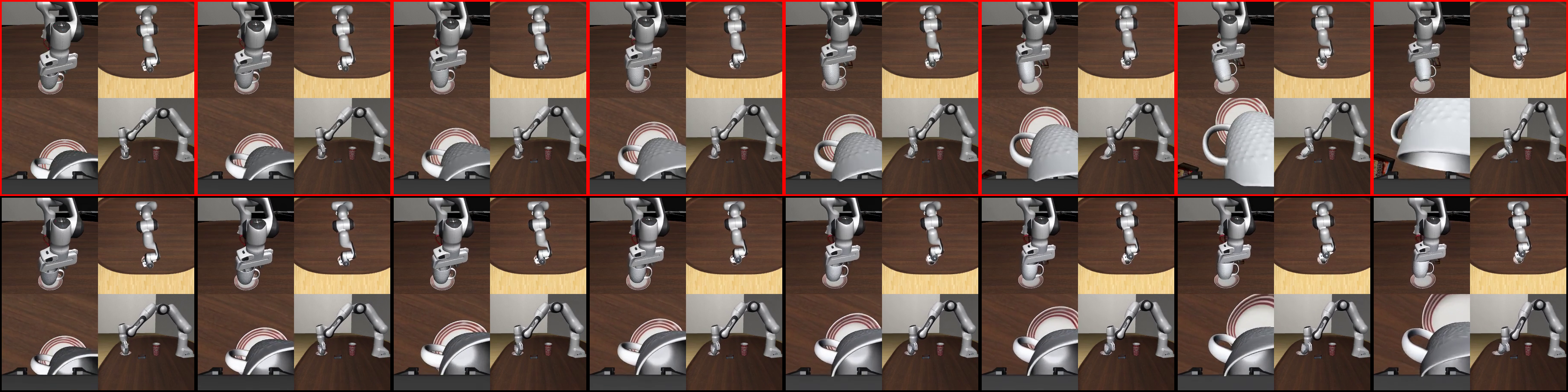}
    \caption{Qualitative results for knocking the mug over backwards. }
\end{subfigure}
    \begin{subfigure}{\linewidth}
        \centering
        \includegraphics[width=\linewidth]{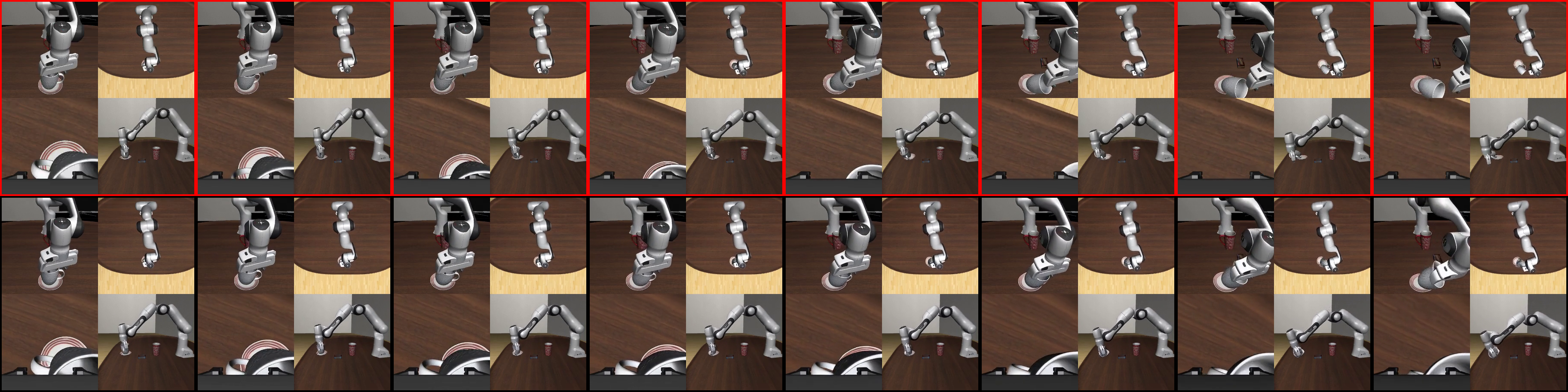}
        \caption{Qualitative results for knocking the mug over to the right.}
    \end{subfigure}
    \caption{Qualitative results for OOD actions that knocks the mug over, with the first row being ground truth frames with red borders, and the second row being generated imagination frames. The videos are downsampled 5x from 40 frames (2 seconds) to 8 frames.}
    \label{fig:app world model demo}
\end{figure}
We further demonstrate the cross-embodiment modeling capabilities of our framework in Fig.~\ref{fig:app world model cross robot}. Our pose-image conditioned model successfully captures the dynamics across diverse robotic platforms, effectively generalizing from high-degree-of-freedom humanoid hands to mobile manipulators with side view by leveraging the embodiment-agnostic nature of the pose-image representation.

\begin{figure}[htb]
\begin{subfigure}{\linewidth}
    \centering
    \includegraphics[width=\linewidth]{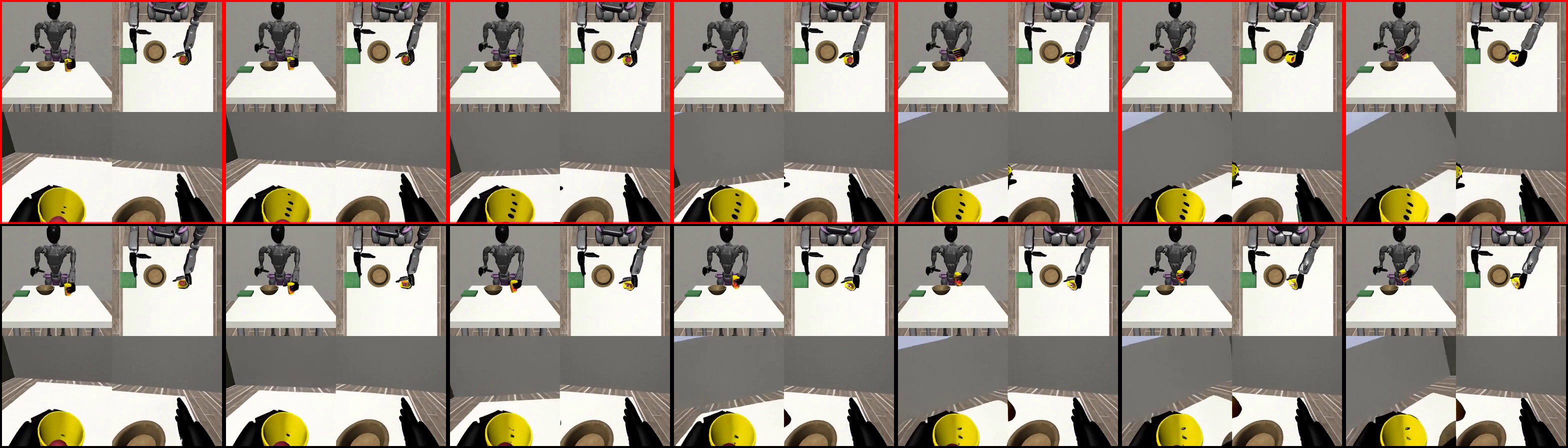}
    \caption{Qualitative results for dexterous humanoid in DexMimicGen dataset. }
\end{subfigure}
    \begin{subfigure}{\linewidth}
        \centering
        \includegraphics[width=\linewidth]{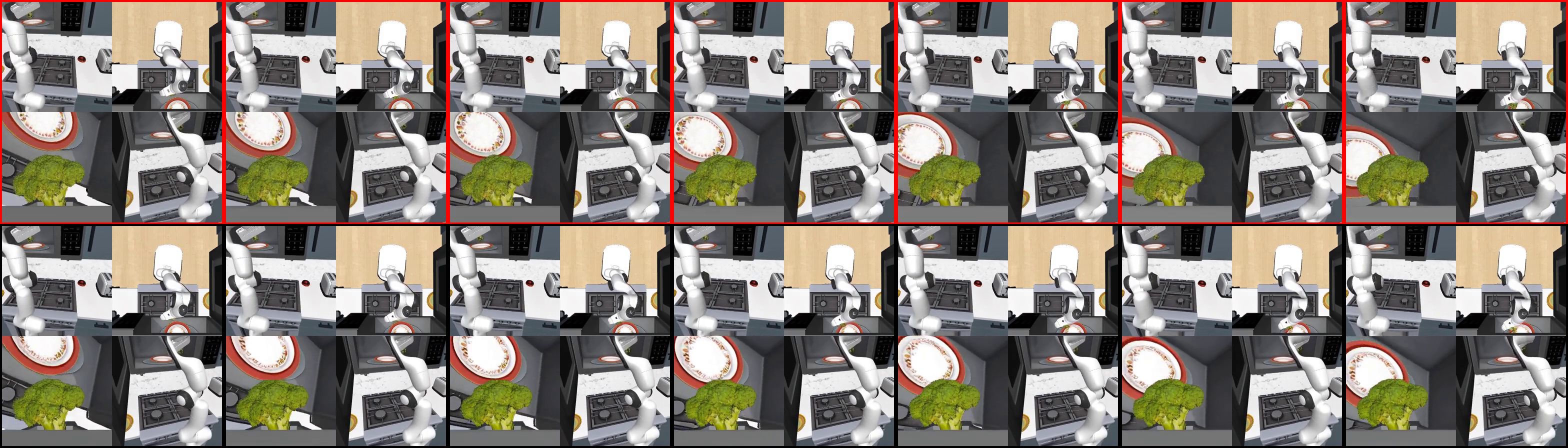}
        \caption{Qualitative results for the Panda Omron robot in the Robocasa dataset.}
    \end{subfigure}
    \caption{Qualitative results cross-embodiment modeling, with the first row being ground truth frames with red borders, and the second row being generated imagination frames. The videos are downsampled 3x from 20 frames (1 second) to 7 frames.}
    \label{fig:app world model cross robot}
\end{figure}

\newpage
\subsection{World Action Planner}\label{app:planner exp details}
Here we provide the details of the world action planner experiments in Section~\ref{sec:planning exp}.

\paragraph{Baseline Setup}
For the generalist policy baselines $\pi_{0.5}$ and Cosmos-Policy, we follow their original recommended training setups and ensure both training convergence and strong in-distribution performance. Additional details are provided in the corresponding experimental sections below.

For the policy enhancement baselines, we adapt each method as closely as possible to its original setup while using the same world model and VLM-based evaluation where applicable. For SAILOR \citep{jain2025smooth}, because we do not have task-specific trained reward models, we use the same VLM as the reward model. For GPC-RANK \citep{qi2025strengthening}, we follow the ranking setup from the original work and use the same VLM reward function. For DREAMSTEER \citep{cui2026dreamsteer}, we sample action proposals from the $\pi_{0.5}$ VLA, incorporate the additional action primitives used in the original method, and use the same VLAC critic model as in the original implementation.

For the vision-language planner baseline, we use the same agent, action proposal prompts and triangulation procedure, optional diffusion policy, and robot controller as in WAP. The proposed actions are executed directly, without world-model imagination, global action optimization, or local action search. {To ensure a fair comparison, we explicitly include common failure modes observed during VLM exploration in the prompt, providing the VLM planner with additional guidance to avoid them. Therefore, the performance gains of our method over the VLM planner baseline arise from the world-model imagination and action-search process, isolating the benefit of explicit visual world modeling and action optimization.} {Notably, even when the VLM agent is explicitly instructed to avoid collisions and maintain sufficient clearance, such failures still occur and are only mitigated after reasoning over imagined world-model rollouts and refining the proposed actions, further demonstrating the effectiveness of our approach.}

\paragraph{Resource Comparison}
Here we summarize the data resources used by finetuning each component in the experiments below. All expert demonstrations are drawn from the same dataset, though the quantities differ across methods. The expert demonstrations used by the world model and for perturbation are the same as the one received by the policy in WAP, thus the unique number of expert demonstrations are 30 for the compositional task setting and 5 for the new layout setting. 

Among the policy enhancement baselines, SAILOR \citep{jain2025smooth} and GPC-RANK \citep{qi2025strengthening} use a diffusion policy finetuned on the same data as the generalist policy baselines, and DREAMSTEER \citep{cui2026dreamsteer} uses the $\pi_{0.5}$ VLA from those baselines.
\begin{table}[htb]
    \centering
    \caption{\textbf{Data resources consumed by each component.} A dash indicates that the component is not used in that setting. Perturbed trajectories are demonstration replays with action noise, as described in Appendix~\ref{app:wm exp setup}.}
    \label{tab:data comparison}
    \small
    \setlength{\tabcolsep}{4pt}
    \renewcommand{\arraystretch}{1.15}
    \begin{tabular*}{\linewidth}{@{\extracolsep{\fill}}>{\raggedright\arraybackslash}p{0.24\linewidth}>{\raggedright\arraybackslash}p{0.21\linewidth}>{\raggedright\arraybackslash}p{0.21\linewidth}>{\raggedright\arraybackslash}p{0.24\linewidth}@{}}
        \toprule
        Component & Compositional task & New layout & Novel task \\
        \midrule
        Policy in WAP & 30 sub-task demonstrations & 5 unmodified-layout demonstrations & -- \\
        World model in WAP & The 30 demos above plus 70 perturbed & The 5 above plus 10 perturbed & 50 exploratory trajectories in simulation; 10 failure trajectories on the real robot \\
        Generalist policy baselines ($\pi_{0.5}$, cosmos-policy) & 40--50 sub-task demonstrations & 40--50 unmodified-layout demonstrations & 0 target trajectories (zero-shot); 50 exploratory trajectories (same-data) \\
        \bottomrule
    \end{tabular*}
\end{table}

\newpage
\subsubsection{Compositional Task Generalization}\label{app: libero long details}
\begin{figure}[h]
    \centering
    \includegraphics[width=\linewidth]{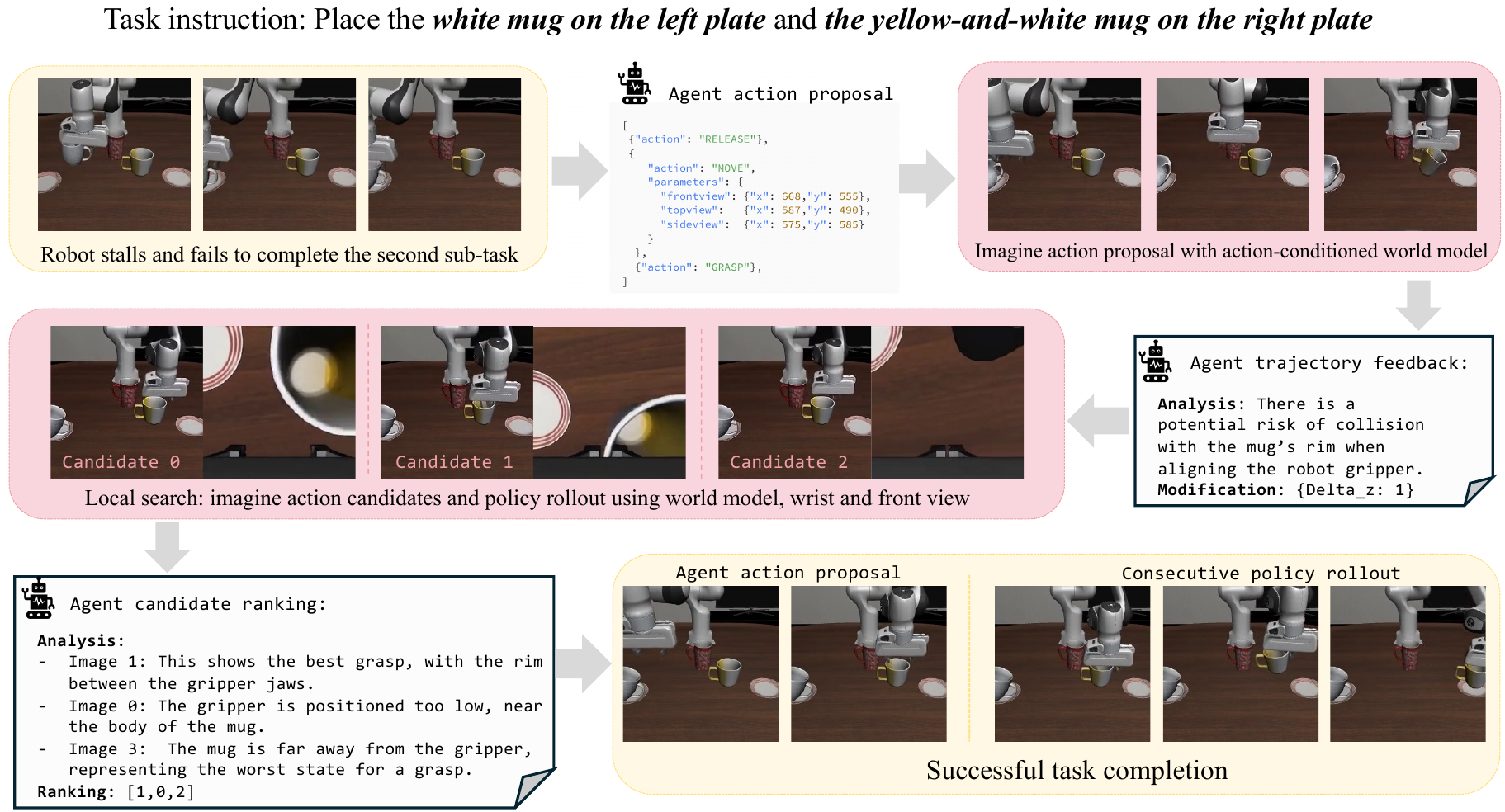}
    \caption{\textbf{Illustration of world action planner in compositional task generalization}}
    \label{fig:libero10}
\end{figure}
\paragraph{Environment Setup}
We select the five tasks from LIBERO-Long, each composed of two tasks from the LIBERO-90 suite, involving pick and place operations as well as articulated object manipulation. Table~\ref{tab:app libero task composition} lists the selected tasks and their respective subtasks. We select all the tasks in LIBERO-Long that has direct corresponding subtasks in LIBERO-90 with the same scene, excluding ones that do not have direct counterparts.
\begin{table}[htb]
    \centering
    \caption{\textbf{Composition of LIBERO-Long tasks from LIBERO-90 subtasks.} Both subtasks in each row share the scene specified in the first column. Instructions are shortened for readability; task membership follows the official LIBERO task map \citep{liu2023libero}.}
    \label{tab:app libero task composition}
    \small
    \setlength{\tabcolsep}{4pt}
    \renewcommand{\arraystretch}{1.15}
    % Source: https://github.com/Lifelong-Robot-Learning/LIBERO/blob/master/libero/libero/benchmark/libero_suite_task_map.py
    \begin{tabular*}{\linewidth}{@{\extracolsep{\fill}}>{\raggedright\arraybackslash}p{0.36\linewidth}>{\raggedright\arraybackslash}p{0.29\linewidth}>{\raggedright\arraybackslash}p{0.29\linewidth}@{}}
        \toprule
        \multicolumn{1}{l}{LIBERO-Long composed task} & \multicolumn{1}{l}{LIBERO-90 subtask 1} & \multicolumn{1}{l}{LIBERO-90 subtask 2} \\
        \midrule
        \textbf{Living room scene 2}\newline Put alphabet soup and tomato sauce in the basket.
        & Put alphabet soup in the basket.
        & Put tomato sauce in the basket. \\
        \addlinespace
        \textbf{Living room scene 5}\newline Put the white mug on the left plate and the yellow and white mug on the right plate.
        & Put the white mug on the left plate.
        & Put the yellow and white mug on the right plate. \\
        \addlinespace
        \textbf{Living room scene 6}\newline Put the white mug on the plate and chocolate pudding to the right of the plate.
        & Put the white mug on the plate.
        & Put chocolate pudding to the right of the plate. \\
        \addlinespace
        \textbf{Living room scene 1}\newline Put alphabet soup and the cream cheese box in the basket.
        & Put alphabet soup in the basket.
        & Put the cream cheese box in the basket. \\
        \addlinespace
        \textbf{Kitchen scene 4}\newline Put the black bowl in the bottom drawer of the cabinet and close it.
        & Put the black bowl in the bottom drawer of the cabinet.
        & Close the bottom drawer of the cabinet. \\
        \bottomrule
    \end{tabular*}
\end{table}

\paragraph{World Action Planner Setup} We train our diffusion policy, the $\pi_{0.5}$ baseline, and the cosmos-policy using an action chunk length of 40 (equivalent to a 2-second horizon) following \citet{kim2026cosmos}, maintaining a 10-step execution horizon for each chunk before re-planning the subsequent sequence. Following \citet{chi2025diffusion}, we implement our diffusion policy using a transformer architecture, with the specific hyper-parameters detailed in Table~\ref{tab:app policy detail}.
\begin{table}[htb]
    \centering
    \caption{\textbf{Diffusion policy architecture and training hyper-parameters.}}
    \label{tab:app policy detail}
    \small
    \renewcommand{\arraystretch}{1.05}
    \begin{tabular*}{0.6\linewidth}{@{\extracolsep{\fill}}l r@{}}
    \toprule
       Hyper-parameter  &  Value\\
       \midrule
       Layers  &  12 \\
       Heads & 8 \\
       Dimension & 512 \\
       Observation length & 1 \\
       Prediction length & 40 \\
       Execution action steps & 10 \\
       Input images & wrist view and agent view \\
       Image resolution & $128\times128$ \\
       Crop size & $116\times116$ \\
       Batch size & 64 \\
       \bottomrule
    \end{tabular*}
\end{table}
The policy is trained on the LIBERO-90 dataset using 30 expert demonstrations per task, using the slightly less data as the VLA and WAM baselines. The world model is trained for 10k steps with a global batch size of 64, using 100 trajectories per task from the LIBERO-90 suite exclusively, as per the protocol in Section~\ref{sec:video gen exp}.

\paragraph{Baseline Details} The baselines follow the same configuration as our diffusion policy, utilizing an action chunk length of 40 and a 10-step execution horizon. These models are trained on the LIBERO-90 dataset, with 40 to 50 demonstrations per task after filtering out unsuccessful demonstrations and no-op actions following \citet{kim2024openvla}. Crucially, we also included our VLM agent to decompose the subtasks and set the prompt to the baseline policies to be the prompt of the current sub-task, and the VLM switching to the prompt of the second sub-task when finishing the first one, to ensure the prompt input are completely in-distribution.

We train the $\pi_{0.5}$ baseline for 40k steps using the official LIBERO JAX implementation, achieving an average in-domain success rate of 95.8\% on the LIBERO-90 tasks. Evaluation on the LIBERO-Long suite indicates that while $\pi_{0.5}$ achieves high success rates on non-compositional tasks which only involves environmental variations, it struggles with compositional generalization, particularly in scenarios requiring active navigation toward secondary targets. Detailed performance metrics are summarized in Table~\ref{tab:pi libero10 data}. 
\begin{table}[htb]
    \centering
    \caption{\textbf{Performance of $\pi_{0.5}$ trained on LIBERO-90 and evaluated on LIBERO-Long.} Success is highest on the task with a near-duplicate counterpart in the training suite; the Kitchen Scene 4 task also has a nontrivial success rate.}
    \label{tab:pi libero10 data}
    \footnotesize
    \setlength{\tabcolsep}{4pt}
    \renewcommand{\arraystretch}{1.05}
    \begin{tabular*}{\linewidth}{@{\extracolsep{\fill}}>{\raggedright\arraybackslash}p{0.8\linewidth}c@{}}
        \toprule
        Task  &  Success rate (\%) \\
        \midrule
        LIVING ROOM SCENE2 put both the alphabet soup and the tomato sauce in the basket & 4 \\
        LIVING ROOM SCENE2 put both the cream cheese box and the butter in the basket & 0 \\
        KITCHEN SCENE3 turn on the stove and put the moka pot on it & 6 \\
        KITCHEN SCENE4 put the black bowl in the bottom drawer of the cabinet and close it & 44 \\
        LIVING ROOM SCENE5 put the white mug on the left plate and put the yellow and white mug on the right plate & 0 \\
        STUDY SCENE1 pick up the book and place it in the back compartment of the caddy & 90 \\
        LIVING ROOM SCENE6 put the white mug on the plate and put the chocolate pudding to the right of the plate & 0 \\
        LIVING ROOM SCENE1 put both the alphabet soup and the cream cheese box in the basket & 0 \\
        KITCHEN SCENE8 put both moka pots on the stove & 0 \\
        KITCHEN SCENE6 put the yellow and white mug in the microwave and close it & 0 \\
        \bottomrule
    \end{tabular*}
\end{table}
Success is highest when a test task closely mirrors the training distribution. For instance, the task ``STUDY SCENE1: pick up the book and place it in the back compartment of the caddy'' in LIBERO-Long has a direct counterpart in LIBERO-90 (``STUDY SCENE2''), with only minor environmental variations, and reaches 90\%. The Kitchen Scene 4 task reaches 44\%. This comparatively strong result may reflect the short transition between its two subtasks: after placing the bowl, the robot need only close the nearby drawer. By contrast, $\pi_{0.5}$ scores at most 4\% on the other four selected compositional tasks in Table~\ref{tab:app compositional full}. In these pick-and-place compositions, the gripper commonly stalls after completing the first subtask and rarely moves toward the second target object. These results reinforce our observation that imitation-learning policies rely heavily on memorized motion patterns and lack the test-time reasoning necessary to synthesize novel trajectories for adaptation. For the cosmos-policy baseline, constrained by computational resources, we trained the model following the official setup for 30k steps, achieving an action L1 loss of 0.025. This configuration yields an average in-domain success rate of 93\% on LIBERO-90 tasks. Among the five selected compositional tasks, cosmos-policy scores 12\% on the Kitchen Scene 4 task and 0\% on the other four.

Policy-enhancement methods such as SAILOR, GPC-RANK, and DREAMSTEER \citep{jain2025smooth,qi2025strengthening,cui2026dreamsteer}---which refine policy-generated actions rather than our VLM agent proposals---remain weak on average, reaching 16.4\%, 12.0\%, and 20.0\%, respectively. Kitchen Scene 4 is an exception, with success rates of 54\%, 50\%, and 60\%; on the other four selected compositions, SAILOR reaches at most 18\%, GPC-RANK at most 10\%, and DREAMSTEER at most 26\%. In those pick-and-place compositions, the base policy often produces near-zero-magnitude actions after finishing the first subtask. Gaussian perturbations in SAILOR's MPPI exploration \citep{jain2025smooth} can move the gripper toward the next object, but rarely yield successful task completion. These results highlight the importance of agent intervention in our system.

\begin{table}[htb]
    \centering
    \caption{\textbf{Per-task success rates (\%) for compositional task generalization.} These results underlie the compositional task generalization averages in Table~\ref{tab:libero10}.}
    \label{tab:app compositional full}
    \footnotesize
    \setlength{\tabcolsep}{2.5pt}
    \renewcommand{\arraystretch}{1.2}
    \begin{tabular*}{\linewidth}{@{\extracolsep{\fill}}>{\raggedright\arraybackslash}p{0.31\linewidth}*{7}{c}@{}}
        \toprule
        Compositional task & $\pi_{0.5}$ & cosmos & SAILOR & \shortstack{GPC-\\RANK} & \shortstack{DREAM\\STEER} & \shortstack{VLM\\planner} & \textbf{WAP} \\
        \midrule
        PnP alphabet soup \& tomato sauce & 4 & 0 & 18 & 10 & 26 & 56 & \textbf{72} \\
        PnP white mug \& yellow and white mug & 0 & 0 & 0 & 0 & 0& 28 & \textbf{68} \\
        PnP white mug \& chocolate pudding & 0 & 0 & 8 & 0 &14 & 46 & \textbf{78} \\
        PnP alphabet soup \& cream cheese box & 0 & 0 & 2 & 0 & 0 & 32 & \textbf{70} \\
        Put black bowl in bottom drawer \& close drawer & 44 & 12 & 54 & 50 & 60  & 48 & \textbf{72} \\
        \bottomrule
    \end{tabular*}
\end{table}

\newpage
\subsubsection{New Layout Generalization}\label{app:details libero object}
\begin{figure}[htb]
    \centering
    \includegraphics[width=\linewidth]{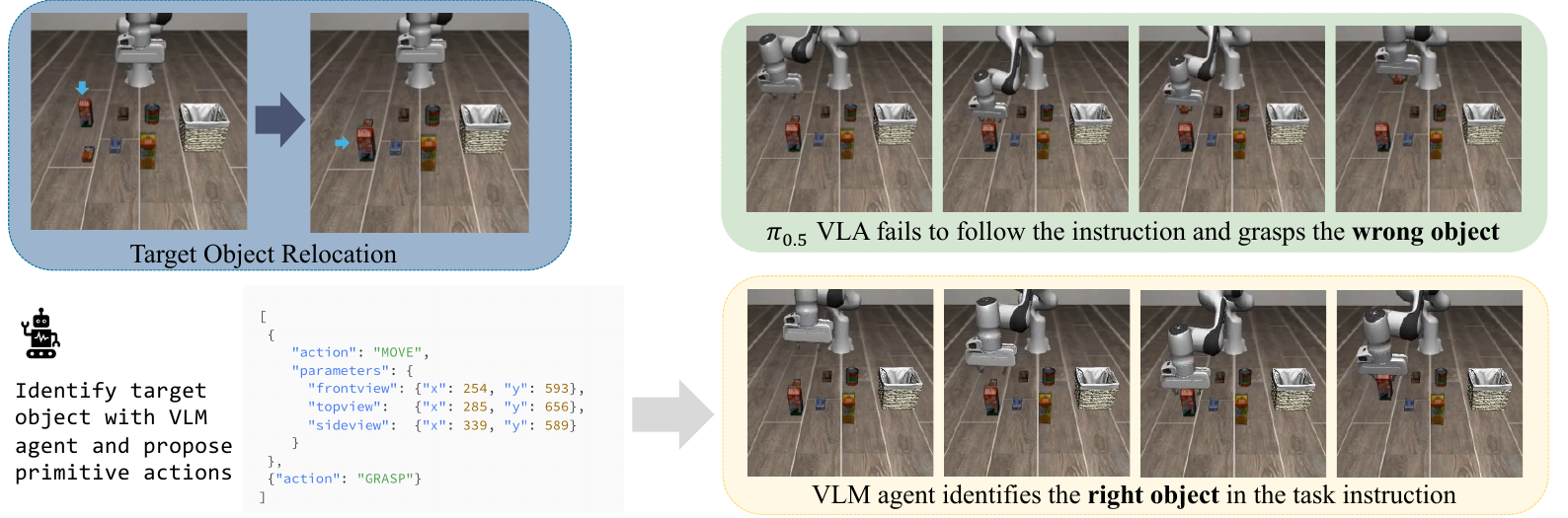}
    \caption{\textbf{Illustration of new layout generalization}}
    \label{fig:libero object}
\end{figure}
\paragraph{Environment Setup} We provide the images of the original layout and our modified layout below in Fig.~\ref{fig:app libero object layout}.
\begin{figure}[htb]
    \centering
    \includegraphics[width=\linewidth]{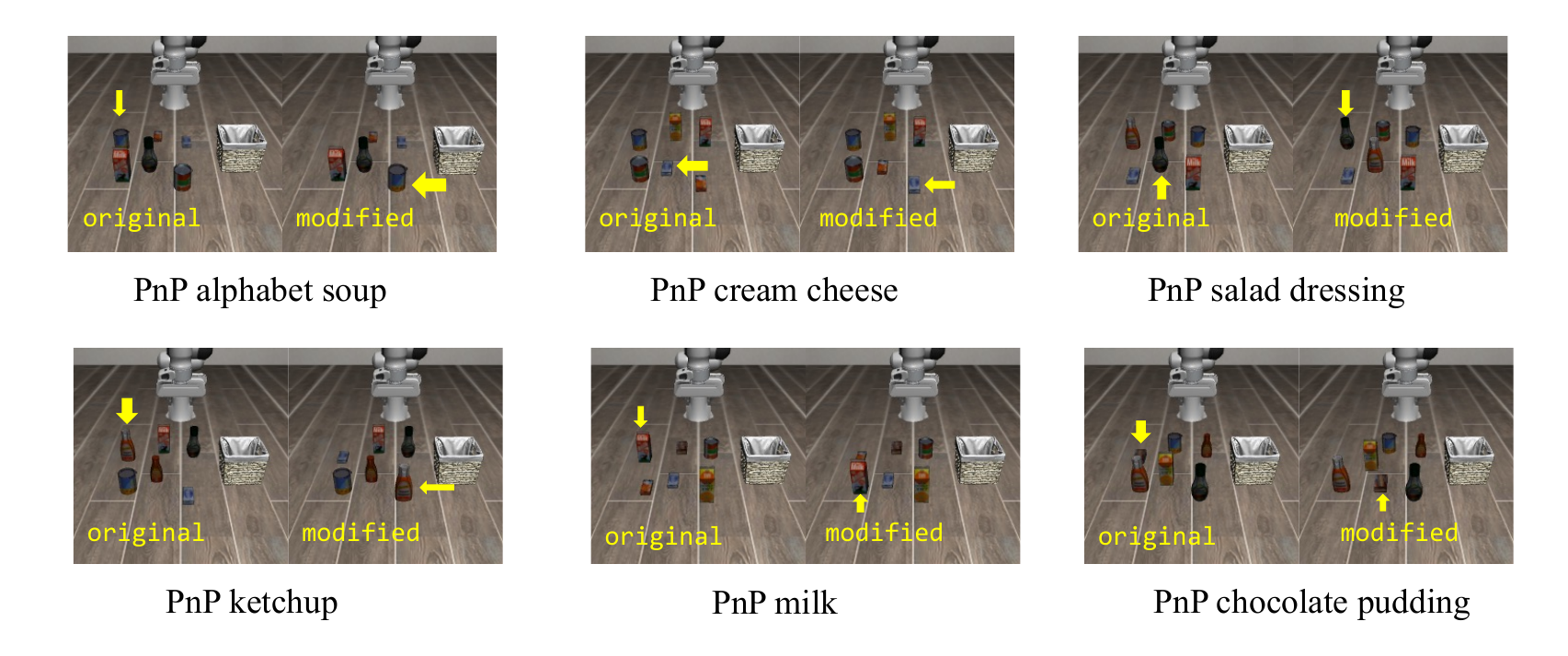}
    \caption{Original object layout and our modified object layout in LIBERO-Object. Yellow arrow points to the target object to be grasped into the basket.}
    \label{fig:app libero object layout}
\end{figure}

\paragraph{World Action Planner Setup} For our method, we fine-tune the diffusion policy and world model (pre-trained on LIBERO-90 as in the previous compositional task generalization experiments) using the official LIBERO-Object dataset under its original layout without the position modification introduced at test time. The diffusion policy is fine-tuned with only 5 demonstrations per task. To enhance the robustness of the world model, we augment the fine-tuning set with 10 additional trajectories generated by adding Gaussian noise (with std of 0.08 and 0.16) to the expert actions.

\paragraph{Baseline Setup} We directly evaluate the official LIBERO checkpoints for the $\pi_{0.5}$ and cosmos-policy baselines. We observe that $\pi_{0.5}$ achieves non-zero success rates on the ``PnP Chocolate Pudding'' task because the target object is not sufficiently displaced from its original training location; this allows the policy to occasionally capture the pudding during subsequent re-grasp attempts, as shown in Fig.~\ref{fig:app libero object pi suc}. In most other cases, however, the policy fails by erroneously targeting distractor objects near the original coordinates, as illustrated in Fig~\ref{fig: app libero object pi fail}. For other tasks where the target object is displaced far away, the $\pi_{0.5}$ model exhibits zero success rates, with similar patterns of grasping distractor objects placed near the original coordinates.
\begin{figure}
    \centering
    \begin{subfigure}{\linewidth}
    \centering
    \includegraphics[width=\linewidth]{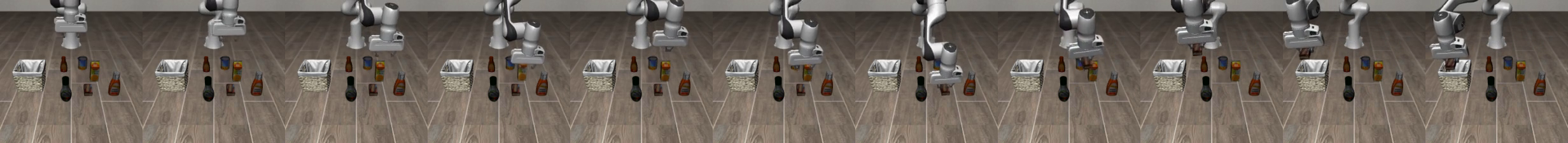}
    \caption{Low probability success example of $\pi_{0.5}$ on the modified layout of PnP chocolate pudding.}
    \label{fig:app libero object pi suc}
    \end{subfigure}\\
    
    \begin{subfigure}{\linewidth}
        \includegraphics[width=\linewidth]{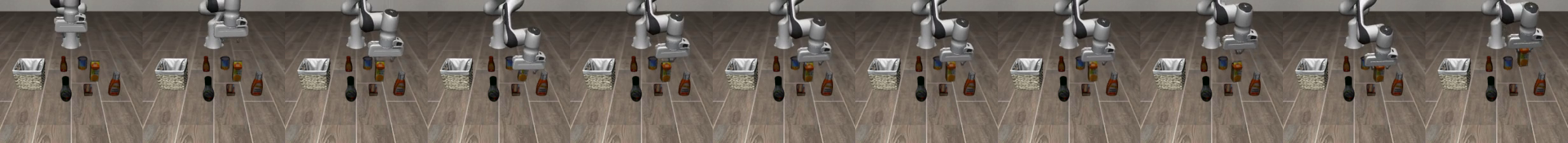}
        \caption{Failure example of $\pi_{0.5}$ on the modified layout of PnP chocolate pudding.}
        \label{fig: app libero object pi fail}
    \end{subfigure}
    \caption{Qualitative evaluation of $\pi_{0.5}$ on the modified layout of PnP chocolate pudding, down-sampled from 20 FPS to 1 FPS for clarity. The VLA exhibits low success rates, as the gripper rarely drifts toward the correct target and instead frequently grasps distractor objects placed at the original position during re-grasp attempts.}\label{fig:app libero object pi}
\end{figure}

The cosmos-policy baseline performs worse than $\pi_{0.5}$, yielding a zero success rate across all tasks. Notably, the WAM consistently repeats grasp attempts at the original training coordinates, failing to exhibit the local exploration observed in $\pi_{0.5}$. Qualitative results of the environment rollout and the predicted future states from the WAM are illustrated in Fig.~\ref{fig:app cosmos libero object}. While the predicted images are visually realistic, they do not facilitate successful task completion. These results suggest that because the model is trained primarily to mimic demonstration trajectories and predict action outcomes based on pre-trained video generation, it lacks the necessary reasoning and planning capabilities to generalize to novel tasks or environmental layouts.
\begin{figure}
    \centering
    \includegraphics[width=\linewidth]{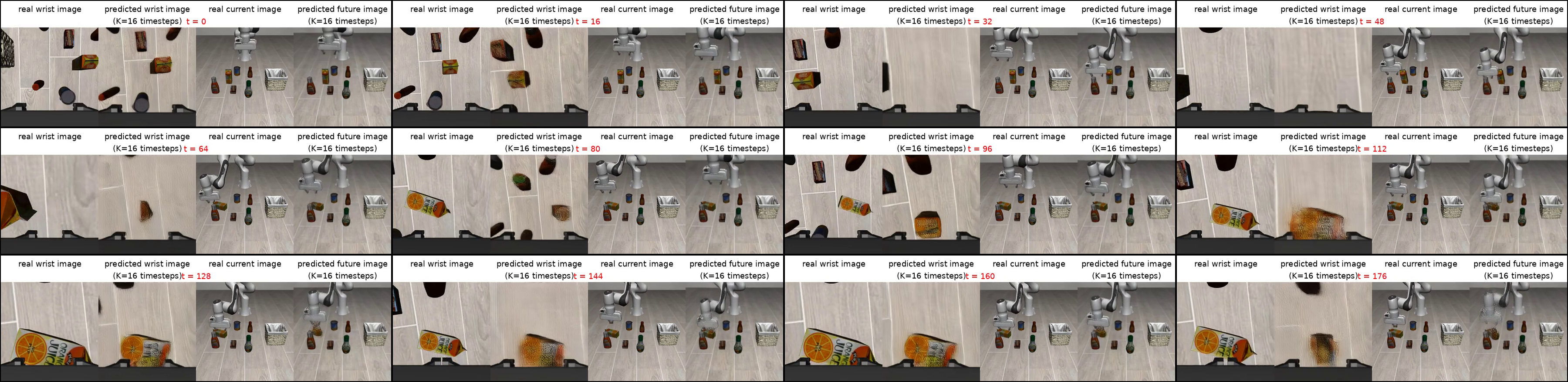}
    \caption{Qualitative evaluation of the cosmos-policy on the modified layout of the PnP chocolate pudding task, illustrating ground-truth environment rollouts alongside predicted future states from the WAM. Despite the modified layout, the model continues to target the original target object coordinates, resulting in repetitive and unsuccessful re-grasp cycles at the original position of the target object.}
    \label{fig:app cosmos libero object}
\end{figure}

For policy enhancement methods such as SAILOR, GPC-RANK, and DREAMSTEER \citep{jain2025smooth,qi2025strengthening,cui2026dreamsteer}, average success rates remain low at 3.7\%, 2.7\%, and 10.7\%, respectively, because the actions generated by the policy often target incorrect positions. These results underscore the importance of independent reasoning and planning when encountering novel scenarios, rather than relying exclusively on the demonstrations within the imitation learning dataset.

\begin{table}[htb]
    \centering
    \caption{\textbf{Per-task success rates (\%) for new-layout generalization.} These results underlie the new-layout averages in Table~\ref{tab:libero-object}.}
    \label{tab:app new layout full}
    \footnotesize
    \setlength{\tabcolsep}{2.5pt}
    \renewcommand{\arraystretch}{1.05}
    \begin{tabular*}{\linewidth}{@{\extracolsep{\fill}}>{\raggedright\arraybackslash}p{0.31\linewidth}*{7}{c}@{}}
        \toprule
        New-layout task & $\pi_{0.5}$ & cosmos & SAILOR & \shortstack{GPC-\\RANK} & \shortstack{DREAM\\STEER} & \shortstack{VLM\\planner} & \textbf{WAP} \\
        \midrule
        PnP alphabet soup & 0 & 0 & 0 & 0 & 8&  30 & \textbf{88} \\
        PnP cream cheese & 0 & 0 & 0 & 0 & 12 & 50 & \textbf{86} \\
        PnP salad dressing & 0 & 0 & 0 & 0 & 14 & 32 & \textbf{90} \\
        PnP ketchup & 0 & 0 & 0 & 0 & 0& 16 & \textbf{66} \\
        PnP milk & 0 & 0 & 0 & 0 & 20 & 34 & \textbf{84} \\
        PnP chocolate pudding & 10 & 0 & 22 & 16 & 10 & 64 & \textbf{78} \\
        \bottomrule
    \end{tabular*}
\end{table}

\newpage
\subsubsection{Novel Task Generalization}\label{app:robosuite details}
\begin{figure}[htb]
    \centering
    \includegraphics[width=\linewidth]{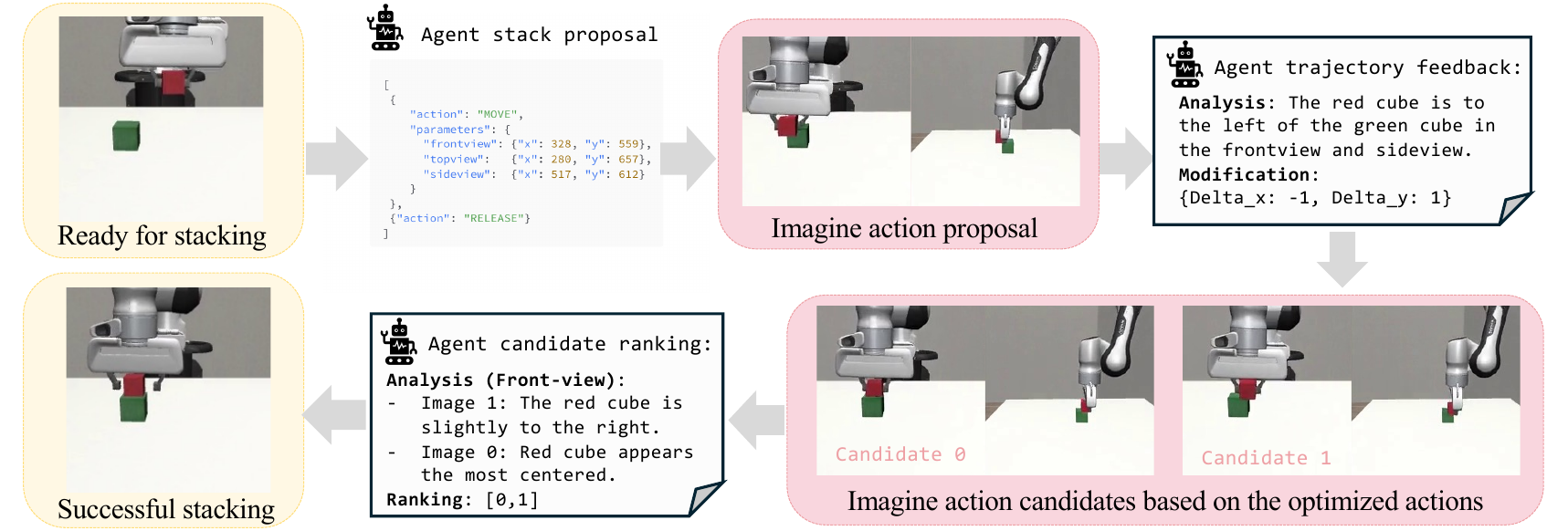}
    \caption{\textbf{Illustration of world action planner with novel task planning in \texttt{StackCube}.}}
    \label{fig:robosuite}
\end{figure}
\paragraph{Environment Setup} We evaluate on the \texttt{PickPlaceCan} and \texttt{StackCube} tasks of Robosuite \citep{zhu2020robosuite}, with 50 trials per task. Unlike the previous two settings, no task-specific policy and no expert demonstration are available here: the World Action Planner proposes and refines target end-effector poses, and the gripper is actuated by the two hard-coded primitives described below. This setting therefore isolates the planner itself, since the only learned components in the loop are the VLM agent and the action-conditioned world model.

\paragraph{Action Primitives} Actions are 7-dimensional end-effector operational space velocity commands $[\Delta x, \Delta y, \Delta z, \Delta \psi_r, \Delta \psi_p, \Delta \psi_y, g]$, where $g = 1$ closes and $g = -1$ opens the gripper, and each segment below is held for 10 control steps:
\begin{itemize}
    \item \texttt{GRASP}: descend with the gripper open, close the gripper in place, then lift with the gripper closed, \\
    $[0,0,-0.5,0,0,0,-1] \times 10 + [0,0,0,0,0,0,1] \times 10 + [0,0,1,0,0,0,1]\times 10$.
    \item \texttt{RELEASE}: open the gripper in place, then retract upward, \\
    $[0,0,0,0,0,0,-1]\times 10 + [0,0,1,0,0,0,-1]\times 10$.
\end{itemize}
We point out that these action primitives are not policies and they require no knowledge or prior about the specific task and environment configuration, while requiring only basic knowledge about the robot: they are expressed purely in the gripper frame and reference no object pose, target location, or scene layout. The same two primitives are reused unchanged across both tasks, and across the real-robot experiments in Appendix~\ref{app:real robot}. All object-dependent decisions---where to grasp, at what orientation, and where to place---are made by the planner rather than encoded in the primitives.

\paragraph{World Action Planner Setup} We initialize the world model checkpoint from the LIBERO-90 world model of the compositional task generalization experiments (Appendix~\ref{app: libero long details}). At test time, we collect 50 trajectories through exploration using the VLM agent without any expert demonstrations or robot policies, training the world model on the data for 10 epochs. We point out that although some of these exploratory trajectories happen to be successful, we use them solely as video data for world model test-time training and never as demonstrations for a policy; prior methods \citep{jain2025smooth,qi2025strengthening} instead require expert demonstrations in addition to the in-distribution sampled trajectories in order to succeed. Search scales follow Appendix~\ref{app:planner detail}: a global modification scale of $0.06$, and a local grid scale $\delta$ of $0.02$ for \texttt{StackCube} and $0.05$ for \texttt{PickPlaceCan}, following the rule of setting the scale to roughly half the size of the manipulated object. The tighter scale on \texttt{StackCube} reflects the smaller cube and the tighter placement tolerance of stacking, which is also where the gap to the vision-language planner is largest: without imagined rollouts to verify the placement pose, the VLM planner reaches only 22\% on \texttt{StackCube} against 58\% on \texttt{PickPlaceCan}, while WAP reaches 76\% and 80\% respectively (Table~\ref{tab:app zero shot full}).

\paragraph{Baseline Setup} The policy-based baselines depend on a learned base policy, whereas WAP uses no learned task policy in this setting. We evaluate them under two regimes: \textit{zero-shot}, using the pretrained base policies without target-environment finetuning; and \textit{same-data adaptation}, finetuning the corresponding base policies on the same 50 VLM-collected exploratory trajectories used to train WAP's world model. These trajectories are diverse exploration data rather than expert demonstrations. The base policies remain weak after same-data adaptation, and the policy-enhancement methods improve them only modestly because their candidates are still sampled from the base-policy distribution. Table~\ref{tab:app zero shot full} reports both regimes.

\begin{table}[htb]
    \centering
    \caption{\textbf{Per-task success rates (\%) under two policy-training regimes.} For policy-based baselines, ``zero-shot'' denotes no target-environment policy finetuning, while ``same data'' denotes finetuning on the same 50 exploratory trajectories used to train WAP's world model. The vision-language planner and WAP use no learned task policy, so each has one result per task spanning both regimes. Each result is over 50 trials.}
    \label{tab:app zero shot full}
    \footnotesize
    \setlength{\tabcolsep}{3pt}
    \renewcommand{\arraystretch}{1.05}
    \begin{tabular*}{0.88\linewidth}{@{\extracolsep{\fill}}>{\raggedright\arraybackslash}p{0.25\linewidth}cccc@{}}
        \toprule
        \multirow{2}{*}{Method} & \multicolumn{2}{c}{\texttt{PickPlaceCan}} & \multicolumn{2}{c}{\texttt{StackCube}} \\
        \cmidrule(lr){2-3}\cmidrule(l){4-5}
        & Zero-shot & Same data & Zero-shot & Same data \\
        \midrule
        $\pi_{0.5}$ & 0 & 16 & 0 & 4 \\
        cosmos-policy & 0 & 4 & 0 & 0 \\
        SAILOR & 0 & 8 & 0 & 4 \\
        GPC-RANK & 0 & 10 & 0 & 6 \\
        DREAMSTEER & 8 & 24 & 0 & 12 \\
        \midrule
        Vision-language planner & \multicolumn{2}{c}{58} & \multicolumn{2}{c}{22} \\
        \textbf{World action planner} & \multicolumn{2}{c}{\textbf{80}} & \multicolumn{2}{c}{\textbf{76}} \\
        \bottomrule
    \end{tabular*}
\end{table}
% ---------------------------------------------------------------------------
% TODO(author): reviewer-facing details still missing from this subsection.
%  1. Which checkpoints the pi_0.5 and cosmos-policy baselines use here
%     (the LIBERO-90 models from D.2.1, the official LIBERO checkpoints,
%     or something Robosuite-specific), and which policy SAILOR and
%     GPC-RANK sample from. The Baseline Setup paragraph above argues
%     structurally; a reviewer will still want the concrete answer.
%  2. Success criterion for each task, and the initial-state
%     randomization used across the 50 trials.
%  3. What converts the planned end-effector pose into actions here --
%     presumably the pose-conditioned diffusion policy of App. C.2,
%     which is worth naming explicitly.
%  4. Failure analysis for the remaining 20% / 24% of WAP failures.
% ---------------------------------------------------------------------------

\subsubsection{Real Robot Novel Task Experiments}\label{app:real robot}
Here we provide additional details for the real robot experiments.

\paragraph{Hardware and Task Setup} We evaluate on a 7-DoF Franka Panda arm with a parallel-jaw gripper, observed by two exterior cameras as shown in Fig.~\ref{fig:real-robot}; the world model imagines both views jointly as two panels of $256\times256$ each. The task is to pick a can from the table and place it inside a white basket. We adopt the same protocol as the novel task generalization setting of Appendix~\ref{app:robosuite details}: no task-specific policy and no expert demonstration are available, gripper closing and opening are issued as hard-coded primitives, and the planner is responsible only for the target end-effector poses, which are executed by the controller shipped with the Franka Panda without any customized or learned low-level controller. The entire real-robot pipeline therefore consists of only the VLM agent, the world model, and the shipped controller. We randomize the position and orientation of both the can and the basket for every rollout.

\paragraph{Failure Exploration Data Collection} We first collect 10 trajectories on the \texttt{PnP green can} task using the VLM proposal alone. All 10 exploration trajectories fail: the VLM either collides with the rim or drop the can outside the basket. (Fig.~\ref{fig:real robot fail real}). This failure mode is a direct consequence of the VLM reasoning about the scene semantically rather than physically. We stress that these trajectories contain \textit{no successful demonstration}; unlike prior world-model policy enhancement methods \citep{jain2025smooth,qi2025strengthening}, we never provide the system with an expert trajectory for the real robot.
\begin{figure}[htb]
    \centering
    \includegraphics[width=0.92\linewidth]{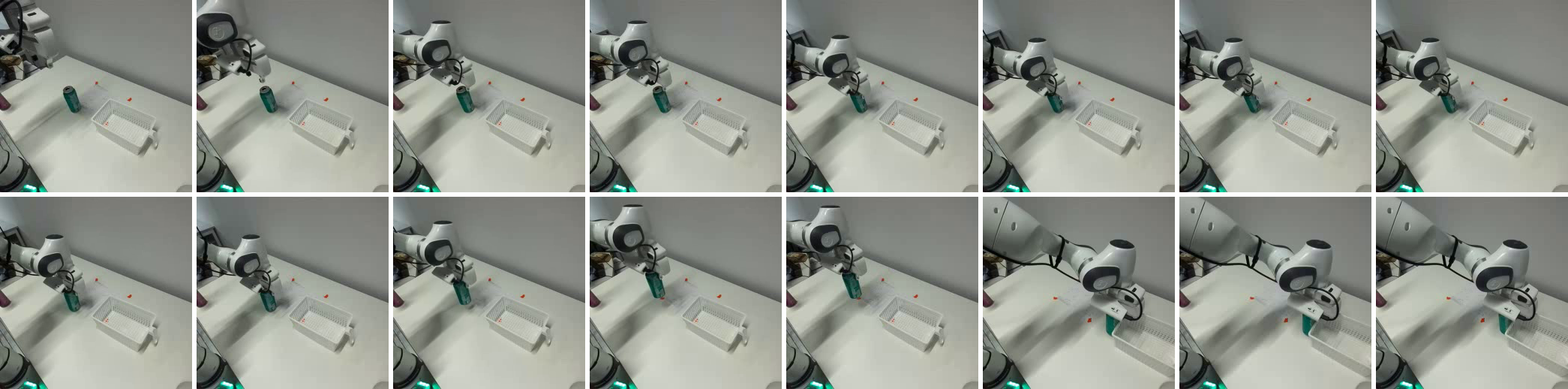}
    \caption{\textbf{Representative failed VLM trajectory on the real robot.} Frames are ordered left to right, top to bottom. The VLM grasps the can correctly but commands a placement pose that drives the can into the basket rim.}
    \label{fig:real robot fail real}
\end{figure}

\paragraph{World Model Training at Test Time} Starting from the same LIBERO-90 world model of the compositional task generalization experiments (Appendix~\ref{app: libero long details}), we train the world model on these 10 failed exploration trajectories for 100 steps at test time deployment. Fig.~\ref{fig:world-model-qualitative} shows imagined rollouts after finetuning for both the green can and the \textit{previously unseen} purple can: in both cases the world model reproduces the contact between the can and the basket rim, indicating that the collision is predicted from the geometry of the commanded motion rather than memorized from the green-can trajectories. This transfer is what makes the planner usable on a new object without collecting any new data.
\begin{figure}[htb]
    \centering
    \begin{subfigure}{0.92\linewidth}
        \centering
        \includegraphics[width=\linewidth]{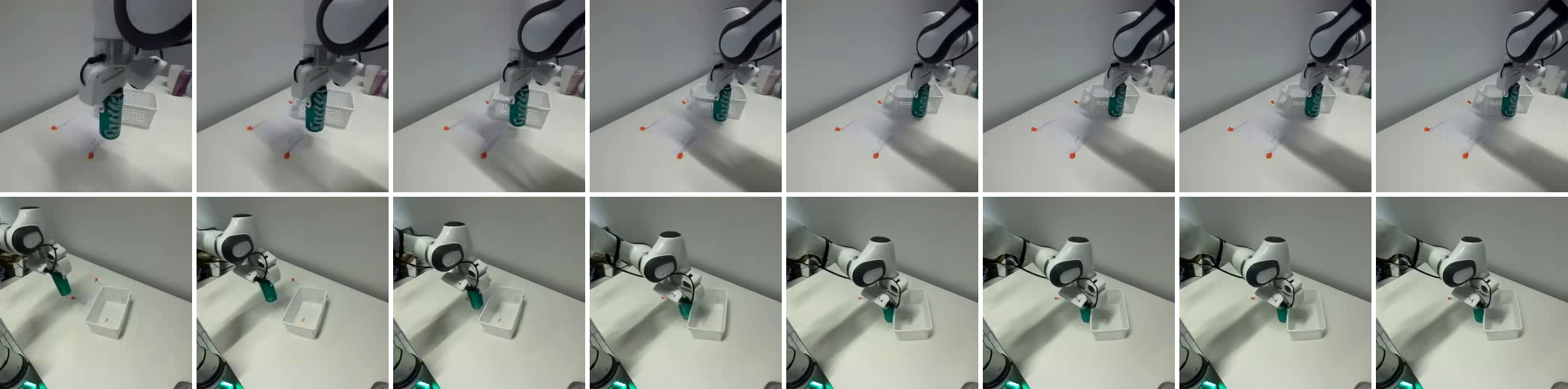}
        \caption{Placing the green can, seen during world model finetuning.}
    \end{subfigure}

    \vspace{4pt}

    \begin{subfigure}{0.92\linewidth}
        \centering
        \includegraphics[width=\linewidth]{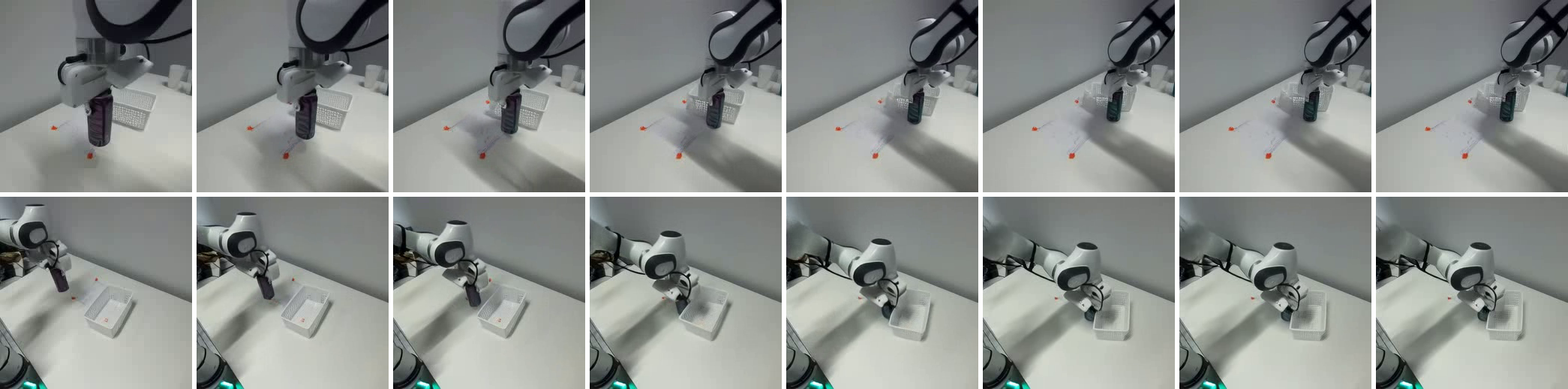}
        \caption{Placing the purple can, never seen during world model finetuning.}
    \end{subfigure}
    \caption{\textbf{Qualitative examples of real-robot world-model imagination.} Each panel shows the two imagined exterior camera views, one per row. The imagined rollout exposes the collision between the can and the basket rim before any action is executed.}
    \label{fig:world-model-qualitative}
\end{figure}

\paragraph{Planning and Results} Given the imagined rollout, the VLM agent reports the predicted collision and global action optimization raises the placement pose for the can to clear the rim, after which the trajectory is executed and the can is released inside the basket (Fig.~\ref{fig:real robot success}). Per-task success rates over 20 trials per task are reported in Table~\ref{tab:app real robot full}. The vision-language planner reaches only 10\% on both cans, reproducing the rim collision it cannot foresee, whereas WAP reaches 70\% on the green can and 60\% on the purple can. The imitation-learning and policy-enhancement baselines score 0\% on most cases because they require finetuning on successful demonstrations, while the exploratory trajectories collected are all failures. The exception is DREAMSTEER where one trial succeeded due to sampling from the primitive library \citep{cui2026dreamsteer}. Of the 14 failed trials across both tasks, six leave insufficient clearance over the basket rim, four lose the can when the gripper slips during grasping, and four release the can slightly outside the basket.
\begin{figure}[htb]
    \centering
    \includegraphics[width=0.92\linewidth]{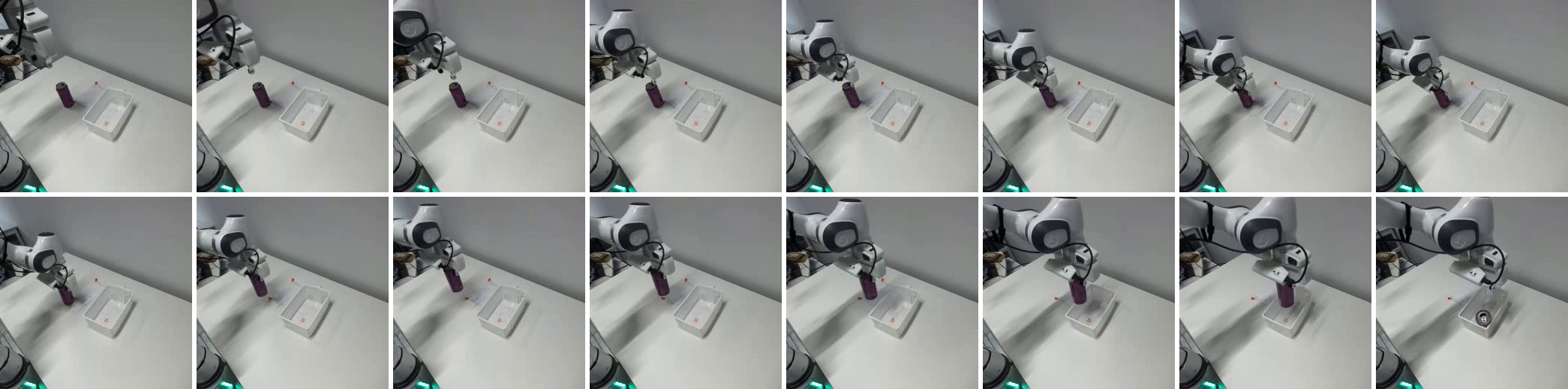}
    \caption{\textbf{Successful \texttt{PnP purple can} rollout with the World Action Planner.} Frames are ordered left to right, top to bottom. The optimized placement pose clears the basket rim, and the can is released inside the basket.}
    \label{fig:real robot success}
\end{figure}
\begin{table}[htb]
    \centering
    \caption{\textbf{Per-task success rates (\%) under two policy-training regimes.} For policy-based baselines, ``zero-shot'' denotes no target-environment policy finetuning, while ``same data'' denotes finetuning on the same 10 exploratory trajectories used to train WAP's world model. The vision-language planner and WAP use no learned task policy, so each has one result per task spanning both regimes. Each result is over 20 trials.}
    \label{tab:app real robot full}
    \footnotesize
    \setlength{\tabcolsep}{3pt}
    \renewcommand{\arraystretch}{1.05}
    \begin{tabular*}{0.88\linewidth}{@{\extracolsep{\fill}}>{\raggedright\arraybackslash}p{0.25\linewidth}cccc@{}}
        \toprule
        \multirow{2}{*}{Method} & \multicolumn{2}{c}{\texttt{PnP green can}} & \multicolumn{2}{c}{\texttt{PnP purple can}} \\
        \cmidrule(lr){2-3}\cmidrule(l){4-5}
        & Zero-shot & Same data & Zero-shot & Same data \\
        \midrule
        $\pi_{0.5}$ & 0 & 0 & 0 & 0 \\
        cosmos-policy & 0 & 0 & 0 & 0 \\
        SAILOR & 0 & 0 & 0 & 0 \\
        GPC-RANK & 0 & 0 & 0 & 0 \\
        DREAMSTEER & 0 & 10& 0 & 0 \\
        \midrule
        Vision-language planner & \multicolumn{2}{c}{10} & \multicolumn{2}{c}{10} \\
        \textbf{World action planner} & \multicolumn{2}{c}{\textbf{70}} & \multicolumn{2}{c}{\textbf{60}} \\
        \bottomrule
    \end{tabular*}
\end{table}

\paragraph{Successful Imagination of the World Model}
Here we demonstrate that our world model can imagine the successful trajectories after placement, which it has never seen during training. As shown in Fig.~\ref{fig:world-model-sucess}, the world model accurately imagined the successful clearance during placement, an action trajctory it has never seen during finetuning. Such results shows that the world model successfully generalizes to the real world scene using only 10 failure trajectories, despite being trained on only limited simulation data, demonstrating the superior generalization capabilities of our world model.
\begin{figure}[h]
    \centering
    \begin{subfigure}{0.92\linewidth}
        \centering
        \includegraphics[width=\linewidth]{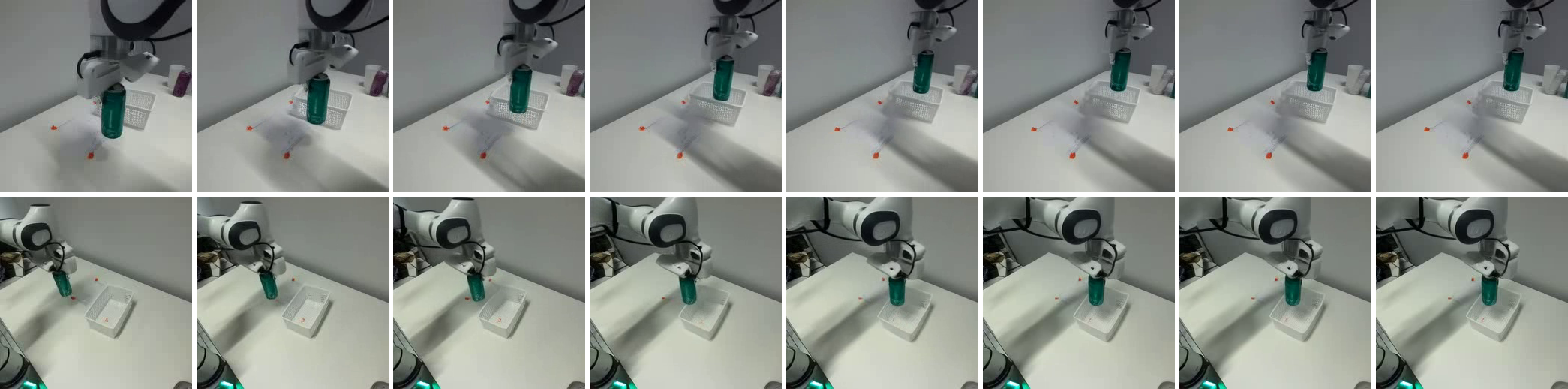}
        \caption{World model imagination of successfully placing the green can.}
    \end{subfigure}
    \begin{subfigure}{0.92\linewidth}
        \centering
        \includegraphics[width=\linewidth]{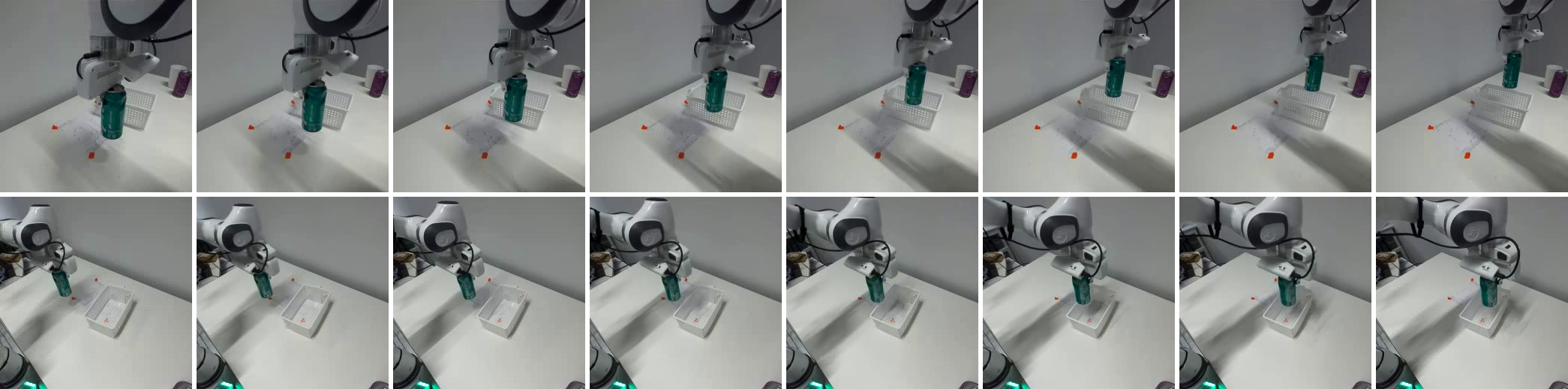}
        \caption{Real Execution of successfully placing the green can.}
    \end{subfigure}

    \vspace{4pt}

    \begin{subfigure}{0.92\linewidth}
        \centering
        \includegraphics[width=\linewidth]{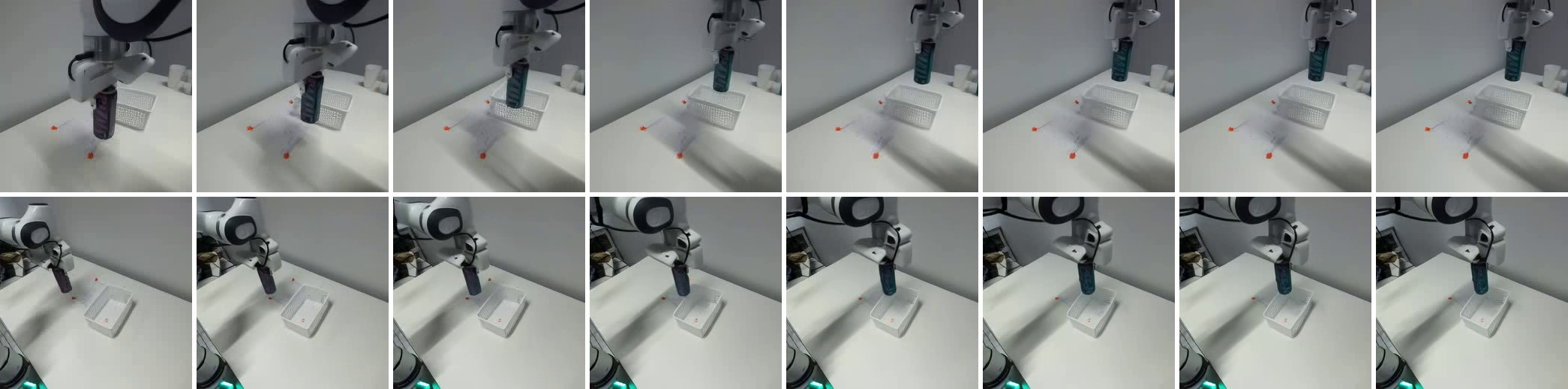}
        \caption{World model imagination of successfully placing the purple can.}
    \end{subfigure}
    \begin{subfigure}{0.92\linewidth}
        \centering
        \includegraphics[width=\linewidth]{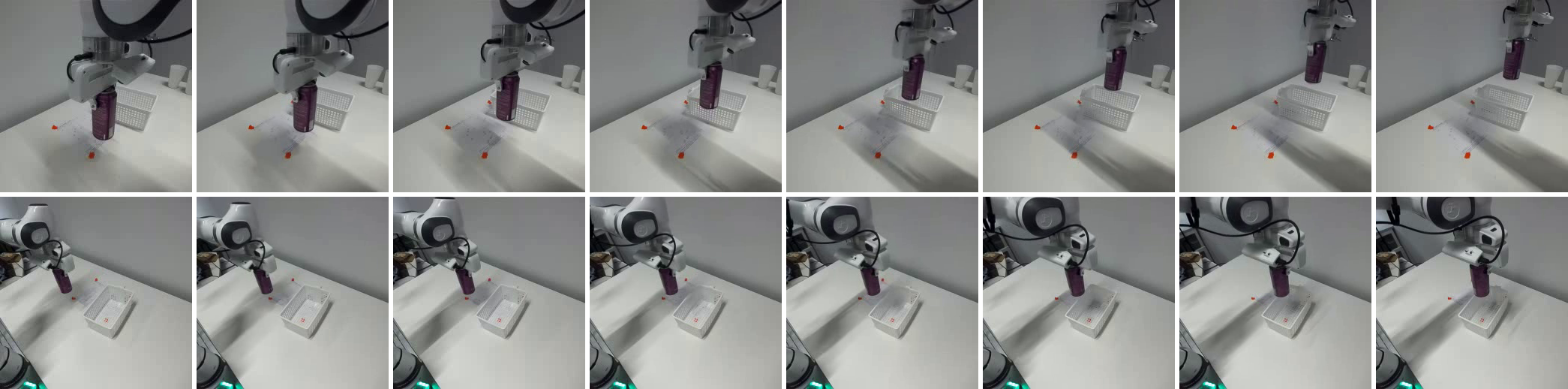}
        \caption{Real Execution of successfully placing the purple can.}
    \end{subfigure}
    \caption{\textbf{Qualitative examples of real-robot world-model imagination versus ground truth for successful placement actions.} Each panel shows the two exterior camera views, one per row. The imagined and real execution corresponds to the same actions and states. The imagined rollout versus ground truth demonstrates that the world model can accurately imagine the successful lifted actions that it has never seen during training.}
    \label{fig:world-model-sucess}
\end{figure}

\clearpage
\subsection{Planning Sample Efficiency}\label{app:efficiency}
For the baseline methods, we use the same initial action proposal and rollout horizon as WAP to ensure a fair comparison. We provide the baselines with ground-truth simulator rollouts to ensure a stronger baseline, while WAP instead relies on trajectories imagined by the world model.

We adopt the same MPPI-style update rules, and search the hyper-parameters starting from their original hyper-parameter choices, and ensure using the same number of VLM calls, so wall-time is determined by the number of samples. We also provide the results with the original curated reward design instead of our VLM agent in Table~\ref{tab:mpc & bon app}.

\begin{table}[htb]
    \centering
    \caption{\textbf{Planning sample efficiency against MPC.} Success rates (\%) are averaged within each setting. VLMPC uses ground-truth simulator rollouts, whereas WAP uses world-model imagination.}
    \label{tab:mpc & bon app}
    \small
    \begin{tabular*}{0.9\linewidth}{@{\extracolsep{\fill}}lccc@{}}
        \toprule
        Method & Compositional task & New layout & Novel task \\
        \midrule
        VLMPC original reward (8 samples) & 34.0 & 36.3 & 64.0 \\
        VLMPC original reward (16 samples) & 48.0 & 44.0 & 72.0 \\
        \midrule
        VLMPC agent reward (8 samples) & 36.0 & 38.3 & 64.0 \\
        VLMPC agent reward (16 samples) & 46.0 & 50.0 & 76.0 \\
        \midrule
        \textbf{WAP w/ global optimization only (1 sample)} & 58.0 & 66.0 & 65.0 \\
        \textbf{WAP (6 samples)} & \textbf{72.0} & \textbf{82.0} & \textbf{78.0} \\
        \bottomrule
    \end{tabular*}
    \vspace{-5pt}
\end{table}

We find that MPC methods based on Gaussian sampling often waste computation on implausible candidates produced by an undirected proposal distribution. In contrast, our global optimization is guided by VLM feedback and can mitigate potential risks from a single world-model-imagined trajectory (Fig.~\ref{fig:ablation}), whereas MPPI typically requires multiple sampling iterations to refine its action distribution.
In addition, our local grid search evaluates physically grounded candidates at neighboring poses, further reducing the number of implausible samples that can arise from the Gaussian sampling used in CEM or MPPI \citep{ning2025prompting}.
\begin{figure}[htb]
    \centering
    \includegraphics[width=\linewidth]{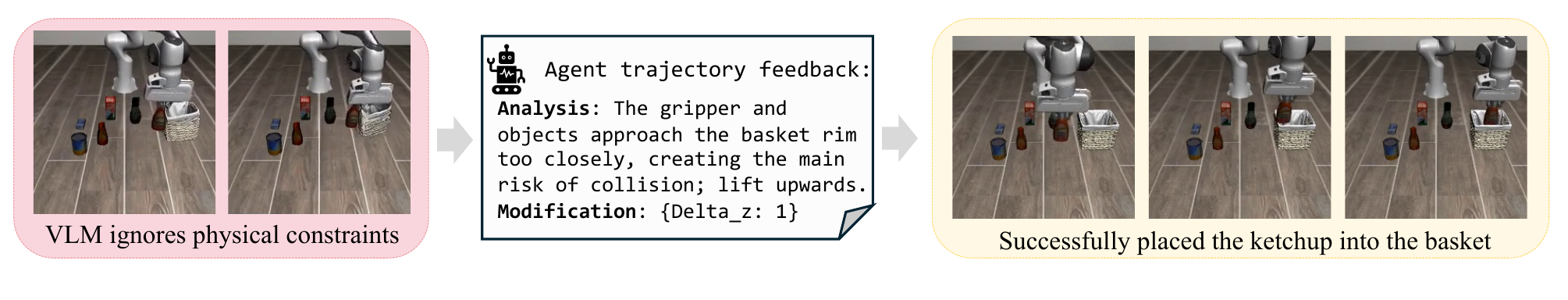}
    \caption{\textbf{Global Action Optimization.} While VLMs often ignore physical constraints, our framework leverages world model imaginations to identify potential collisions and optimize action sequences for safe execution.}
    \label{fig:ablation}
    \vspace{-5pt}
\end{figure}

\subsection{Ablation Experiments}\label{app:ablation exp}
Here we provide the results for the ablation experiments.

\subsubsection{Ablation of Planning Stages} \label{app:ablation of planning}
We conduct a stage-wise ablation study to evaluate the contributions of each module within the World Action Planner. The results, presented in Table~\ref{tab:ablation planner}, demonstrate that each stage incrementally improves the designed pipeline. For easy tasks, object clearance via global action optimization is sufficient. 
\begin{table}[htb]
    \centering
    \caption{\textbf{Results for ablations of our action planning pipeline.} We ablate different components in our world action planner using the LIBERO-Long tasks, where each component contributes to higher success rates.}
    \label{tab:ablation planner}
    \footnotesize
    \renewcommand{\arraystretch}{1.05}
    \begin{tabular*}{\linewidth}{@{\extracolsep{\fill}}>{\raggedright\arraybackslash}p{0.34\linewidth}cccc@{}}
        \toprule
        Task & VLM planner & + Global & + Local & + Policy rollout \\
        \midrule
        PnP alphabet soup \& tomato sauce & 56 & 64 & 64 & \textbf{72} \\
        PnP white mug \& yellow and white mug & 28 & 40 & 54 & \textbf{68} \\
        PnP white mug \& chocolate pudding & 46 & 70 & 72 & \textbf{78} \\
        PnP alphabet soup \& cream cheese box & 32 & 48 & 66 & \textbf{70} \\
        \bottomrule
    \end{tabular*}
\end{table}
Notably, for the PnP alphabet soup and cream cheese task, when grasping the cream cheese box after placing the alphabet soup, the gripper is rotated, requiring the agent to find the right orientation for the grasp. While the VLM struggles to describe the required rotation analytically, it can effectively identify the optimal orientation when evaluating imaginations of candidate poses, as illustrated in Fig.~\ref{fig:app ablation rotation.}.
\begin{figure}[htb]
    \centering
    \includegraphics[width=\linewidth]{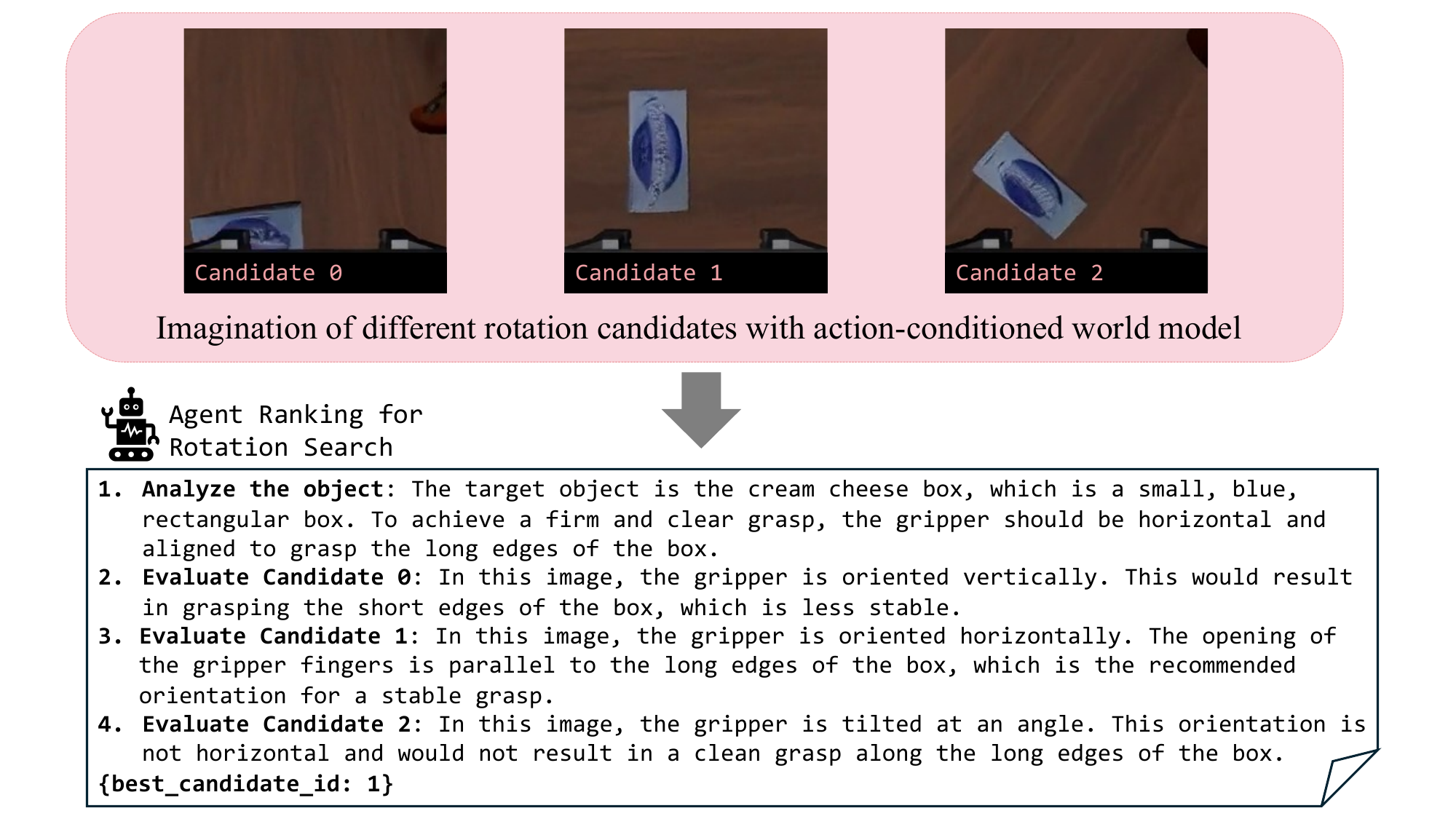}
    \caption{Illustration of local search for gripper rotation. The VLM successfully identifies the right rotation from imagination of candidates, while failing to describe the correct rotation analytically. }
    \label{fig:app ablation rotation.}
\end{figure}
For the fine-grained PnP white mug and yellow-and-white mug task, transitioning from the first mug to the next presents a specific challenge: it is often unclear which state serves as a valid ``in-distribution'' starting point for the policy to successfully trigger a grasp and complete the second sub-task. To resolve this ambiguity, we imagine consecutive policy rollouts following various action candidates to determine which state allows the policy to effectively continue the task. As illustrated in Fig.~\ref{fig:app ablation policy imagination}, we can identify the optimal state for completing the grasp by imagining the consecutive policy rollout (shown in row 2), whereas selecting the correct candidate remains difficult without the benefit of further imagination (as in row 1).
\begin{figure}[htb]
    \centering
    \includegraphics[width=\linewidth]{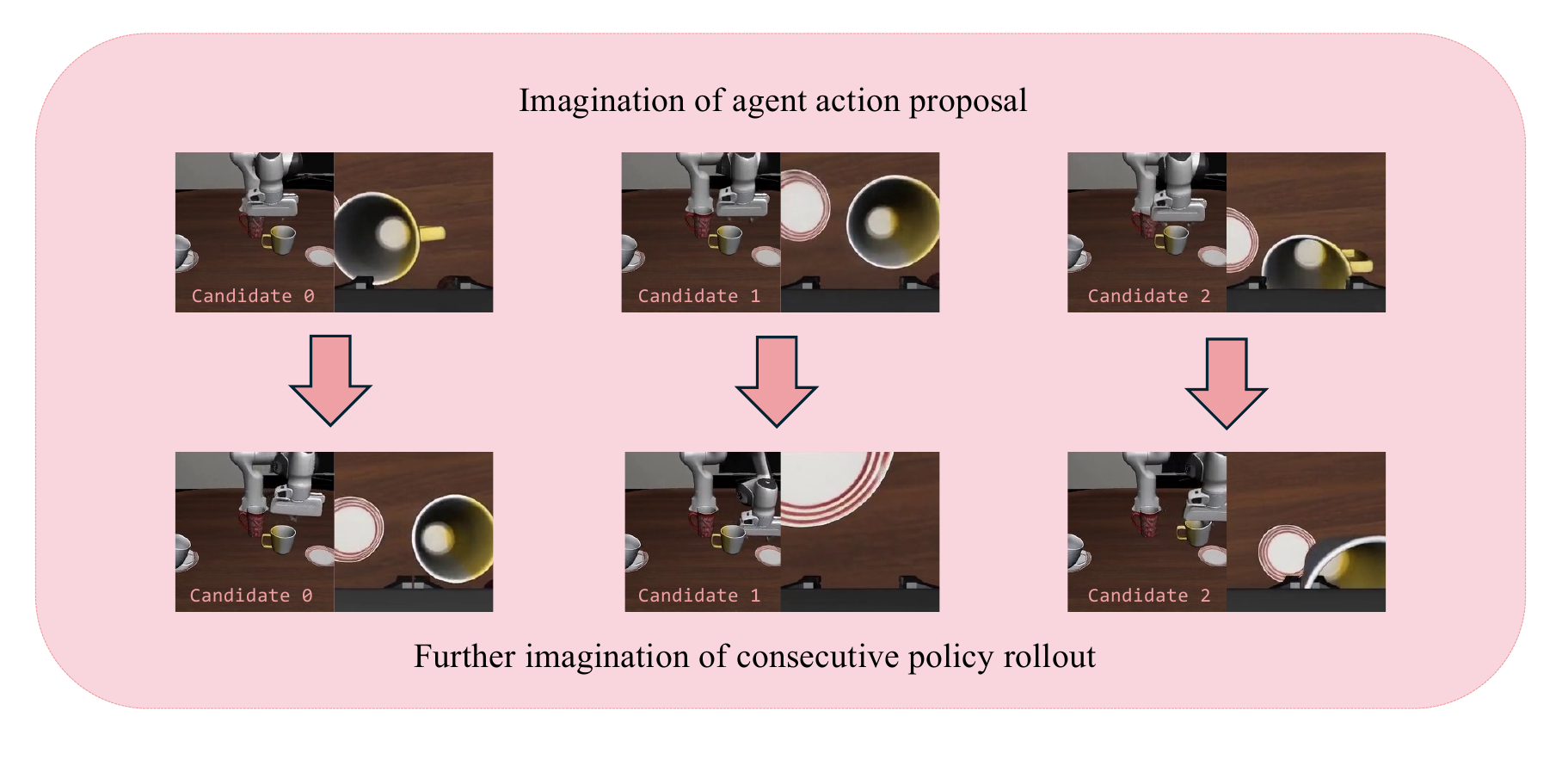}
    \caption{Illustration of imagining consecutive policy rollout. While in the first row, it is hard to identify which one is the good state to continue the second sub-task by grasping the mug, we can easily see that for Candidate 2, the policy successfully grasped the mug and continues the second sub-task from the consecutive policy rollout imagination in the second row. }
    \label{fig:app ablation policy imagination}
\end{figure}

\subsubsection{Ablation of VLM Agent}\label{app:ablation vlm agent}
Here we provide ablation of the choice of VLM agent from the default Gemini 3.0 Flash against an additional GPT-5.4 in the new-layout generalization setup. We report the Vision-language-planner baseline for each agent that directly executes action proposal without search along our world action planner results. As shown in Table~\ref{tab:app agent ablation}, the results remain consistent across VLM agent choices.

\begin{table}[h]
    \centering
    \caption{\textbf{Ablation of VLM agent.} Per-task success rates (\%) across VLM agent choices.}
    \label{tab:app agent ablation}
    \footnotesize
    \setlength{\tabcolsep}{3pt}
    \renewcommand{\arraystretch}{1.1}
    \begin{tabular*}{\linewidth}{@{\extracolsep{\fill}}lcccccc@{}}
        \toprule
        {VLM Agent}
        & {PnP soup}
        & {PnP cheese}
        & {PnP salad dressing}
        & {PnP ketchup}
        & {PnP milk}
        & {PnP pudding} \\
        \midrule
        GPT-5.4 baseline  & 26 & 52 & 30 & 14 & 38 & 68 \\ 
        GPT-5.4 with WAP & 84 & 82 & 88 & 64 & 80 & 76\\
        Gemini-3-flash baseline & 30 & 50 & 32 & 16 & 34 & 64 \\
        Gemini-3-flash with WAP & 88 & 86 & 90 & 66 & 84 & 78 \\
        \bottomrule
    \end{tabular*}
\end{table}

\subsubsection{Robustness Analysis}\label{app:robustness}
Here we evaluate the hyperparameter robustness of our method, especially sensitivity to local search scale and global search scale. As described in Appendix~\ref{app:detail method}, we chose the search scale to be approximately half the object size or height in magnitude. We further provide ablation results on these parameters in Table~\ref{tab:app search scale robustness}, highlighting the robustness of our method.

\begin{table}[h]
    \centering
    \caption{\textbf{Robustness to local and global search scales.} Per-task success rates (\%) across search-scale choices. ``Original'' denotes the default setting used in our main experiments.}
    \label{tab:app search scale robustness}
    \footnotesize
    \setlength{\tabcolsep}{3pt}
    \renewcommand{\arraystretch}{1.1}
    \begin{tabular*}{\linewidth}{@{\extracolsep{\fill}}ccccc@{}}
        \toprule
        \multicolumn{5}{l}{\textbf{Local search scale ablation}} \\
        \addlinespace[2pt]
        \multicolumn{1}{l}{Local search scale}
        & \multicolumn{1}{l}{\shortstack[l]{PnP soup \&\\ tomato sauce}}
        & \multicolumn{1}{l}{\shortstack[l]{PnP white mug \&\\ yellow and white mug}}
        & \multicolumn{1}{l}{\shortstack[l]{PnP mug \&\\ pudding}}
        & \multicolumn{1}{l}{\shortstack[l]{PnP soup \&\\ cheese}} \\
        \midrule
        0.04 & 66 & 66 & 74 & 72 \\
        0.05 (original) & 72 & 68 & 78 & 70 \\
        0.07 & 70 & 72 & 74 & 66 \\
        \bottomrule
    \end{tabular*}

    \vspace{6pt}

    \begin{tabular*}{\linewidth}{@{\extracolsep{\fill}}ccccccc@{}}
        \toprule
        \multicolumn{7}{l}{\textbf{Global search scale ablation}} \\
        \addlinespace[2pt]
        \multicolumn{1}{l}{Global search scale}
        & \multicolumn{1}{l}{PnP soup}
        & \multicolumn{1}{l}{PnP cheese}
        & \multicolumn{1}{l}{\shortstack[l]{PnP salad\\ dressing}}
        & \multicolumn{1}{l}{PnP ketchup}
        & \multicolumn{1}{l}{PnP milk}
        & \multicolumn{1}{l}{PnP pudding} \\
        \midrule
        0.04 & 78 & 80 & 86 & 62 & 80 & 76 \\
        0.06 (original) & 88 & 86 & 90 & 66 & 84 & 78 \\
        0.08 & 86 & 82 & 84 & 70 & 90 & 80 \\
        \bottomrule
    \end{tabular*}
\end{table}

\newpage
\section{Wall Clock Time Analysis}\label{app:clocl}
Here we report the wall clock time of our world model. When training on a H100 GPU with batch size 2, one single training step takes approximately 4.8 seconds. When running inference on a A100 GPU with batch size 1, one single forward takes 0.85 seconds, so predicting 20 frames requires 20 denoising steps and takes around 17 seconds. As a result, the global optimization will take around 30 seconds including world model imagination and VLM reasoning, and the local search can take 2 to 3 minutes. For the real world experiments, a complete task completion will take approximately 2 minutes with global optimization. We note that the computation cost of local search can be significantly decreased by adopting an auto-regressive world model and using KV cache for the shared history frames among all the candidates, which we leave as future work.

We emphasize that while generating imagined rollouts with the world model is computationally intensive, the system does not require continuous re-planning at every control step. In our experiments, we only invoke the planning system at automatically with the VLM at sub-task starting points. Furthermore, the overall planning latency is significantly influenced by the VLM agent's reasoning time. This process is analogous to ``System 2'' reasoning in LLM agents, where a deliberate thinking phase and chain-of-thought reasoning is required to solve difficult problems \citep{jaech2024openai,li2024chain}.

\newpage
\begin{figure*}[htb]
\centering
\begin{tcolorbox}[
    colback=white, 
    colframe=black!20, 
    width=\textwidth, 
    arc=2pt, % Small rounded corners
    outer arc=2pt,
    left=10pt, right=10pt, top=10pt, bottom=10pt
]

You are a helpful assistant for controlling a robot arm in a simulated environment. You will give advice on how to complete the task based on the image inputs.

\paragraph{Task instruction} Here is the task instruction: \verb|{task_instruction}|

\paragraph{Observation} Here is the starting state of the current episode \verb|<images>| and here is the current state images \verb|<images>|. The execution history so far is \verb|<images>|.

\paragraph{Action Proposal}
    Now, please propose the next actions for the robot to complete the task. You should complete the task following the order in the instruction. 
    The available atomic actions are:
    \begin{itemize}
        \item \textbf{MOVE} Here you move the gripper to a target position, and you should point it out in the multiview current state images. 
    The position should be represented by x,y pixel coordinates normalized to 0-1000. 
    To grasp or operate an object, move towards it. If you need to place an object, point to the target gripper position so the object can be placed following the instruction.
    Examples:
    \begin{verbatim}
{"action": "MOVE",
"parameters": {"frontview": {"x": 500, "y": 300},
"topview": {"x": 450, "y": 350},
"sideview": {"x": 480, "y": 320}}}
    \end{verbatim}
    \item \textbf{ROTATION} Here you rotate the gripper, and you return the rotation in Euler angles \verb|[delta_roll, delta_pitch, delta_yaw]| in degrees.
    Example:
    \begin{verbatim}
    {"action": "ROTATE","parameters": {"delta_roll": 0,
    "delta_pitch": 15, "delta_yaw": 0}}
    \end{verbatim}
    \item \textbf{RELEASE} Here you release the object by opening the gripper.
    
    Example: \verb|{"action": "RELEASE","parameters": {}}|
    \item \textbf{GRASP} Here you grasp the object by closing the gripper. 
    
    Example: \verb|{"action": "GRASP","parameters": {}}|
    \end{itemize}
    
    Finally, return a list of actions in json format.
\textbf{Skills}
\begin{itemize}
    \item When placing objects into baskets or target areas, ensure sufficient clearance to avoid collision.
    \item When grasping objects, ensure the gripper is aligned with the grasp region to enable successful grasp.
    \item When navigating the gripper, ensure sufficient clearance and avoid collision with objects.
\end{itemize}
\end{tcolorbox}
\caption{Example prompt template for agent action proposal.}
\label{fig:prompt proposal}
\end{figure*}

\begin{figure*}[htb]
\centering
\begin{tcolorbox}[
    colback=white, 
    colframe=black!20, 
    width=\textwidth, 
    arc=2pt, % Small rounded corners
    outer arc=2pt,
    left=10pt, right=10pt, top=10pt, bottom=10pt
]
\paragraph{Optimize gripper position}
        Here is the gripper position of our next robot actions \verb|<images>|, and I want you to look carefully and analyze the position of the gripper and the object, and optimize the gripper position following the guidelines below.
        Return the gripper position adjustments in x y z directions:
        The coordinate system and axis are defined as \verb|{coordinate_system}|.
        
        \paragraph{Guidelines}
        \begin{itemize}
            \item Understand what the gripper is trying to do, based on your task instruction understandings.
        \verb|{task_description}|.
        \item Analyze the trajectory frame by frame, and inspect whether there may be potential collisions with any object along the trajectory. If there is any potential collision with rims and objects, adjust the gripper so that the trajectory is fully clear of any obstacles.
        
        \item When the gripper is trying to place the object, first verify whether the position is above the target region to be placed. If not, zoom in and adjust the posiiton of the gripper.
        
        \item If the gripper is about to grasp an object, the gripper should be aligned and directly above the object. Adjust the gripper if it is clearly misaligned with the object.
        
        \end{itemize}
       
        \paragraph{Format}
        Please return your optimization direction in json format, for example:
        \begin{verbatim}
        {"Delta_x": 0,"Delta_y": -1,"Delta_z": 1}
        \end{verbatim}
        where x, y, z can only be -1, 0 or 1, with 0 being no movement in that direction, 1 being move towards positive direction and -1 being move towards negative direction.
        Return 0 in the corresponding axis if the gripper is mostly aligned and safe.
\end{tcolorbox}
\caption{Example prompt template for agent feedback for action optimization.}
\label{fig:optimization prompt}
\end{figure*}

\begin{figure*}[htb]
\centering
\begin{tcolorbox}[
    colback=white, 
    colframe=black!20, 
    width=\textwidth, 
    arc=2pt, % Small rounded corners
    outer arc=2pt,
    left=10pt, right=10pt, top=10pt, bottom=10pt
]

I have a set of images \verb|<images>| showing the robot actions trying to complete the task \verb|{task_description}|. I want you to rank them from best to worst in terms of whether the action is accomplishing the task goal safely. Here are more detailed guidelines:
\paragraph{Guidelines}
\begin{itemize}
    \item If the gripper is trying to grasp an object, I want you to rank them from best to worst in terms of whether the grasp is firm and clear. The best image should be a clear and firm grasp at the right place. 
    
    Reflect on the trajectory history and task instructions, and understand which object is the gripper trying to grasp.

    \item If the gripper orientation is different in the candidate images, you should identify the best orientation for a clear and stable grasp.
    
\end{itemize}

\paragraph{Reminder}
\begin{itemize}
    \item You should rank all the images following the same standard. If none of them is perfect, you should rank them by which one is closest.
    \item When ranking the later images, refer and reflect the previous candidates to rank them faithfully. For example, if the first image is blurry and the second image is clear, then the second image should be ranked higher than the first image.
\end{itemize} 
Return the ranking result in json format, for example: \verb|[0, 2, 1]|, where the ids in the list are the image ids ranked from best to worst, with the first being the best. The range of the ids should be from 0 to N-1, where N is the total number of candidate images.
\end{tcolorbox}
\caption{Example prompt template for agent candidate ranking for local search.}
\label{fig:ranking prompt}
\end{figure*}
\newpage

\clearpage
\end{document}